%% file: arxiv.tex
\documentclass{article}

\usepackage{iclr2027_conference,times}
\usepackage[T1]{fontenc}
\usepackage[utf8]{inputenc}
\usepackage{microtype}
\usepackage{amsmath,amssymb,amsthm,mathtools,bm}
\usepackage{booktabs}
\usepackage{colortbl}
\usepackage{longtable}
\usepackage{tabularx}
\usepackage{graphicx}
\usepackage{float}
\usepackage{placeins}
\usepackage[most]{tcolorbox}
\usepackage{xcolor}
\usepackage{enumitem}
\usepackage{needspace}
\usepackage{tikz}
\usetikzlibrary{arrows.meta,positioning,calc,fit}
\usepackage{pgfplots}
\pgfplotsset{compat=1.18}
\usepackage{hyperref}
\usepackage[nameinlink,noabbrev]{cleveref}
\crefname{section}{Section}{Sections}
\crefname{subsection}{Section}{Sections}
\crefname{appendix}{Appendix}{Appendices}
\crefname{equation}{Equation}{Equations}
\crefname{figure}{Figure}{Figures}
\crefname{table}{Table}{Tables}
\crefname{theorem}{Theorem}{Theorems}
\crefname{proposition}{Proposition}{Propositions}
\crefname{lemma}{Lemma}{Lemmas}
\crefname{corollary}{Corollary}{Corollaries}
\crefname{definition}{Definition}{Definitions}
\crefname{remark}{Remark}{Remarks}

\definecolor{deepblue}{RGB}{36,82,128}
\definecolor{deepred}{RGB}{166,62,55}
\definecolor{deepgreen}{RGB}{51,120,91}
\definecolor{softgray}{RGB}{242,244,247}
\hypersetup{
    colorlinks=true,
    linkcolor=deepblue,
    citecolor=deepgreen,
    urlcolor=deepblue
}

\setlist{leftmargin=*,nosep}

\newtheorem{theorem}{Theorem}
\newtheorem{proposition}[theorem]{Proposition}
\newtheorem{lemma}[theorem]{Lemma}
\newtheorem{corollary}[theorem]{Corollary}
\theoremstyle{definition}
\newtheorem{definition}{Definition}
\theoremstyle{remark}

\newcommand{\R}{\mathbb{R}}
\newcommand{\E}{\mathbb{E}}
\newcommand{\cD}{\mathcal{D}}
\newcommand{\cM}{\mathcal{M}}

\newcommand{\inner}[2]{\left\langle #1,#2\right\rangle}

\DeclareMathOperator*{\argmin}{arg\,min}

\title{Self-Confirming Superposition Traps \\ in Reinforcement Learning}

\author{%
Dai Shi\textsuperscript{1}\quad
Andi Han\textsuperscript{2}\quad
Feng Chen\textsuperscript{3}\\
\bfseries Yiqun Duan\textsuperscript{4}\quad
Junbin Gao\textsuperscript{2}\quad
Jos\'e Miguel Hern\'andez-Lobato\textsuperscript{1}\\[0.5em]
\textnormal{\small\textsuperscript{1}\,University of Cambridge\quad
\textsuperscript{2}\,University of Sydney}\\
\textnormal{\small\textsuperscript{3}\,University of Adelaide\quad
\textsuperscript{4}\,Facebook}
}

\iclrfinalcopy

\begin{document}

\maketitle
\lhead{Preprint}

\begin{abstract}
Reinforcement learning (RL) trains representations on data selected by the
agent's policy and updates that policy using the returns they support.
We show that this loop can sustain a lower-return policy even when
representation fitting is globally optimal for the policy's own data.
Specifically, we identify a \emph{self-confirming superposition trap}, in
which every optimal code shares features that rarely co-activate under
the current policy but interfere after an alternative action. The
resulting control error lowers that action's return and reinforces its
avoidance, although choosing it would earn more after representation
adaptation at the same capacity. In a tied two-step model, we characterize which feature counts and
representation dimensions admit a trap. We also show separately that
continued visitation, equal feature frequencies, and independent
controller learning need not prevent it. Because rarely chosen actions
contribute little to the fitting objective, optimal fitting can tolerate
errors large enough to reverse their return advantage. We derive a bound
on this distortion in a broader finite-action model and use it to prove a
sufficient replay condition: retaining enough training weight on the best
separately adapted action preserves its return advantage despite residual
representation error. Experiments with neural PPO show that agents
initialized toward different actions develop different interference
patterns and opposite mean return
rankings at the same capacity, with corresponding differences in their
final policies. Motivated by the replay condition, we test whether
preserving access to neglected training states can improve control.
Providing such access reduces measured interference and improves
sequential return, with return gains even when the encoder is frozen.
We extend these tests to MiniGrid and DMControl, where training-state
access, replay reweighting, and overlap penalties can improve control.
In DreamerV3--Crafter, protecting world-model fitting weight also improves
cumulative training scores at unchanged capacity. Our code is available at
\url{https://anonymous.4open.science/r/self-confirming-superposition-traps-SCST}.
\end{abstract}

\section{Introduction}
\label{sec:intro}
\label{sec:loop}

Reinforcement learning (RL) agents learn from experience selected by their
own actions. With a shared representation, the policy determines which
features the controller must distinguish, and control quality determines
the returns guiding future choices
\citep{yarats2021proto,moalla2024representation}. This feedback appears in
visual control \citep{yarats2021sacae,yarats2022drqv2} and world-model
learning for prediction and planning
\citep{hansen2024tdmpc2,hafner2025dreamerv3}.

Through feature superposition, a representation can store more features
than it has dimensions by allowing their directions to overlap.
Sharing makes limited capacity available to more features, but can
introduce interference when one feature affects the response to another.
The cost depends on how often this interference occurs and how heavily
its errors are penalized, linking feature sharing to sparsity, importance,
and co-occurrence \citep{elhage2022toy,prieto2026statistics}.
However, in RL, the policy itself shapes the training distribution and can change which features occur together without changing their individual frequencies. Two features may rarely co-activate in the data currently collected, yet appear together frequently under an action the agent tends to avoid. An allocation that fits the collected data well can therefore impair control after that action.

Even if the actor knows these returns exactly, it evaluates alternatives
through the representation learned from its own data. A different policy
could earn more after adapting that representation at the same capacity.
We ask: \emph{can optimal representation fitting stabilize a
lower-return policy?}

We construct a \emph{self-confirming superposition trap} in which every
globally optimal code under a lower-return policy shares the features
responsible for reversing the action ranking. Replicator updates from exact returns
sustain this allocation (\cref{sec:exact}). We characterize the
feature counts and representation dimensions admitting a trap and show
that equal feature frequencies, positive visitation, and independent
controller learning can each leave it intact (\cref{sec:robustness}).
The common problem is that rarely chosen actions can incur substantial
error at little cost to fitting. A finite-action analysis bounds the
error needed for reversal and establishes local stability despite small
fitting and value errors (\cref{sec:setting}). Protecting the best
separately adapted action's replay weight can preserve its ranking
without eliminating all representation error or increasing capacity
(\cref{sec:interventions}).

Our experiments follow this feedback from matched-feature data to neural
policy learning (\cref{sec:experiments}). Co-activation changes and initial
policy preferences lead to different interference patterns, return
rankings, and subsequent choices. Training-state interventions then test
whether broader access can improve control, including with frozen
encoders. We extend the intervention tests to MiniGrid and DMControl and
test protected world-model fitting in DreamerV3 on Crafter.

\section{From Feature Sharing to a Self-Confirming Trap}
\label{sec:exact}
\label{sec:task-construction}
\label{sec:optimal-geometry}

We establish the trap in a two-step task encoding three features in two
dimensions. Limited capacity requires feature sharing, and the policy
determines which shared features occur together. This dependence can
stabilize a lower-return policy even when its code is globally optimal.

\paragraph{Setting.}
The branch policy $\pi$ selects $B=E$ with probability $p$ and
$B=S$ with probability $1-p$. The environment then samples a pair of
active feature indices $G\subset\{1,2,3\}$ according to the branch.
Under $S$, $G$ is equally likely to be $\{1,3\}$ or $\{2,3\}$.
Under $E$, it is equally likely to be $\{1,3\}$ or $\{1,2\}$
(\cref{fig:trap-overview}a). The three feature values form the vector
$Z=(Z_1,Z_2,Z_3)^\top\in\R^3$, which is independent of $(B,G)$,
with $\E[Z]=0$ and $\E[ZZ^\top]=I_3$. The observation $o$ contains
the active feature values, with $o_i=Z_i$ for $i\in G$ and $o_i=0$
otherwise.
A learned code $V=[v_1,v_2,v_3]$ assigns each feature a unit direction
$v_i\in\R^2$. The encoder forms $h=Vo+\nu$, where $\nu$ is independent,
centered noise with covariance $\sigma^2I_2$. The controller uses tied
linear weights to produce the terminal action $\widehat Z=V^\top h$. Unit columns and
weight tying give each feature a self-response of one, so cross-feature
errors depend on the overlap of their directions. 
The terminal reward $r_B=b_B-\sum_{i\in G}(\widehat Z_i-Z_i)^2$
subtracts the squared errors on active features from a fixed base
reward $b_B$. We write $\delta=b_E-b_S$ for the base-reward difference
between $E$ and $S$, before accounting for control error.

\begin{figure}[t]
\centering
\includegraphics[width=\textwidth]{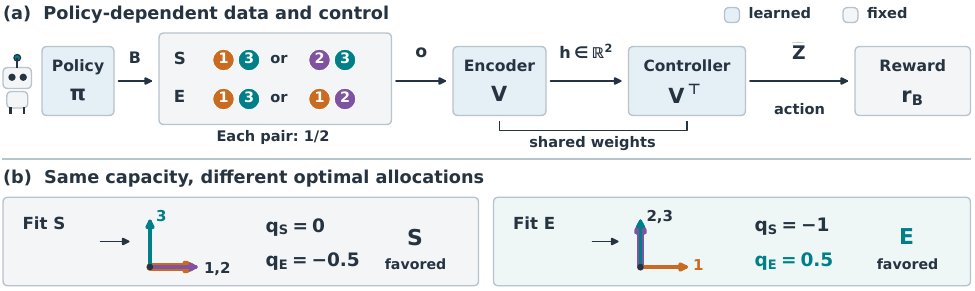}
\caption{
Optimal feature sharing can favor the worse policy.
}
\label{fig:trap-overview}
\end{figure}

\paragraph{Optimal allocation and returns.}
With squared overlap $g_{ij}=\langle v_i,v_j\rangle^2$, an active pair
$\{i,j\}$ incurs expected control error $2g_{ij}+2\sigma^2$.
Averaging over each branch gives
$D_S=g_{13}+g_{23}+2\sigma^2$ and
$D_E=g_{12}+g_{13}+2\sigma^2$,
with actual expected return $q_B(V)=b_B-D_B(V)$. Fitting the code to the
policy's data minimizes $(1-p)D_S+pD_E$, or, up to the constant noise term,
\begin{equation}
L_p(V)=p g_{12}+g_{13}+(1-p)g_{23}.
\label{eq:frame-objective}
\end{equation}
Overlap $g_{12}$ has weight one in $D_E$ and weight $p$ in $L_p$.
For a fixed code, reducing $p$ therefore lowers the contribution of this
error to fitting without improving control on $E$.
At $p=0$, every optimal code has $v_1=\pm v_2\perp v_3$:
features $1$ and $2$ share a direction because $S$ never requires them
together. With $b_S=0$, $b_E=0.5$, and $\sigma=0$, this code gives
$q_S=0$ and $q_E=-0.5$. Fitting to $E$ instead makes features $2$ and
$3$ share and gives $q_E=0.5$ at the same capacity
(\cref{fig:trap-overview}b). Thus the current code favors $S$, but
choosing $E$ and refitting yields more return. If the same allocation
remains optimal when $E$ is visited occasionally, its return disadvantage
pushes the policy back toward $S$.

\paragraph{Defining the trap.}
\label{sec:trap-definition}
We state the definition for finitely many actions and linear
encoder--controller models with either tied weights or an independently
parameterized controller, although we also develop more general results in
\cref{sec:setting,sec:interventions}.
To do so, let $x\in\Delta_A$ be a distribution over $A$ actions and
$\theta\in\Theta$ collect the encoder and controller parameters.
Fitting to policy $x$ minimizes the average control loss
$R_x(\theta)=\sum_a x_aD_a(\theta)$, with minimum $R^\star(x)$
and minimizer set $\cM_0(x)$. Since action returns are
$q_a(\theta)=b_a-D_a(\theta)$, all these optimal fits yield the
same adapted return $J^\star(x)=x^\top b-R^\star(x)$.
In the two-step task, $\theta=V$ and we write $J^\star(p)$ for
$J^\star((1-p,p))$.
Finally, let $q^\perp$ denote the reference returns obtained by removing
all cross-feature error (i.e., error caused by other active features
contributing to a feature's predicted value) while preserving self-response
and noise (\cref{app:interference-reference}). Comparing the actual return
vector $q=(q_a)_{a=1}^A$ with $q^\perp$ tests whether this interference
reverses the policy ranking. 
\begin{definition}[Self-confirming Superposition Trap]
\label{def:trap}
For the specified actor $\dot x=F(x,q(\theta))$, fix feasible
action distributions $x^\star,y$ and a nonempty set $P$ of feature pairs, each
more often co-active under $y$ than under $x^\star$.
Let $I_a^P(\theta)$ be the expected squared cross-feature error contributed
by $P$ after action $a$ (\cref{eq:I_a^p}), and put $\Delta x=y-x^\star$.
The action distribution $x^\star$ is a \emph{self-confirming superposition trap},
witnessed by $P,y$, if:
\begin{enumerate}[label=(\roman*)]
\item \emph{Allocation:} At every $\theta\in\cM_0(x^\star)$, each pair
in $P$ has nonorthogonal encoder directions.
\item \emph{Reversal:} At every $\theta\in\cM_0(x^\star)$,
\[
(\Delta x)^\top q(\theta)<0<(\Delta x)^\top q^\perp(\theta),
\qquad
(\Delta x)^\top I^P(\theta)>-(\Delta x)^\top q(\theta).
\]
\item \emph{Feedback:} Exact-response trajectories exist locally from
nearby feasible action distributions, and $x^\star$ is locally asymptotically
stable relative to the feasible policy set, uniformly over absolutely
continuous policy trajectories with measurable
$\theta(t)\in\cM_0(x(t))$ (\cref{def:policy-stability}).
\item \emph{Suboptimality:} $J^\star(x^\star)<J^\star(y)$.
\end{enumerate}
\end{definition}

\paragraph{Policy feedback and stability.}
\label{sec:dynamics}
\label{sec:deployed-gap}
\label{sec:bistability}
To show that the two-step task contains a trap, one shall establish that nearby policies return to pure $S$ as the code is refitted. We assume
\emph{exact adaptation}: at each $p$, the shared encoder--controller code
$V$ globally minimizes the current fitting loss $L_p$. Although optimal codes need
not be unique, \cref{thm:frame} shows that all have the same squared
overlaps $g_{12},g_{13},g_{23}$ at a fixed $p$. They therefore give the
same branch returns, with control-loss difference
$\cD(p):=D_E-D_S=g_{12}(p)-g_{23}(p)$. This difference equals $1$ for
$p\le p_-=(3-\sqrt5)/2$, equals $-1$ for $p\ge p_+=(\sqrt5-1)/2$,
and decreases strictly between these thresholds
(\cref{fig:theory-mechanism}a).
The actual return gap is $\Delta_q(p)=q_E-q_S=\delta-\cD(p)$.
For this construction, consider the replicator flow
$\dot x_a=x_a[q_a(\theta)-x^\top q(\theta)]$, which for the branch policy
gives the natural policy gradient $\dot p=p(1-p)[\delta-\cD(p)]$
\citep{kakade2001natural}, although ordinary expected policy-gradient
flow for a binary logit has the same attractors and basin boundary
(\cref{app:binary-policy-updates}).

\begin{figure}[t]
\centering
\includegraphics[width=\textwidth]{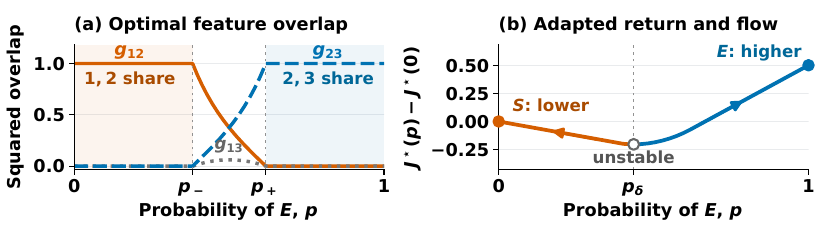}
\caption{Globally optimal fitting can sustain the worse policy.
(a) Optimal overlaps. (b) Adapted return at $\delta=0.5$.
Arrows show updates, and filled/open points mark stable/unstable policies.}
\label{fig:theory-mechanism}
\end{figure}

\begin{theorem}[Self-confirming Trap and Bistability]
\label{thm:trap}
For $0<\delta<1$, in the tied three-feature task, the policy $p=0$ is a self-confirming superposition trap, witnessed by the feature pair
$\{1,2\}$ and the pure-$E$ policy, with adapted return deficit
$J^\star(1)-J^\star(0)=\delta$.
For every $\delta\in(-1,1)$, a unique $p_\delta\in(p_-,p_+)$ satisfies
$\cD(p_\delta)=\delta$. Both endpoints are asymptotically stable relative
to $[0,1]$, while $p_\delta$ is unstable. Every interior initialization
below $p_\delta$ converges to zero, and every one above it converges to one.
\end{theorem}

\cref{fig:theory-mechanism}b shows why the initial policy matters: trajectories on either side of $p_\delta$ increase adapted return but
converge to different endpoints. Those approaching $S$ therefore improve throughout learning while ending at the worse policy, despite globally optimal fitting (\cref{app:dynamics}). Our experiments test whether initialization dependence persists under neural
PPO (\cref{sec:exp-feedback}).

\section{Robustness of Superposition Traps}
\label{sec:robustness}
\label{sec:matching-complement}
In this section, we show in two separate settings that the trap can persist with equal feature frequencies or an independently parameterized controller.
\paragraph{Equal feature frequencies.}
In the three-feature example, changing the policy changes both how
often individual features appear and which features appear together.
We now hold individual feature frequencies fixed across branches to
determine whether policy-dependent co-activation can still sustain the
trap. We encode $2d$
features as unit vectors in $\R^d$, $d\ge2$, and let $M_S,M_E$ be two
partitions of the features into pairs with no pair in common. Branch $B$
samples the active pair uniformly from all unordered pairs except those
in $M_B$. Every feature then appears with probability $1/d$ under either
branch, regardless of the policy. The dimension choice lets each of the
$d$ excluded pairs share one of $d$ orthogonal directions without
cross-feature error on $B$.

\begin{theorem}[Equal-frequency Superposition Trap]
\label{thm:matching-complement}
For the matching construction above and $p\in[0,1]$, the minimum of
$(1-p)D_S+pD_E$ over unit-column encoders is
$R_{\rm mc}^\star(p)=2\sigma^2+\frac{\min\{p,1-p\}}{d-1}$.
For $p<1/2$, every optimum has
$g_{ij}=\mathbf1\{\{i,j\}\in M_S\}$ for $i\ne j$. For $p>1/2$,
the same statement holds with $M_E$. At $p=1/2$, both matching
allocations are optimal. Under exact adaptation and the replicator
actor, $0<\delta<1/(d-1)$ makes both pure policies locally
asymptotically stable. Pure $S$ is a self-confirming superposition trap
with adapted return $\delta$ below that of pure $E$.
\end{theorem}

Beyond the fixed dimension ratio (i.e., 2$d$ and $d$) in \cref{thm:matching-complement}, we characterize
which feature counts $d_{\mathrm{in}}\ge2$ and representation dimensions
$d_{\mathrm{out}}\ge1$ admit a trap in the same tied two-step model.

\begin{theorem}[Trap Existence across Dimensions]
\label{thm:dimension-trap}
The tied two-step
model admits a self-confirming superposition trap at pure $S$ for
some branch pair distributions and $\delta>0$ if and only if
$d_{\mathrm{in}}>d_{\mathrm{out}}\ge2$.
If every feature must occur with probability $2/d_{\mathrm{in}}$
under each branch, such a trap exists if and only if
$d_{\mathrm{in}}>d_{\mathrm{out}}\ge2$ and $d_{\mathrm{in}}\ge4$.
\end{theorem}

Prior studies examine how sparsity and co-occurrence shape feature
sharing for given data distributions
\citep{elhage2022toy,prieto2026statistics}.
Here, the policy also controls co-occurrence, and the resulting optimal
feature sharing can in turn stabilize a lower-return policy even when
individual feature frequencies remain fixed
(\cref{app:dimension-traps}).
In our numerical experiments, we test whether this feedback develops
from finite data across different feature counts and representation
dimensions, using initial datasets with matched counts and active-value
distributions for each feature (\cref{sec:exp-feedback}).

\paragraph{Independently parameterized controller.}
\label{sec:untied}

We return to the three-feature task with the same optimal fitting and
replicator updates as in \cref{sec:exact}, jointly optimizing the
unit-length encoder directions and an independent linear controller,
$\widehat Z_i=w_i^\top h$.

\begin{theorem}[Untied-controller Trap at Any Positive Noise]
\label{thm:untied}
\label{cor:untied-local}
Let $\eta=\sigma^2>0$ and $s_\eta=(1+\eta)^{-2}$.
For $0<\delta<s_\eta$, pure $S$ is a self-confirming
superposition trap.
Both pure policies are locally exponentially stable for
$|\delta|<s_\eta$, uniformly over global-minimizer selections.
\end{theorem}

At every joint optimum for $S$, features $1$ and $2$ still share an
encoder direction, so the controller cannot separate their values when
they occur together on $E$. Optimizing the controller reduces the
interference penalty from $1$ to $s_\eta$, giving
$q_E-q_S=\delta-s_\eta$ (\cref{app:untied-endpoints}). Controller
flexibility therefore narrows the range of reward advantages that
sustain the trap, but does not eliminate it. In our numerical
experiments, we test whether this feedback persists with a nonlinear
encoder and an independently learned controller under PPO
(\cref{fig:empirical-mechanism}). Finally, in the tied three-feature task, neither complete avoidance of
$E$ nor instantaneous code adaptation is necessary for the lower-return
policy to remain attracting. The trap persists when each branch must be
selected with at least a sufficiently small positive probability.
When code and policy instead follow the simultaneous gradient dynamics
in \cref{sec:visitation-adaptation}, their lower-return equilibrium
remains locally attracting at every positive representation update rate,
including when both branches must continue to be selected.

\section{Policy Dynamics under Representation Adaptation}
\label{sec:setting}
\label{sec:envelope-location}
\label{sec:reversal-budget}
\label{sec:invasion-margin}

At the traps in \cref{sec:exact,sec:robustness}, superposition makes the
avoided action earn less under the current code despite its higher
return after separate adaptation. We now bound the excess control error
left by fitting and ask when the resulting return ranking stabilizes
the lower-return policy. Both analyses use the finite-action fitting model
introduced in \cref{sec:trap-definition}, without prescribing feature
geometry or requiring tied weights. We assume fixed continuous action
losses on a compact parameter class $\Theta$.

\paragraph{The error needed for reversal.}
An action's advantage after separate adaptation can disappear under a
shared code. We bound the extra control error needed to reverse its
return ranking and relate that requirement to its training weight.

\begin{theorem}[Joint-deficit Bound]
\label{thm:joint-deficit}
For each action, let $D_a^\star=\min_\theta D_a(\theta)$,
$\bar q_a=b_a-D_a^\star$, and
$\psi_a(\theta)=D_a(\theta)-D_a^\star\ge0$.
Define $\Psi(x)=\min_\theta\sum_a x_a\psi_a(\theta)$ and
$\varepsilon_J=\min_\theta\max_a\psi_a(\theta)$.
For $x\in\Delta_A$ and $\varepsilon\ge0$, write
$\cM_\varepsilon(x)=\{\theta\in\Theta:
R_x(\theta)\le R^\star(x)+\varepsilon\}$.
Every $\theta\in\cM_\varepsilon(x)$ satisfies
$\sum_a x_a\psi_a(\theta)\le\Psi(x)+\varepsilon
\le\varepsilon_J+\varepsilon$ and, whenever $x_ax_b>0$,
\[
\left|(q_a(\theta)-q_b(\theta))-(\bar q_a-\bar q_b)\right|
\le\frac{\Psi(x)+\varepsilon}{\min\{x_a,x_b\}}.
\]
For estimated returns $\widetilde q_a=q_a(\theta)+\xi_a$, if
$|\xi_i-\xi_j|\le\zeta$, $\bar q_i>\bar q_j$, and
$\widetilde q_i\le\widetilde q_j$, then
$\Psi(x)+\varepsilon\ge x_i(\bar q_i-\bar q_j-\zeta)_+$,
where $z_+=\max\{z,0\}$.
If $\varepsilon_J=0$, then
$R^\star(x)=\sum_a x_aD_a^\star$ on the whole simplex, and every
optimizer at a full-support action distribution attains every
action's separate minimum.
\end{theorem}

\Cref{thm:joint-deficit} shows why accurate fitting need not preserve
the return advantage of a separately adapted action. Errors enter the
fitting objective in proportion to the action's training weight, so a
rarely selected action can remain poorly controlled at little cost to
fitting. If one shared code attains all action minima, positive training
weights instead make every optimizer preserve the separately adapted
return ranking (\cref{app:joint-deficit}). We also test whether nearly optimal fits to finite datasets still reverse branch returns and whether orthogonal capacity removes the reversal (\cref{app:sampled-check}).

\paragraph{When the ranking sustains a policy.}
\Cref{thm:joint-deficit} bounds changes in return ordering at a fixed
policy. We then show that a strict return advantage across all optimal
codes makes a pure policy locally stable under refitting, even with
small fitting and value errors.

\begin{theorem}[Finite-action Robustness]
\label{thm:simplex-robustness}
For a vertex \(e_i\), define its worst-selection invasion margin
\begin{equation*}
\gamma_i=\min_{\theta\in\cM_0(e_i)}\min_{j\ne i}
\bigl[q_i(\theta)-q_j(\theta)\bigr].
\end{equation*}
If \(\gamma_i>0\), then for every \(0\le\zeta<\gamma_i\) there are
\(\bar\varepsilon>0\), a neighborhood \(U_i\), and \(\underline\gamma_i>0\) such that
every replicator solution using measurable
\(\theta(t)\in\cM_{\varepsilon(t)}(x(t))\),
\(0\le\varepsilon(t)\le\bar\varepsilon\), and relative value errors
\(\max_{a,b}|\xi_a-\xi_b|\le\zeta\), satisfies \(\frac{d}{dt}\log(x_j/x_i)\le-\underline\gamma_i\) for every \(j\ne i\) with \(x_j>0\) while the trajectory lies in \(U_i\). Hence \(e_i\) is locally exponentially stable,
uniformly over these selections and errors. If vertices \(e_i,e_k\) both have
positive margins and
\(\bar q_i<\bar q_k\), then under exact
responses both are stable, \(e_i\) is strictly suboptimal, and the same
robustness applies.
\end{theorem}

\Cref{thm:simplex-robustness} shows that the traps in
\cref{thm:trap,thm:matching-complement,thm:untied} do not depend on
perfectly accurate fitting or value estimates. Because superposition
gives $S$ a strict return advantage, nearby policies still converge to
it under sufficiently small errors (\cref{app:local-persistence}).
Improving adapted return does not rule out this convergence either:
under exact fitting and returns, the same policy dynamics can increase
adapted return while approaching the lower-return policy
(\cref{app:envelope-ascent}).

\section{Mitigating Superposition Traps}
\label{sec:interventions}
\label{sec:protocol}
\label{sec:exp-measurement}

We derive a sufficient replay condition for preserving return ordering
in the finite-action model, then compare overlap penalties and entropy
regularization in the tied two-step task under exact adaptation with
$0<\delta<1$.

\Cref{thm:joint-deficit} motivates reweighting because an action can
remain poorly controlled when its errors contribute little to fitting.
Experience replay retains past observations so their training weight
need not fall with current visitation. Mixing current-policy data with
stored observations whose action proportions are $\nu\in\Delta_A$
gives fitting weights $r(x)=(1-\alpha)x+\alpha\nu$ with
$0<\alpha\le1$. Each action with $\nu_a>0$ therefore retains weight
at least $\alpha\nu_a$ even when its current selection probability is zero.

\begin{corollary}[A Sufficient Replay Condition]
\label[corollary]{cor:general-replay}
In the finite-action model of \cref{sec:setting}, suppose $k$ uniquely
maximizes $\bar q_a$, with gap
$\Delta=\min_{j\ne k}(\bar q_k-\bar q_j)>0$, and
$r_k(x)\ge\beta>0$ for all $x\in\Delta_A$. Fix fitting and relative
value-error bounds $\varepsilon,\zeta\ge0$ such that
$m:=\Delta-\zeta-(\varepsilon_J+\varepsilon)/\beta>0$.
Every absolutely continuous replicator solution on $t\ge0$ using
measurable $\theta(t)\in\cM_\varepsilon(r(x(t)))$ and
$\max_{a,b}|\xi_a(t)-\xi_b(t)|\le\zeta$ satisfies
$\widetilde q_k-\widetilde q_j\ge m$ for all $j\ne k$ almost everywhere.
If $x_k(0)>0$, then
$x_j(t)/x_k(t)\le[x_j(0)/x_k(0)]e^{-mt}$, so $x(t)\to e_k$.
\end{corollary}

For the replay mixture above, taking $\beta=\alpha\nu_k$ reduces the
sufficient margin condition to
$\alpha\nu_k>(\varepsilon_J+\varepsilon)/(\Delta-\zeta)$, with
$\Delta>\zeta$. If the best action is unknown, uniform reference weights
$\nu_a=1/|A|$ give every action a fitting-weight floor $\alpha/|A|$.
With full-support initialization, the guarantee applies when this floor
satisfies the condition. Distributing replay over more actions weakens
the guarantee, though recovery may still occur below this sufficient bound.
Replay can thus restore the ranking at unchanged capacity even when
$\varepsilon_J>0$, provided the required observations are available and
fitting meets the assumed accuracy. A permanent replay mixture may still
leave the selected action's loss above its separate minimum, so policy
convergence need not attain its separately adapted return
(\cref{app:general-replay}).

While replay preserves training on avoided actions, overlap penalties
target interference directly.
In the tied task, a penalty can eliminate collinear pairs while leaving
enough interference to sustain the lower-return basin. Removing that
basin requires a stronger penalty that restores $E$'s return advantage
(\cref{prop:gram-threshold}). Entropy regularization changes the other
side of this feedback by encouraging visits before the action becomes
competitive. With a weak bonus, the policy can visit $E$ more often
while retaining the exact feature collision that makes it unattractive.
A sufficiently strong bonus leaves a unique interior attractor, which
need not maximize environmental return (\cref{prop:entropy-trap}). We therefore evaluate interventions
through actual control and return in
\cref{sec:exp-data-access,sec:exp-public}, since changes in feature
geometry or visitation alone need not resolve the disadvantage.

\section{Experiments}
\label{sec:experiments}

We test whether policy-dependent co-activation sustains lower-return
behavior through interference, then evaluate the interventions in
\cref{sec:interventions} on sequential control and public benchmarks,
including agents that learn a world model.
Full protocols and supplementary results appear in
\crefrange{app:sampled-check}{app:public-dreamer}.

\subsection{Co-activation and Learned Control}
\label{sec:exp-feedback}

\paragraph{Co-activation with matched feature marginals.}
To isolate co-activation in the four-feature, two-dimensional task of
\cref{thm:matching-complement}, we fit tied codes to paired datasets with
$r=0.4$ or $0.6$ of observations from $E$. Feature counts and active-value
distributions match exactly. Both fits reach population losses near
$0.5802$, against an optimum of $0.58$, but allocate overlap differently
(\cref{fig:matching-sampled}a). Their mean $E-S$ return gaps on fresh
observations are $-0.496$ and $1.496$, respectively.

To test continued learning across dimensions, we use the equal-frequency
constructions of \cref{thm:dimension-trap} with 16 or 32 features in two
or four dimensions. Codes fitted separately to $S$- and $E$-data have
matched feature counts and active-value distributions, and every fit
lies within $2.31\times10^{-4}$ of its population optimum. Starting from
$p(E)=0.5$, codes learn from policy-collected observations and policies
from fresh branch-return estimates. All 40 continuations per condition
approach opposite pure policies in each setting
(\cref{fig:matching-sampled}b). The $S$-fitted group earns less, so co-activation can redirect learning
toward lower return at unchanged capacity.

\begin{figure}[t]
\centering
\includegraphics[width=\textwidth]{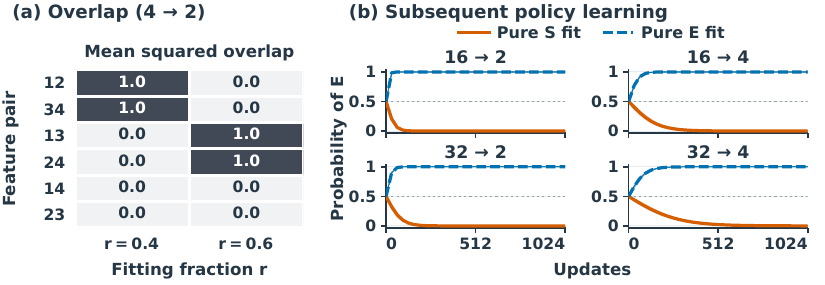}
\caption{Different co-activation patterns lead the same initial policy
toward opposite branches.}
\label{fig:matching-sampled}
\end{figure}

\begin{figure}[t]
\centering
\includegraphics[width=\textwidth]{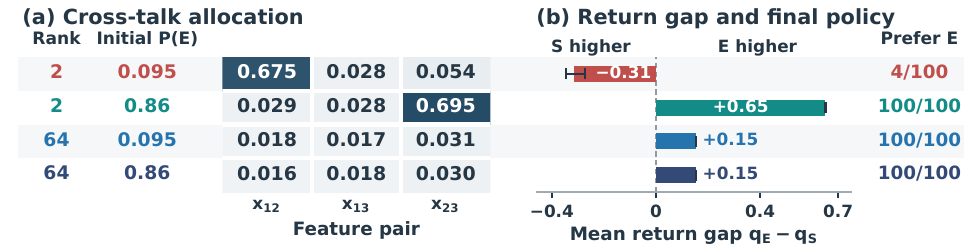}
\caption{Initial preference changes cross-talk allocation and mean
branch-return ordering.}
\label{fig:empirical-mechanism}
\end{figure}

\paragraph{Reward-only neural control.}
We then test nonlinear encoders and independent controllers trained by
reward-only PPO \citep{schulman2017ppo} in the three-feature task.
Fixed rank-2 or rank-64 projectors vary capacity at matched parameter
counts. We compare low and high initial $p(E)$ while retaining both branches
in training. Rank-two low-start agents
develop stronger cross-talk on $E$'s exclusive pair and a negative mean
$E-S$ return gap. High-start agents concentrate cross-talk on $S$'s
exclusive pair and have a positive mean gap
(\cref{fig:empirical-mechanism}). Only 4 of 100 low-start runs finish
with $p(E)>0.5$, compared with all 100 high-start runs, whose deployed
return is $0.131\,[0.125,0.135]$ higher at the same capacity
(95\% seed-bootstrap interval). At rank 64, both starts give small
cross-feature responses, positive mean gaps, and $E$-favoring policies
in every run. Thus initial preferences can sustain different interference
patterns and returns even when the controller learns independently.

\subsection{Sequential Control and Training-State Access}
\label{sec:exp-data-access}

We test whether access to later training states improves learning when
control errors restrict data collection. In a four-stage version of the
neural task, $S$ attempts to exit for immediate reward and $E$ to advance
toward the final reward. Failed control repeats the stage and uses episode budget.

\begin{figure}[t]
\centering
\includegraphics[width=\textwidth]{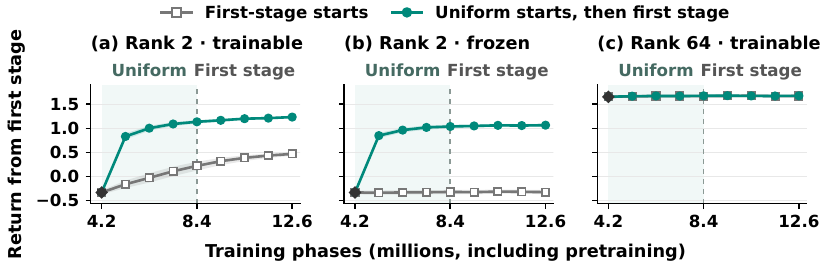}
\caption{Temporary stage access improves rank-two return, including with frozen encoders.}
\label{fig:sequential-intervention-time}
\end{figure}

\paragraph{Ordinary learning.}
Rank-2 and rank-64 PPO agents train from the first stage with matched
parameter counts, nonlinear encoders, and independent controllers.
Across three initial policies, their mean undiscounted returns are
$0.61$ and $1.67$. Discounted advance--exit gaps are negative in all
60 rank-two runs and positive in all 60 rank-64 runs. Rank-two controllers also respond more strongly to the other active
feature (\cref{app:sequential-neural}).

\paragraph{Training-state interventions.}
We use separate checkpoints pretrained on balanced advance-branch data
and compare paired continuations at fixed capacity and training budget.
One always starts at the first stage. The other temporarily samples
all four stages uniformly with fresh episode budgets before returning
to first-stage starts. All evaluations start at the first stage.
We also freeze rank-two encoders while the controller, route policy,
and critic continue learning.

Temporary uniform starts improve final return in all 20 rank-two pairs
with either trainable or frozen encoders, and the gains persist after
ordinary training resumes (\cref{fig:sequential-intervention-time}).
With trainable encoders, final first-stage $E$ mean-control error is lower
than under ordinary starts ($0.219$ versus $0.610$), as is interference
measured by resampling non-target inputs ($0.039$ versus $0.259$).
All 20 pairs improve on both diagnostics
(\cref{app:sequential-measurements}). Rank-64 checkpoints already perform
well and yield similar returns under both conditions. Gains with frozen encoders require no further encoder adaptation.

\begin{figure}[t]
\centering
\includegraphics[width=\textwidth]{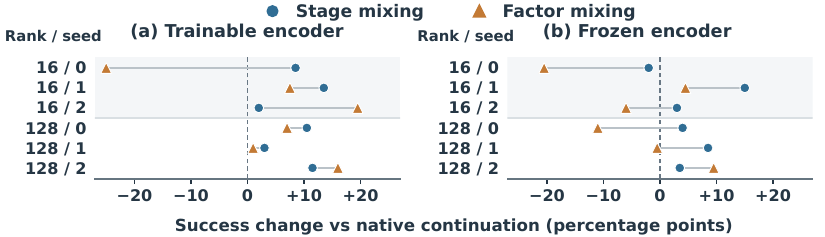}
\caption{Stage mixing improves DoorKey success more consistently than
factor mixing.}
\label{fig:public-interventions}
\end{figure}

\subsection{Interventions on Public Tasks}
\label{sec:exp-public}

We test training-state access beyond the constructed tasks in MiniGrid
DoorKey \citep{chevalierboisvert2023minigrid}, where PPO must collect a key,
open a door, and reach the goal. Paired continuations share checkpoints,
capacity, and training budgets, with trainable or frozen sensory encoders.
Stage mixing combines native resets with reachable states after key
pickup, unlocking, or door traversal. This oracle supplies additional
state access while preserving rewards and elapsed task time. From native
starts, success and return improve in 11 of 12 comparisons, including five
of six with frozen encoders (\cref{fig:public-interventions}). Factor mixing instead varies door states and deposited-key positions at
door decisions. It lowers local one-action oracle regret in all 12 comparisons but
improves native-start success in seven. Broader stage access improves
completion more consistently. The task also requires reaching
and opening the door (\cref{app:public-interventions}).

\paragraph{Replay, overlap, and entropy.}
We next test interventions with native state access on DMControl
Finger-turn-hard and Hopper-hop \citep{tassa2018control}. We continue
self-trained TD-MPC2 \citep{hansen2024tdmpc2} and SAC
\citep{haarnoja2018sac} checkpoints with unchanged architectures, resets,
and rewards. At matched additional training budgets, we rebalance replay
across task phases, penalize local representation overlap, or adjust
existing entropy regularization. SAC retains automatic entropy tuning with targets $-d/2$ or $0$
for action dimension $d$.
\Cref{tab:public-continuous} reports mean and SD over five evaluation
groups per selected checkpoint (\cref{app:public-dmcontrol}). Phase
replay improves TD-MPC2 on both tasks across all three original seeds.
Overlap penalties at $\lambda=0.003$ improve both TD-MPC2 tasks and
raise SAC's Hopper return from $39.9$ to $51.4$.

\begin{table}[t]
\centering
\caption{Returns after continuing TD-MPC2 and SAC with each intervention.}
\label{tab:public-continuous}
\small
\setlength{\tabcolsep}{4pt}
\renewcommand{\arraystretch}{0.82}
\setlength{\aboverulesep}{0.3ex}
\setlength{\belowrulesep}{0.45ex}
\begin{tabularx}{\textwidth}{lXrr}
\toprule
\rowcolor{deepblue!10}
\textbf{Learner} & \textbf{Intervention} & \textbf{Finger-turn-hard} & \textbf{Hopper-hop} \\
\midrule
\rowcolor{deepblue!5}
\textcolor{deepblue}{\textbf{TD-MPC2}} & Baseline & \cellcolor{deepblue!7}$749.3\,\pm\,35.7$ & \cellcolor{deepblue!19}$177.1\,\pm\,2.9$ \\
 & Phase replay & \cellcolor{deepblue!27}$\color{deepblue}\mathbf{863.3}\,\pm\,60.3$ & \cellcolor{deepblue!25}$206.0\,\pm\,2.3$ \\
 & Overlap $\lambda=0.003$ & \cellcolor{deepblue!20}$825.6\,\pm\,70.4$ & \cellcolor{deepblue!20}$184.3\,\pm\,3.9$ \\
 & Overlap $\lambda=0.03$ & \cellcolor{deepblue!13}$782.6\,\pm\,37.5$ & \cellcolor{deepblue!4}$103.0\,\pm\,4.4$ \\
 & Prior entropy $\times2$ & \cellcolor{deepblue!4}$734.8\,\pm\,25.9$ & \cellcolor{deepblue!18}$172.6\,\pm\,4.1$ \\
 & Prior entropy $\times5$ & \cellcolor{deepblue!15}$797.8\,\pm\,45.4$ & \cellcolor{deepblue!27}$\color{deepblue}\mathbf{217.2}\,\pm\,4.1$ \\
\midrule
\rowcolor{deepgreen!6}
\textcolor{deepgreen}{\textbf{SAC}} & Baseline & \cellcolor{deepgreen!20}$257.7\,\pm\,33.4$ & \cellcolor{deepgreen!19}$39.9\,\pm\,1.8$ \\
 & Phase replay & \cellcolor{deepgreen!26}$269.2\,\pm\,77.8$ & \cellcolor{deepgreen!11}$28.8\,\pm\,2.8$ \\
 & Overlap $\lambda=0.003$ & \cellcolor{deepgreen!15}$247.6\,\pm\,45.6$ & \cellcolor{deepgreen!27}$\color{deepgreen}\mathbf{51.4}\,\pm\,1.5$ \\
 & Overlap $\lambda=0.03$ & \cellcolor{deepgreen!24}$265.5\,\pm\,45.9$ & \cellcolor{deepgreen!13}$32.6\,\pm\,1.6$ \\
 & Target entropy $-d/2$ & \cellcolor{deepgreen!27}$\color{deepgreen}\mathbf{271.5}\,\pm\,41.5$ & \cellcolor{deepgreen!14}$32.9\,\pm\,1.6$ \\
 & Target entropy $0$ & \cellcolor{deepgreen!4}$227.5\,\pm\,55.7$ & \cellcolor{deepgreen!4}$19.4\,\pm\,1.4$ \\
\bottomrule
\end{tabularx}
\end{table}

\paragraph{Protected world-model fitting.}
We test whether protecting infrequent experiences improves
DreamerV3--Crafter \citep{hafner2025dreamerv3}, which learns recurrent
latent dynamics from $64\times64$ RGB images and trains policies through
imagined trajectories.
We compare matched \texttt{size12m} models trained for one million
interaction steps. Each replay position is grouped by the highest
pickaxe or sword held: none, wood, stone, or iron. Protected fitting
retains 75\% of the original world-model loss weighting and distributes
25\% equally across represented tiers, increasing the influence of
less frequent tiers. Inventory supplies weighting metadata only and is
excluded from model and policy inputs. Architecture, replay sampling,
update budgets, and actor and critic objectives remain unchanged.
Across three paired seeds, mean cumulative training score rises from
$7.83$ to $8.43$, with gains in all three pairs
(\cref{tab:dreamer-paired}). For the displayed checkpoints, we compare 32-frame forecasts on two
illustrative clips with identical 32-frame histories and recorded future actions. Protected fitting lowers
mean pixel error by $9.8\%$ and $14.8\%$ on these clips.
\Cref{fig:dreamer-crafter-fitting} shows the training curve and forecasts
alongside ground truth (protocol in \cref{app:public-dreamer}).

\begin{figure}[t]
\centering
\includegraphics[width=\textwidth]{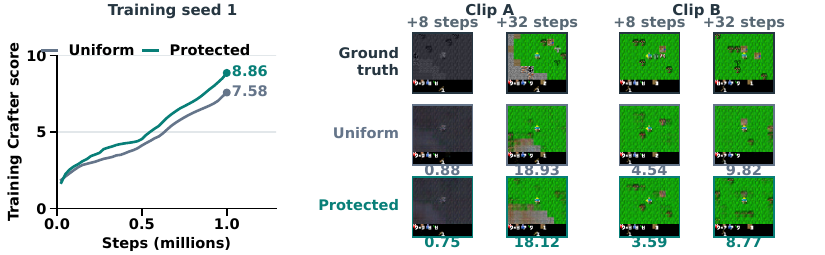}
\caption{Dreamer training and matched-input forecasts. Labels give pixel MSE $\times1{,}000$.}
\label{fig:dreamer-crafter-fitting}
\end{figure}

\section{Conclusion}
\label{sec:conclusion}
In this paper, we identify self-confirming superposition traps, where
globally optimal representation fitting sustains a lower-return policy.
They persist across dimensions, equal feature frequencies, continued
visitation, and independent controllers in our two-step models. Our finite-action analysis
links return reversal and stability to fitting weights and gives a
sufficient replay condition for restoring the better action's ranking.
Experiments recover the feedback with learned controllers and show that
training-state and fitting interventions can improve sequential control
and performance on public tasks. Our guarantees assume specified losses
and policy updates, while public-task gains do not isolate superposition
as their cause. Future work should identify harmful superposition in
learned agents and extend intervention guarantees to general multistep
control.
\label{main:end}
\clearpage

\section*{Reproducibility Statement}
Model assumptions and theorem statements are given in
\crefrange{sec:exact}{sec:interventions}, with detailed proofs and
derivations in \crefrange{app:definition}{app:general}.
Experimental task definitions, architectures, training and intervention
settings, and evaluation procedures are documented in
\crefrange{app:sampled-check}{app:public-dreamer}, including seed
conventions and the distinction between training variability and
repeated evaluation of a fixed checkpoint. The anonymous repository
linked in the abstract documents the released code and saved records,
with separate instructions for reproducing the covered numerical results
and figures or running new training experiments. Executable algebra
checks supplement the written proofs.

\section*{AI Use Statement}
Generative AI tools (OpenAI Codex and Anthropic Claude) were used for
manuscript drafting, editing, and restructuring. The authors take responsibility for the final manuscript, including its mathematical claims, experimental results, and references.

\bibliographystyle{iclr2027_conference}
\bibliography{references}
\clearpage
\appendix
\noindent

\section{Complete Statements for the Main Constructions}
\label{app:moved-theory-statements}

\subsection{Complete Feature-allocation Statement}

\begin{proposition}[Optimal Feature Allocation]
\label{thm:frame}
For $p\in[0,1]$, minimize $L_p$ over unit-column matrices
$V\in\R^{2\times3}$. Let $p_-=(3-\sqrt5)/2$,
$p_+=(\sqrt5-1)/2$, and $\chi=p(1-p)$. The minimum is
\begin{equation}
L^\star(p)=
\begin{cases}
p,&0\le p\le p_-,\\
\frac32-\frac14(\chi+\chi^{-1}),&p_-<p<p_+,\\
1-p,&p_+\le p\le1.
\end{cases}
\label{eq:optimal-loss}
\end{equation}
Every global minimizer has the same squared Gram entries:
$(g_{12},g_{13},g_{23})=(1,0,0)$ for $0\le p\le p_-$ and
$(0,0,1)$ for $p_+\le p\le1$. For $p_-<p<p_+$, they are
\begin{equation}
(g_{12},g_{13},g_{23})=\frac14\left(
\frac{1-p}{p^2}-\frac{p^2}{1-p},\;
4+\chi-\chi^{-1},\;
\frac{p}{(1-p)^2}-\frac{(1-p)^2}{p}
\right).
\label{eq:optimal-overlaps}
\end{equation}
\end{proposition}

\subsection{Positive Visitation and Finite Adaptation}
\label{sec:visitation-adaptation}

The inferior endpoint assigns zero probability to $E$, but the feedback
also survives when both branches continually supply training data.
Impose a visitation floor through
$p_\epsilon(s)=\epsilon+(1-2\epsilon)s$, $s\in[0,1]$, with
$0<\epsilon<p_-$, and use the actor
$\dot s=(1-2\epsilon)s(1-s)[\delta-\cD(p_\epsilon(s))]$ under exact adaptation.
The constrained endpoints $p=\epsilon$ and $p=1-\epsilon$ remain
asymptotically stable for $|\delta|<1$. When $0<\delta<1$, the lower
endpoint retains the collision and loses $(1-2\epsilon)\delta$ in
adapted return, satisfying \cref{def:trap} on this feasible interval
(\cref{cor:policy-floor}). Positive visitation leaves the trap intact
when the interfering pair still carries too little training weight to
change the optimal allocation.

Representation learning need not finish between policy updates for the
inferior equilibrium to remain attracting. Let $z$ denote the code's two
line angles relative to $v_3$, removing the common rotation, and let
$M=M^\top\succ0$. Simultaneous gradient adaptation and policy updates take
the form
\[
\dot z=-\kappa M\nabla_zL_{p_\epsilon(s)}(z),\qquad
\dot s=(1-2\epsilon)s(1-s)
\bigl[\delta-(g_{12}(z)-g_{23}(z))\bigr].
\]
For $0\le\epsilon<p_-$, $|\delta|<1$, and every
$\kappa>0$, both endpoint policy--code pairs remain locally asymptotically
stable. Each central global optimizer at $p_\delta$, paired with
$s_\delta=(p_\delta-\epsilon)/(1-2\epsilon)$, is a saddle with one unstable
direction (\cref{prop:finite-rate}). This establishes local persistence
under simultaneous learning; the global basins in \cref{thm:trap} rely on
exact adaptation. Both results concern the tied three-feature model,
leaving the role of its feature statistics and controller constraints to
the constructions that follow.

\section{Additional Experimental Figures and Tables}
\label{app:moved-experiments}

\paragraph{Capacity and controller flexibility.}
\label{sec:exp-capacity}
The frozen-encoder gains in \cref{sec:exp-data-access} motivate testing
what a different controller can recover from the same code. Fixed-code
probes show that independent linear and nonlinear controllers can retain
return reversal. Balancing their data can correct the ranking without
improving the fixed policy's return, because its preferred branch can
lose performance (\cref{app:decoder-probe}). We then compare controller
flexibility with representation capacity when both are learned. With
16, 32, and 64 features, increasing rank improves neural PPO return
(\cref{fig:highdim-capacity}), whereas linear controllers outperform MLPs
in 69 of 70 paired comparisons. Within this budget, capacity helps more
consistently than controller flexibility, but reversal does not persist
across all tested scales (\cref{app:highdim-learning}).

\begin{figure}[!ht]
\centering
\includegraphics[width=\textwidth]{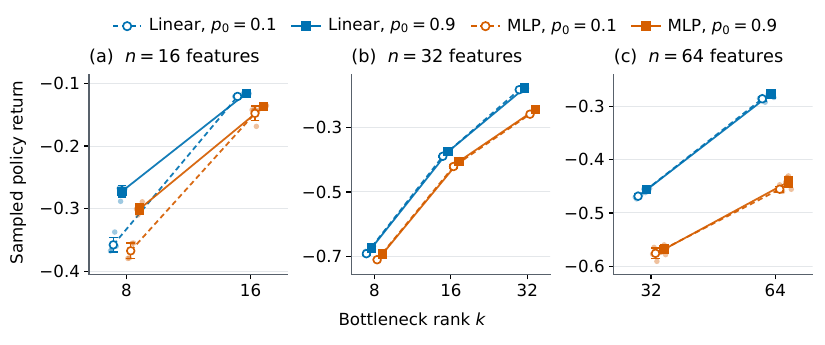}
\caption{Return increases with rank, while the linear controller gives higher
means than the MLP. Small points show five seeds; large markers give
means with sample SD. Conditions are offset horizontally, and panels
use separate return scales.}
\label{fig:highdim-capacity}
\end{figure}

\begin{figure}[!ht]
\centering
\includegraphics[width=\textwidth,trim=0 7bp 0 3bp,clip]{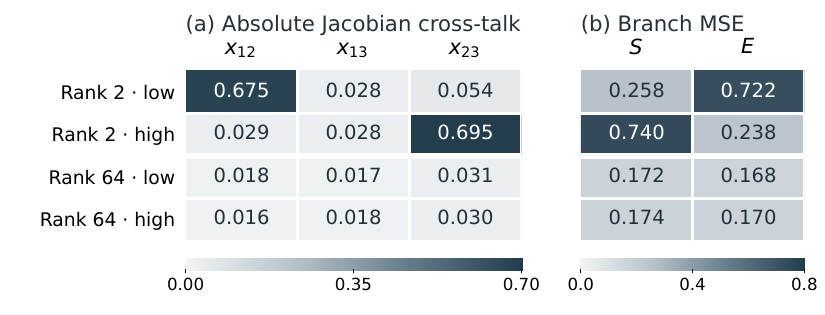}
\caption{Initial preference changes the allocation of cross-talk and control
error under rank-two PPO. Means use all 100 seeds per condition; cross-talk
measures absolute Jacobian responses.}
\label{fig:e3-bridge}
\end{figure}

\begin{table}[!ht]
\centering\small
\setlength{\tabcolsep}{3pt}
\renewcommand{\arraystretch}{1.14}
\caption{Low-start rank-two PPO learns lower-return behavior.
Each condition uses 100 seeds; brackets give 95\% seed-bootstrap intervals.}
\label{tab:ppo-results}
\begin{tabular*}{\textwidth}{@{\extracolsep{\fill}}rrcrr@{}}
\toprule
Rank & Initial \(p(E)\) & Final \(p(E)<0.5\) & Mean \(E-S\) gap & Return [95\% CI] \\
\midrule
2 & $0.095$ & 96/100 & $-0.313$ & $-0.2511\ [-0.2557,-0.2453]$ \\
2 & $0.86$ & 0/100 & $0.651$ & $-0.1205\ [-0.1215,-0.1194]$ \\
64 & $0.095$ & 0/100 & $0.153$ & $-0.0260\ [-0.0262,-0.0257]$ \\
64 & $0.86$ & 0/100 & $0.154$ & $-0.0281\ [-0.0283,-0.0278]$ \\
\bottomrule
\end{tabular*}
\end{table}

\begin{figure}[!ht]
\centering
\includegraphics[width=\textwidth,trim=0 15bp 0 8bp,clip]{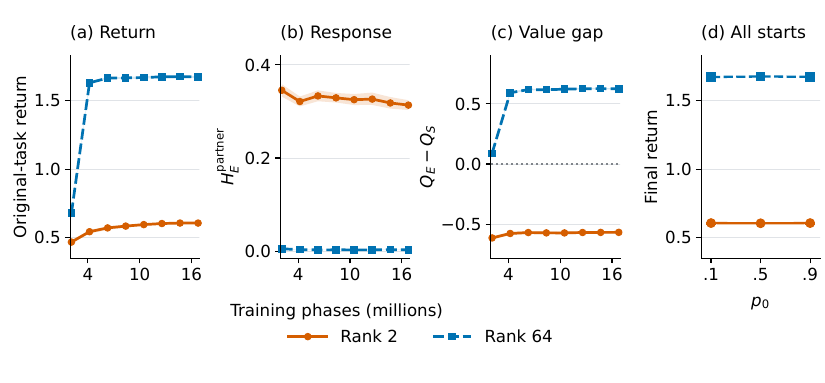}
\caption{Rank-two agents retain lower returns, stronger partner responses,
and negative advance--exit rollout gaps. Training at \(p_0=0.1\) (a--c)
and endpoints across starts (d) use 20 seeds per condition and 95\% intervals.}
\label{fig:sequential-baseline}
\end{figure}

\begin{table}[!ht]
\centering\small
\setlength{\tabcolsep}{3pt}
\renewcommand{\arraystretch}{1.14}
\caption{Temporary exposure changes improve narrow return after Hub training
resumes. Gains are paired differences; brackets give 95\% intervals over
20 training seeds.}
\label{tab:sequential-results}
\begin{tabular*}{\textwidth}{@{\extracolsep{\fill}}rlrrr@{}}
\toprule
Rank & Encoder & Hub-only fork & Uniform, then Hub & Paired gain \\
\midrule
2 & Trainable & $0.468\ [0.430,0.501]$ & $1.234\ [1.209,1.255]$ & $0.766\ [0.716,0.815]$ \\
2 & Frozen & $-0.326\ [-0.369,-0.282]$ & $1.066\ [1.033,1.098]$ & $1.392\ [1.325,1.457]$ \\
64 & Trainable & $1.673\ [1.663,1.681]$ & $1.676\ [1.668,1.684]$ & $0.004\ [-0.004,0.011]$ \\
\bottomrule
\end{tabular*}
\end{table}

\section{Feature Interference and the Four Conditions}
\label{app:definition}

\Cref{def:trap} uses the same policy comparison to connect a geometric
allocation to a reversal and a loss of adapted return. Two details matter
for that connection. The interference term must follow from the encoder and controller,
and the named feature pairs must account for the reversal.

\subsection{Separating Cross-Feature Error from Self-Response and Noise}
\label{app:interference-reference}

Let \(G\subseteq[n]\) be the active features after action \(a\), with
\(r_{a,i}=\Pr(i\in G\mid a)\) and
\(\omega_{a,ij}=\Pr(i,j\in G\mid a)\). Conditional on the action and mask,
the amplitudes have mean zero and covariance \(I_n\). The encoder and
controller are
\begin{equation*}
h=\sum_{j\in G}v_jZ_j+\nu,\qquad
\widehat Z_i=w_i^\top h,\qquad \|v_i\|=1,
\end{equation*}
where \(\nu\) is independent, centered, and has covariance
\(\sigma^2I\). Reward subtracts the sum of squared errors on active
features. Writing \(T_{ij}=w_i^\top v_j\), expansion of each squared error
and cancellation of the centered cross terms gives
\begin{align*}
D_a(\theta)&=B_a(\theta)+I_a^{\rm all}(\theta),\\
B_a(\theta)&=\sum_i r_{a,i}
\bigl[(T_{ii}-1)^2+\sigma^2\|w_i\|^2\bigr],\\
I_a^{\rm all}(\theta)&=\sum_{i<j}\omega_{a,ij}
\bigl[T_{ij}^2+T_{ji}^2\bigr].
\end{align*}
For a set \(P\) of unordered feature pairs, its contribution is
\begin{align}\label{eq:I_a^p}
I_a^P(\theta)=\sum_{\{i,j\}\in P}\omega_{a,ij}
\bigl[(w_i^\top v_j)^2+(w_j^\top v_i)^2\bigr].    
\end{align}
Consequently \(q_a^\perp(\theta)=b_a-B_a(\theta)\) removes exactly
the cross-feature responses. In particular, it does not rename an
arbitrary difference between two parameter choices as interference.

This reference is realizable without changing the self-response or the
noise variance. Take an orthonormal family \(e_0,e_1,\ldots,e_n\) in
\(\R^{n+1}\), and set
\begin{equation*}
v_i^\perp=e_i,\qquad
w_i^\perp=T_{ii}e_i+
\sqrt{\|w_i\|^2-T_{ii}^2}\,e_0.
\end{equation*}
The square root is real by Cauchy--Schwarz. These vectors satisfy
\((w_i^\perp)^\top v_i^\perp=T_{ii}\),
\(\|w_i^\perp\|=\|w_i\|\), and
\((w_i^\perp)^\top v_j^\perp=0\) for \(j\ne i\).
Thus the same activation law, amplitudes, reward and noise level give
exactly \(q^\perp(\theta)\). In the tied case \(T_{ii}=\|w_i\|=1\),
so \(w_i^\perp=v_i^\perp=e_i\) and \(n\) dimensions suffice.
The construction preserves the linear controller output rule; it does not claim
membership in the original capacity-constrained class or impose a new
optimizer on the reference.

\subsection{Why the Witness Must Account for the Same Policy Comparison}
\label{app:definition-attribution}

For \(\Delta x=y-x^\star\), put
\(\alpha_{ij}=\sum_a\Delta x_a\omega_{a,ij}\).
Every pair in the fixed witness set \(P\) has \(\alpha_{ij}>0\), and
its contribution to the policy contrast is
\begin{equation*}
(\Delta x)^\top I^P(\theta)
=\sum_{\{i,j\}\in P}\alpha_{ij}(T_{ij}^2+T_{ji}^2).
\end{equation*}
Allocation requires \(\langle v_i,v_j\rangle\ne0\) for every
pair in \(P\) at every optimum. This rules out attributing
controller cross-talk between orthogonal features to their superposition.
Reversal requires the displayed total to exceed the magnitude of the
negative deployed policy contrast. A pattern whose cross-feature terms
are too small to account for that contrast cannot meet the requirement.
This identifies which feature signals contribute the adverse return;
it does not assert that each encoder overlap alone causes the corresponding
controller response. The definition concerns tied controllers or unrestricted
jointly optimized linear controllers. An arbitrary restriction on the controller
could itself enforce erroneous mixing even with an orthogonal encoder,
which would require a separate attribution argument.

The reference \(q^\perp\) and this attribution check play distinct roles.
The former verifies that removing all interference favors \(y\). The
latter verifies that the named nonorthogonal pairs suffice to account for
the adverse comparison. Neither implies the other: pairs outside \(P\)
can have negative \(\alpha_{ij}\), so removing their interference can
reduce the contrast. We retain both inequalities. The partial sum is an
accounting of terms in the actual squared-error loss, and is not asserted
to be the return of a feasible encoder and controller obtained by editing one part of
a Gram matrix.

For example, append two features with fixed identical unit directions,
orthogonal to the original code, that are always active under both
branches. With a tied controller this adds \(2+2\sigma^2\) to both branch
distortions. It changes neither the actor gap nor any comparison of adapted
returns. These features collide at every optimum, yet their pair has
\(\alpha_{ij}=0\) for every policy comparison and cannot witness
\cref{def:trap}. This example explains why a collision and a reversed
ordering, stated independently, would not identify the mechanism.

The definition is deliberately an attribution within the stated linear
feature model. It requires a common adequate pattern across all optima;
it does not characterize every form of representation-induced bad
equilibrium. The actor and feasible policy set are fixed for each model:
the replicator and the stated floor and entropy variants update through
the deployed values, and each stability claim uses its specified dynamics.
An actor that ignores these values would require a separate account of
the feedback. Multiple pairs may be needed, and the actor may continue to
be trapped after any one pair is removed. Nor does the definition itself
prove that the allocation varies with the action distribution. The solved phase diagrams
establish this dependence, while feedback establishes that nearby policies
return as their representations adapt. For two actions the contrast is
just a positive multiple of \(q_E-q_S\), including when the feasible
action distributions satisfy a visitation floor. For more actions it compares the
same two policies in all four conditions, rather than an unrelated pair
of actions and a separate return benchmark.

\subsection{Witnesses in the Constructions}
\label{app:definition-witnesses}

In the tied three-feature model, choose \(P=\{\{1,2\}\}\),
\(x^\star=(1,0)\), and \(y=(0,1)\). The pair's co-activation
probability increases by \(1/2\), and every low-phase optimum has
\(T_{12}^2+T_{21}^2=2\). Hence
\begin{equation*}
(\Delta x)^\top q=\delta-1,\qquad
(\Delta x)^\top q^\perp=\delta,\qquad
(\Delta x)^\top I^P=1.
\end{equation*}
For \(0<\delta<1\), allocation and both reversal inequalities follow.
\Cref{thm:trap} supplies feedback and the adapted return gain \(\delta\).
With the floor, take \(x^\star=(1-\epsilon,\epsilon)\) and
\(y=(\epsilon,1-\epsilon)\). All three contrasts and the adapted
return gain are multiplied by \(1-2\epsilon>0\), while
\(\epsilon<p_-\) keeps the same pattern optimal. Thus
\cref{cor:policy-floor} satisfies the same definition with positive
visitation of both actions.

In the matching-complement model, take \(P=M_S\) and the two pure
policies. Each pair in \(M_S\) has \(\alpha_{ij}=1/N\), and every
low-policy optimum collapses all \(d\) of these pairs. Their collective
contribution is \(2d/N=1/(d-1)\), whereas the separated gap is
\(\delta\). This proves reversal for the full window in
\cref{thm:matching-complement}; a single pair would in general be
insufficient. In the codimension-one lift, the single pair \(\{1,2\}\)
has contribution \(c_d\), so the same verification applies with
\(0<\delta<c_d\), as in \cref{cor:codimension-one-lift}.

For the untied model at fixed \(\eta>0\),
\cref{prop:untied-endpoint} gives
\(w_i=v_i/(1+\eta)\) at every endpoint optimum. At the low endpoint
the pair \(\{1,2\}\) therefore contributes
\(s_\eta=(1+\eta)^{-2}\) to the contrast. Each active feature has
the same self-response and noise cost, and each branch activates two
features, so \((\Delta x)^\top q^\perp=\delta\).
The three contrasts above become \(\delta-s_\eta,\delta,s_\eta\),
respectively. \Cref{thm:untied} proves feedback and suboptimality for
\(0<\delta<s_\eta\), completing the four conditions without any
large-noise assumption.

For the small-entropy regime in \cref{prop:entropy-trap}, let \(p_L\)
and \(p_H\) be its low and high attracting values of \(p\). They lie in
the respective exact collision phases. Taking the same pair \(\{1,2\}\)
and \(y=(1-p_H,p_H)\), the three tied-model contrasts above are
multiplied by \(p_H-p_L>0\). The entropy proposition supplies feedback
for the regularized actor and a strictly higher environmental return at
\(p_H\), so these interior policies also meet the four conditions.

Finally, \cref{prop:finite-rate} concerns a joint policy--representation
flow. Its endpoint representations and return comparisons are the same,
but its trajectories need not satisfy exact response away from equilibrium.
It proves local persistence of these equilibria under a different learning
dynamic, rather than silently extending the exact-response quantifier in
\cref{def:trap}. Similarly, \cref{sec:setting} gives geometric-free
conditions on representation--actor feedback. Those conditions alone do
not supply the feature attribution required here.

\section{Weighted-Frame Solution}
\label{app:frame}

\subsection{Full Derivation of the Illustrative Example}
\label{app:illustration-derivation}

Recall that $p=\pi(E)$ and that the observation is
$o_i=Z_i\mathbf1\{i\in G\}$. The vector $Z$ is centered, independent
of $(B,G)$, and satisfies $\E[ZZ^\top]=I_3$. The encoder has unit columns
$v_i\in\R^2$, and the tied controller takes the terminal action
$\widehat Z=V^\top(Vo+\nu)$, where independent noise has mean zero
and covariance $\sigma^2I_2$. Only active coordinates contribute to
the squared control loss. The root reward is zero, so the terminal
reward is the episode return.

For an active pair $G=\{i,j\}$, put $c=v_i^\top v_j$. Unit column
lengths give
\[
\widehat Z_i-Z_i=cZ_j+v_i^\top\nu,
\qquad
\widehat Z_j-Z_j=cZ_i+v_j^\top\nu.
\]
Independence from the active pair preserves the stated moments of $Z$;
independence of the centered noise makes the signal--noise cross terms
vanish. Thus, with $g_{ij}=c^2$,
\[
\E\!\left[(\widehat Z_i-Z_i)^2+(\widehat Z_j-Z_j)^2
\mid G=\{i,j\}\right]=2(g_{ij}+\sigma^2).
\]
Averaging over the two equally likely pairs in each branch recovers
the losses stated in \cref{sec:exact}:
\begin{equation}
\begin{aligned}
D_S(V)
&=\tfrac12\,2(g_{13}+\sigma^2)
 +\tfrac12\,2(g_{23}+\sigma^2)
 =g_{13}+g_{23}+2\sigma^2,\\
D_E(V)
&=\tfrac12\,2(g_{12}+\sigma^2)
 +\tfrac12\,2(g_{13}+\sigma^2)
 =g_{12}+g_{13}+2\sigma^2.
\end{aligned}
\label{eq:branch-loss}
\end{equation}
With $V$ fixed, the actual branch return is
$q_B(V)=b_B-D_B(V)$. Averaging the losses under the policy gives
\[
(1-p)D_S(V)+pD_E(V)
=2\sigma^2+p g_{12}+g_{13}+(1-p)g_{23}
=2\sigma^2+L_p(V).
\]
The noise contribution is constant over the feasible unit-column codes,
so minimizing expected control loss at fixed $p$ is equivalent to
minimizing $L_p$.

At $p=0$, $L_0=g_{13}+g_{23}$ is nonnegative and attains zero.
Every minimizer therefore has both $v_1$ and $v_2$ orthogonal to
$v_3$. Its orthogonal complement is one-dimensional, so every optimum
satisfies $v_1=\pm v_2\perp v_3$. A representative is
$v_1=v_2=e_1$, $v_3=e_2$. The corresponding squared overlaps are
$(g_{12},g_{13},g_{23})=(1,0,0)$, giving
$D_S=2\sigma^2$ and $D_E=1+2\sigma^2$.
At $p=1$, minimizing $L_1=g_{12}+g_{13}$ to zero similarly forces
$v_2=\pm v_3\perp v_1$ at every optimum. Choosing
$v_1=e_1$, $v_2=v_3=e_2$ gives squared overlaps $(0,0,1)$ and
reverses the two distortions. Column signs and a common rotation do
not change these losses.

For the numerical illustration, set $b_S=0$, $b_E=0.5$, and
$\sigma=0$. Under the $S$-adapted code, either pair on $S$ is
recovered exactly. On $E$, the pair $\{1,2\}$ instead produces
$h=(Z_1+Z_2)e_1$ and
$\widehat Z_1=\widehat Z_2=Z_1+Z_2$. Its expected squared error is
$\E[Z_1^2+Z_2^2]=2$. The other pair, $\{1,3\}$, has zero error,
so the branch average is $D_E=1$. The $E$-adapted code exchanges
the role of the colliding pair, yielding the quantities below.

\begin{center}
\begin{tabular}{lrrrr}
\toprule
Code fitted on & $D_S$ & $D_E$ & $q_S$ & $q_E$\\
\midrule
$S$ only & $0$ & $1$ & $0$ & $-0.5$\\
$E$ only & $1$ & $0$ & $-1$ & $0.5$\\
\bottomrule
\end{tabular}
\end{center}

Each row holds one code fixed while comparing both actions. Comparing
the two separately adapted policies instead uses $q_S$ from the first
row and $q_E$ from the second: their returns are $0$ and $0.5$ at
the same capacity. For general parameters, write
$L^\star(p)=\min_V L_p(V)$, with the minimum over the same unit-column
class. Since $\delta=b_E-b_S$, adaptation gives
\begin{equation}
\begin{aligned}
J^\star(p)
&=\max_V\bigl[(1-p)q_S(V)+p q_E(V)\bigr]\\
&=b_S+p\delta-2\sigma^2-L^\star(p).
\end{aligned}
\label{eq:scalar-adapted-return}
\end{equation}
Changing the root action in one episode does not perform this
reoptimization; it changes the branch evaluated through the current
code. These comparisons establish the example's return reversal and
adapted-return difference. Their stability under learning is analyzed
separately in \cref{sec:dynamics}.

To solve for the code between the endpoints, fix $u=g_{12}$ and minimize
over $v_3$. Its contribution is the Rayleigh quotient of
$v_1v_1^\top+(1-p)v_2v_2^\top$, whose trace is $2-p$ and determinant
is $(1-p)(1-u)$. Selecting the smaller eigenvalue gives
\begin{equation}
L^\star(p)=\min_{0\le u\le1}f_p(u),\qquad
f_p(u)=pu+\frac{2-p-\sqrt{p^2+4(1-p)u}}{2}.
\label{eq:scalar-frame}
\end{equation}
Every $u\in[0,1]$ is realizable by two unit vectors, so the reduction
is exact. For $0<p<1$, strict convexity makes $u=1$ the unique
minimum exactly when
\[
f_p'(1)=p-\frac{1-p}{2-p}\le0,
\qquad\text{equivalently}\qquad
p\le p_-:=\frac{3-\sqrt5}{2}.
\]
This includes equality. In this range, attaining the minimum also
requires the optimizing $v_3$, forcing $v_1=\pm v_2\perp v_3$.
At $p=0$, $f_0(u)=1-\sqrt u$ has the same unique minimum.
Exchanging features 1 and 3 and replacing $p$ by $1-p$ gives the
high phase, beginning at $p_+=(\sqrt5-1)/2$. The interior stationary
point gives the middle phase's three distinct lines.
\Cref{lem:weighted-frame} supplies the full squared Gram entries and
equality cases for general pair weights, recovering
\cref{thm:frame} upon specialization.

\subsection{Optimal Geometry for Arbitrary Pair Weights}
\label{app:weighted-frame-proof}

To determine every optimum, the minimum loss alone is not enough: we
also need the equality cases and their behavior when a pair receives
zero weight. The following lemma solves the geometric problem for
arbitrary positive pair weights. The policy-dependent objective is then
a direct specialization, with its zero-weight endpoints handled separately.

\begin{lemma}[Weighted Three-Line Frame]
\label{lem:weighted-frame}
For \(A,B,C>0\), consider
\begin{equation}
\min_{\|v_i\|=1,\ v_i\in\R^2}
A g_{12}+B g_{13}+C g_{23}.
\label{eq:weighted-frame}
\end{equation}
If
\begin{equation*}
A\le\frac{BC}{B+C},
\end{equation*}
then every minimizer has \(v_1=\pm v_2\perp v_3\) and the optimum is \(A\).
The two cyclic statements hold for \(B\) and \(C\). If none of the three
conditions holds, then
\begin{equation}
L^\star(A,B,C)
=\frac{A+B+C}{2}
-\frac14\left(\frac{AB}{C}+\frac{AC}{B}+\frac{BC}{A}\right),
\label{eq:weighted-value}
\end{equation}
and all minimizers have
\begin{align}
g_{12}&=\frac12-\frac14\left(
\frac BC+\frac CB-\frac{BC}{A^2}\right),\label{eq:weighted-g12}\\
g_{13}&=\frac12-\frac14\left(
\frac AC+\frac CA-\frac{AC}{B^2}\right),\label{eq:weighted-g13}\\
g_{23}&=\frac12-\frac14\left(
\frac AB+\frac BA-\frac{AB}{C^2}\right).\label{eq:weighted-g23}
\end{align}
The squared Gram geometry is unique in every phase, including equality in a
phase condition.
\end{lemma}

\begin{proof}
Represent each unoriented line by an angle \(\beta_i\), fix
\(\beta_1=0\). Since
\(\cos^2\beta=(1+\cos2\beta)/2\), minimizing
\cref{eq:weighted-frame} is equivalent to minimizing
\begin{equation*}
F(\beta_2,\beta_3)=A\cos2\beta_2+B\cos2\beta_3+C\cos(2\beta_2-2\beta_3).
\end{equation*}
For fixed \(\beta_3\), exact minimization over \(\beta_2\) gives
\begin{equation}
\phi(t)=Bt-\sqrt{A^2+C^2+2ACt},
\qquad t=\cos2\beta_3\in[-1,1].
\label{eq:phi}
\end{equation}
Where the denominator is nonzero,
\begin{equation*}
\phi'(t)=B-\frac{AC}{\sqrt{A^2+C^2+2ACt}},
\qquad
\phi''(t)=\frac{A^2C^2}{(A^2+C^2+2ACt)^{3/2}}>0.
\end{equation*}
The profiled objective is strictly convex wherever it is smooth, so its
minimizing \(t\) is unique. The only possible nonsmooth endpoint occurs
when \(A=C\) and is handled below.

The endpoint \(t=1\) is optimal exactly when
\(B\le AC/(A+C)\), yielding \(v_1=\pm v_3\perp v_2\). At \(t=-1\), if
\(A<C\), the left endpoint is optimal exactly when
\begin{equation*}
B\ge\frac{AC}{C-A}
\quad\Longleftrightarrow\quad
A\le\frac{BC}{B+C},
\end{equation*}
and the geometry is \(v_1=\pm v_2\perp v_3\). If \(A>C\), the analogous
condition is \(C\le AB/(A+B)\), giving the \(2\)-\(3\) collision. These
positive-weight phase conditions are mutually exclusive. At equality,
strict convexity still gives a unique minimizing endpoint, and the vector
used to eliminate \(\beta_2\) is nonzero. Thus \(2\beta_2\) is unique modulo \(2\pi\). If \(A=C\), the square root in \cref{eq:phi} has a
cusp at \(t=-1\), but that endpoint cannot minimize, because for small
\(\epsilon>0\),
\[
\phi(-1+\epsilon)-\phi(-1)=B\epsilon-A\sqrt{2\epsilon}<0.
\]

If no endpoint condition holds, the minimizer is interior and is given by
\begin{equation*}
t^\star=
\frac{A^2C^2/B^2-A^2-C^2}{2AC}\in(-1,1).
\end{equation*}
Substitution into \cref{eq:phi}, followed by restoring the constant
\((A+B+C)/2\) and the factor \(1/2\), gives
\cref{eq:weighted-value}. The minimizing \(\beta_2\) is recovered from
\begin{equation*}
(\cos2\beta_2,\sin2\beta_2)
=-\frac{(A+C\cos2\beta_3,\ C\sin2\beta_3)}{AC/B}.
\end{equation*}
The two choices of \(\sin2\beta_3\) differ by a common reflection. Therefore the
minimizer is unique modulo a common orthogonal transformation and column signs.
The derivative of the objective with respect to a pair weight is that pair's
squared inner product. The envelope theorem therefore gives
\cref{eq:weighted-g12,eq:weighted-g13,eq:weighted-g23} from the derivatives
of \cref{eq:weighted-value}.
\end{proof}

\begin{proof}[Proof of \cref{thm:frame}]
Set \(A=p\), \(B=1\), and \(C=1-p\). The \(1\)-\(2\) collision condition is
\begin{equation*}
p\le\frac{1-p}{2-p}
\quad\Longleftrightarrow\quad
p\le\frac{3-\sqrt5}{2}=p_-.
\end{equation*}
The \(2\)-\(3\) collision condition is
\begin{equation*}
1-p\le\frac{p}{1+p}
\quad\Longleftrightarrow\quad
p\ge\frac{\sqrt5-1}{2}=p_+.
\end{equation*}
The remaining collision condition would require \(1\le p(1-p)\), which is
impossible. In the central phase,
\begin{equation*}
\frac p{1-p}+\frac{1-p}{p}=\frac1{p(1-p)}-2,
\end{equation*}
so \cref{eq:weighted-value} reduces to \cref{eq:optimal-loss}.
Substitution into
\cref{eq:weighted-g12,eq:weighted-g13,eq:weighted-g23} gives
\cref{eq:optimal-overlaps}. At \(p=0\), the
objective forces \(v_1,v_2\perp v_3\), hence \(v_1=\pm v_2\) in two
dimensions; \(p=1\) is symmetric. The two zero-weight endpoints therefore retain the same uniqueness of squared
Gram entries as the adjacent positive-weight phases.
\end{proof}

The derivative of \cref{eq:optimal-loss} gives \cref{eq:D-piecewise}. In the
central phase, with \(\chi=p(1-p)\),
\begin{equation*}
\frac{dL^\star}{dp}
=\frac{1-2p}{4}(\chi^{-2}-1),
\end{equation*}
and a second derivative gives \cref{eq:D-derivative}. At \(p_-,p_+\), the central formula matches the outer values \(+1\) and
\(-1\), and the squared Gram entries also match. Thus \(L^\star\) is
continuously differentiable and \(\cD\) is continuous across both
transitions. Its derivative \(\cD'=L^{\star\prime\prime}\) jumps from
zero in each outer phase to the central one-sided value
\(-\tfrac52(2+\sqrt5)\).

\subsection{Task Blocks under Partial Observation}

The scalar construction also describes a larger linear observation class
when the encoder can remove nuisance variation and preserve each task block.
The condition below ensures that nuisance removal imposes no additional
relations among the effective feature directions. The encoder is bounded,
but the class has no common bound on its operator norm.

\begin{proposition}[Coactivation Reduction under Partial Observation]
\label{prop:partial-observation-reduction}
Let \(\mathcal H\) be a real Hilbert space, let
\(\mathcal N\subset\mathcal H\) be a closed nuisance subspace, and let linear
maps \(\Phi_i:\R^{k_i}\to\mathcal H\), \(i\in[m]\), be task-injective modulo
\(\mathcal N\):
\begin{equation}
\sum_{i=1}^m\Phi_i u_i\in\mathcal N
\quad\Longrightarrow\quad u_i=0\ \text{for every }i.
\label{eq:task-injective-observation}
\end{equation}
Let \(G\subseteq[m]\) be a random active set. Let the joint block vector
\((Z_i)_{i=1}^m\), with \(Z_i\in\R^{k_i}\), be independent of \(G\).
Assume its blocks are centered and pairwise uncorrelated, with
\(\E[Z_iZ_i^\top]=I_{k_i}\). The terminal observation is
\[
y=\sum_{i\in G}\Phi_iZ_i+n,\qquad n\in\mathcal N\quad\text{almost surely}.
\]
Fix \(k\ge\max_i k_i\). Let \(V:\mathcal H\to\R^k\) be any bounded linear
encoder with \(V\mathcal N=\{0\}\) and \(V_i:=V\Phi_i\) satisfying
\(V_i^\top V_i=I_{k_i}\). With \(h=Vy+\nu\),
\(\widehat Z_i=V_i^\top h\), and independent noise \(\E[\nu]=0\),
\(\E[\nu\nu^\top]=\sigma^2I_k\), put
\(\omega_{ij}=\Pr(i,j\in G)\). Then
\begin{equation}
\E\!\left[\sum_{i\in G}\|\widehat Z_i-Z_i\|_2^2\right]
=\sigma^2\sum_i k_i\Pr(i\in G)
 +2\sum_{i<j}\omega_{ij}\|V_i^\top V_j\|_F^2.
\label{eq:fusion-frame-reduction}
\end{equation}
Moreover, as \(V\) varies, every tuple of isometries
\(V_i\in\R^{k\times k_i}\) is realizable. Hence the part of the population
objective depending on the encoder is exactly a coactivation-weighted
fusion-frame problem. For a finite
action mixture, policy enters it through
\(\omega_{ij}(x)=\sum_a x_a\Pr(i,j\in G\mid a)\). In the scalar specialization
\(m=3\), \(k=2\), and with the branch laws of \cref{sec:exact}, this is
\cref{eq:frame-objective}; thus \cref{thm:frame,thm:trap,cor:policy-floor}
hold verbatim for the effective directions \(V\Phi_i\).
\end{proposition}

\begin{proof}
Put \(V_i=V\Phi_i\). Since \(V\mathcal N=\{0\}\) and \(V_i^\top V_i=I_{k_i}\), for
\(i\in G\),
\[
\widehat Z_i-Z_i=\sum_{j\in G\setminus\{i\}}V_i^\top V_jZ_j
 +V_i^\top\nu.
\]
Pairwise uncorrelated blocks and independent isotropic noise give
\[
\E[\|\widehat Z_i-Z_i\|_2^2\mid G]
=\sum_{j\in G\setminus\{i\}}\|V_i^\top V_j\|_F^2+\sigma^2k_i.
\]
The sum over active \(i\) counts every unordered pair twice. Its expectation
over \(G\) gives \cref{eq:fusion-frame-reduction}.

Let \(\Pi:\mathcal H\to\mathcal H/\mathcal N\) be the quotient map. Condition
\cref{eq:task-injective-observation} makes
\(T_0:(u_i)_i\mapsto\sum_i\Pi(\Phi_i u_i)\) injective. Given arbitrary
isometries \((V_i)_i\), define on the finite-dimensional range of \(T_0\)
\[
V_0\,T_0((u_i)_i):=\sum_iV_i u_i.
\]
Injectivity makes this definition unambiguous, and finite dimension makes
\(V_0\) bounded. The closedness of \(\mathcal N\) makes the quotient a
Hilbert space, so \(V_0\) extends by zero on the orthogonal complement of
\(\operatorname{range}(T_0)\). Set \(V=V_0\Pi\). Then
\(V\mathcal N=0\) and \(V\Phi_i=V_i\), proving realizability. If
\cref{eq:task-injective-observation} fails, every realizable tuple obeys the
induced nonzero linear relation, so the unrestricted fusion-frame reduction
need not hold.
\end{proof}

\subsection{Traps across Feature Counts and Representation Dimensions}
\label{app:dimension-traps}

We prove \cref{thm:dimension-trap} by constructing branch distributions
for which every representation fitted to $S$ ranks the branches
oppositely to their separately adapted returns. Most dimensions admit a common
partition construction. Two additional constructions complete the
equal-frequency result. We then show why the excluded dimensions cannot
support the stated trap and why exact zeros in the pair distributions
are unnecessary.

\subsubsection{Model and a sufficient endpoint condition}
\label{app:dimension-endpoints}

Fix $d_{\mathrm{in}}\ge2$, $d_{\mathrm{out}}\ge1$, and $\sigma^2\ge0$. A branch $B\in\{S,E\}$
draws an unordered pair $G=\{i,j\}$ with probability
$\omega_{B,ij}$. Conditional on $(B,G)$, the active feature values have
zero mean and identity covariance. The encoder and tied controller are
$h=Vo+\nu$ and $\widehat Z=V^\top h$, where
$V\in\mathcal V_{d_{\mathrm{in}},d_{\mathrm{out}}}:=(\mathbb S^{d_{\mathrm{out}}-1})^{d_{\mathrm{in}}}$ has unit columns and
$\nu$ is independent centered noise of covariance $\sigma^2I_{d_{\mathrm{out}}}$.
The observation contains the two active values and is zero elsewhere.
As in \cref{sec:exact}, the terminal reward subtracts their squared
control errors from $b_B$. Thus, writing
$g_{ij}(V)=\langle v_i,v_j\rangle^2$, we have
\begin{equation}
D_B(V)=2\sigma^2+2\sum_{i<j}\omega_{B,ij}g_{ij}(V),
\qquad q_B(V)=b_B-D_B(V).
\label{eq:dimension-branch-loss}
\end{equation}
Indeed, for an active pair $\{i,j\}$, the two errors are
$\langle v_i,v_j\rangle Z_j+v_i^\top\nu$ and
$\langle v_i,v_j\rangle Z_i+v_j^\top\nu$.
Their expected squared sum is $2g_{ij}+2\sigma^2$.
For $p=\pi(E)$, let $\cM_0(p)$ minimize
$R_p=(1-p)D_S+pD_E$ over the compact space $\mathcal V_{d_{\mathrm{in}},d_{\mathrm{out}}}$.
The exact-response policy dynamics are
\[
\dot p=p(1-p)\Delta_q(V),
\qquad \Delta_q(V)=q_E(V)-q_S(V),
\qquad V(t)\in\cM_0(p(t)).
\]
All stability statements below concern these dynamics, uniformly over
measurable choices of global minimizers.

Removing cross-feature error gives $q_B^\perp=b_B-2\sigma^2$ for
every code, so $q_E^\perp-q_S^\perp=\delta:=b_E-b_S$.
For a fixed set $P$ of feature pairs, its contribution is
$I_B^P(V)=2\sum_{\{i,j\}\in P}\omega_{B,ij}g_{ij}(V)$.
These identities allow the static comparisons in \cref{def:trap} to
be checked directly.

\begin{lemma}[Endpoint condition for a trap]
\label{lem:dimension-endpoint}
Suppose $\delta>0$ and $\min_V D_S(V)=\min_V D_E(V)=d_0$.
Let $P$ be a fixed nonempty set with
$\omega_{E,ij}>\omega_{S,ij}$ for every $\{i,j\}\in P$.
If every $V\in\cM_0(0)$ satisfies
\[
g_{ij}(V)>0\quad(\{i,j\}\in P),\qquad
\Delta_q(V)<0,\qquad
I_E^P(V)-I_S^P(V)>-\Delta_q(V),
\]
then $p=0$ is a self-confirming superposition trap witnessed by $P$
and $p=1$. Both endpoints are locally exponentially attracting,
exact-response trajectories exist locally, and
$J^\star(1)-J^\star(0)=\delta$.
\end{lemma}

\begin{proof}
The displayed conditions give allocation and reversal because the
interference-removed gap is $\delta>0$. Equality of the two separate
minimum losses gives the stated adapted-return difference.
At every $V\in\cM_0(1)$, we also have
$D_S(V)\ge d_0=D_E(V)$ and hence $\Delta_q(V)\ge\delta>0$.

We prove feedback without requiring a unique or continuous optimizer.
If $p_j\to p_0$ and $V_j\in\cM_0(p_j)$, every convergent subsequence
of $(V_j)$ has its limit in $\cM_0(p_0)$, by compactness and continuity
of $R_p(V)$. The strict signs on the compact endpoint optimizer sets
therefore extend uniformly to neighborhoods of $0$ and $1$.
In particular, there are $a,c>0$ such that
$\Delta_q(V)\le-c$ whenever $0\le p\le a$ and
$V\in\cM_0(p)$. Every absolutely continuous exact-response trajectory
in this neighborhood satisfies
\[
\frac{d}{dt}\log\frac{p(t)}{1-p(t)}\le-c
\]
while $p(t)>0$. Consequently the neighborhood is forward invariant and
$p(t)\le p(0)e^{-ct}/(1-a)$. At the upper endpoint, the same argument
applies to $(1-p)/p$. These bounds imply stability and attraction in
the sense of \cref{def:policy-stability}, with common neighborhoods
and convergence bounds for all optimizer selections.

For local existence, choose a Borel measurable minimizer $V(p)$.
Such a choice exists here directly: for each fixed compact subset
$C\subset\mathcal V_{d_{\mathrm{in}},d_{\mathrm{out}}}$, the set
$\{p:C\cap\cM_0(p)\ne\varnothing\}$ is closed, since it is the
equality set of the continuous minimum losses over $C$ and
$\mathcal V_{d_{\mathrm{in}},d_{\mathrm{out}}}$. Successively taking the first intersecting member
of countable closed-ball covers with radii tending to zero produces
nested compact sets and a Borel measurable limiting minimizer.
Its scalar drift
$f(p)=p(1-p)\Delta_q(V(p))$ is measurable, negative on $(0,a]$, and
bounded away from zero on each compact subinterval of $(0,a]$.
For any $p_0\in(0,a]$, define $p(t)$ by
\[
t=\int_{p(t)}^{p_0}\frac{du}{-f(u)}.
\]
On compact interior intervals the integral and its inverse are
Lipschitz, so $p(t)$ is locally absolutely continuous and satisfies
$\dot p=f(p)$ almost everywhere. Bounded return differences give
$|f(u)|\le Cu$ near zero. The integral therefore diverges as its
lower endpoint approaches zero, yielding a forward solution for all
$t\ge0$. The selected code along it is measurable.
The constant path supplies existence at $p=0$, and the construction
near $p=1$ is identical with the direction reversed. Finally, any
admissible local trajectory has bounded drift and a limit at a finite
end of its interval. Concatenating a solution from that limit extends
it, so maximal nearby trajectories are forward complete. This proves
the required feedback condition as well as uniform attraction.
\end{proof}

\subsubsection{Partition constructions with unequal and equal frequencies}
\label{app:dimension-partitions}

The capacity constraint becomes explicit when branch $B$ draws every
pair joining different groups of a partition $\mathcal C_B$ into
exactly $d_{\mathrm{out}}$ nonempty groups, with positive probability on each such
pair and zero probability within groups. Assigning the groups to
$d_{\mathrm{out}}$ orthogonal directions attains $D_B=2\sigma^2$. Conversely, every
optimum must make all cross-group pairs orthogonal. One representative
from each group then forms an orthonormal basis of $\R^{d_{\mathrm{out}}}$, and every
other vector in a group is orthogonal to the other $d_{\mathrm{out}}-1$
representatives. Hence every optimum has
\begin{equation}
g_{ij}=
\begin{cases}
1,&i,j\text{ belong to the same group of }\mathcal C_B,\\
0,&i,j\text{ belong to different groups of }\mathcal C_B.
\end{cases}
\label{eq:dimension-partition-geometry}
\end{equation}

Choose distinct partitions $\mathcal C_S,\mathcal C_E$ and let $P$
contain the pairs lying in one $S$ group but in different $E$ groups.
This set is nonempty. Otherwise every $S$ group would lie in an $E$
group, and the common number of nonempty groups would force the two
partitions to coincide. Put
$\ell_{E\mid S}=2\sum_{\{i,j\}\in P}\omega_{E,ij}>0$.
At every $S$ optimum,
\[
D_E-D_S=I_E^P-I_S^P=\ell_{E\mid S},\qquad
\Delta_q=\delta-\ell_{E\mid S}.
\]
Both separate minima are $2\sigma^2$, so
\cref{lem:dimension-endpoint} applies whenever
$0<\delta<\ell_{E\mid S}$. Every feature appears on each branch
because each group has at least one other nonempty group.

For arbitrary $d_{\mathrm{in}}>d_{\mathrm{out}}\ge2$, an explicit unequal-frequency instance uses
one group of size $m=d_{\mathrm{in}}-d_{\mathrm{out}}+1$ and $d_{\mathrm{out}}-1$ singleton groups.
Exchange one member of the large group with a singleton to obtain
$\mathcal C_E$ from $\mathcal C_S$, and sample cross-group pairs
uniformly on each branch. The number of supported pairs is
\[
N=\binom{d_{\mathrm{in}}}{2}-\binom m2=\frac{(d_{\mathrm{out}}-1)(2d_{\mathrm{in}}-d_{\mathrm{out}})}2.
\]
Exactly $m-1=d_{\mathrm{in}}-d_{\mathrm{out}}$ pairs shared at an $S$ optimum become active on $E$,
giving
\begin{equation}
\ell_{E\mid S}=\frac{4(d_{\mathrm{in}}-d_{\mathrm{out}})}{(d_{\mathrm{out}}-1)(2d_{\mathrm{in}}-d_{\mathrm{out}})}.
\label{eq:dimension-unequal-margin}
\end{equation}
Large-group members and singletons have feature frequencies
$(d_{\mathrm{out}}-1)/N$ and $(d_{\mathrm{in}}-1)/N$, respectively. Thus every feature has positive
frequency, and this construction proves existence for all $d_{\mathrm{in}}>d_{\mathrm{out}}\ge2$
when equal frequencies are not required.

To impose frequency $2/d_{\mathrm{in}}$ on every feature, partition the features
into $d_{\mathrm{out}}$ groups of balanced sizes $s_1,\ldots,s_{d_{\mathrm{out}}}$ and choose the
cross-group probabilities to have the prescribed marginals.
For $d_{\mathrm{out}}=2$ and even $d_{\mathrm{in}}=2m$, both groups have size $m$ and probability
$1/m^2$ on each cross-group pair suffices, corresponding to total group
weight $W_{12}=1$. For $d_{\mathrm{out}}\ge3$ and $d_{\mathrm{in}}\ge5$,
balanced sizes satisfy $\max_i s_i<d_{\mathrm{in}}/2$. We construct strictly positive
symmetric group weights $W_{ij}$ with zero diagonal and row sums
$d_i:=2s_i/d_{\mathrm{in}}$. Since $\sum_i d_i=2$ and $\max_i d_i<1$, choose
\[
0<\epsilon<\min\left\{
\min_i\frac{d_i}{d_{\mathrm{out}}-1},\,
\frac{2(1-\max_i d_i)}{(d_{\mathrm{out}}-1)(d_{\mathrm{out}}-2)}\right\}.
\]
Reserve weight $\epsilon$ on every group edge. The residual degrees
$d_i'=d_i-(d_{\mathrm{out}}-1)\epsilon$ are positive and less than $T/2$, where
$T=\sum_i d_i'=2-d_{\mathrm{out}}(d_{\mathrm{out}}-1)\epsilon$.
Partition a circle of circumference $T$ into contiguous intervals
$A_i$ of lengths $d_i'$, and define
\[
W_{ij}'=\operatorname{Leb}
\{x\in A_i:x+T/2\pmod T\in A_j\}.
\]
The antipodal map preserves length and is an involution, so $W'$ is
symmetric. No interval of length less than $T/2$ contains an
antipodal pair, apart from null boundary sets, so $W'_{ii}=0$.
Its row sums are $d_i'$. Consequently
$W_{ij}=W_{ij}'+\epsilon$ for $i\ne j$ has the desired positive
entries and row sums. Distribute $W_{ij}$ uniformly among the
$s_is_j$ feature pairs between groups $i,j$. A feature in group $i$
then appears with probability
$\sum_{j\ne i}W_{ij}/s_i=2/d_{\mathrm{in}}$, and the total pair probability is one.

Since $d_{\mathrm{in}}>d_{\mathrm{out}}$, choose two groups with $s_1+s_2>2$ and transpose one
feature from each to define the $E$ law. Relabeling preserves the
uniform feature marginals and the separate optimum. There are
$s_1+s_2-2$ pairs in $P$, each with $S$ probability zero and $E$
probability $W_{12}/(s_1s_2)$. Therefore
\begin{equation}
c_{d_{\mathrm{in}},d_{\mathrm{out}}}:=\ell_{E\mid S}
=\frac{2(s_1+s_2-2)W_{12}}{s_1s_2}>0.
\label{eq:dimension-equal-partition-margin}
\end{equation}
Every $0<\delta<c_{d_{\mathrm{in}},d_{\mathrm{out}}}$ gives the required equal-frequency trap.
For $d_{\mathrm{out}}=2,d_{\mathrm{in}}=2m$, this margin is $4(m-1)/m^2$.
The cases not covered by these partitions are odd $d_{\mathrm{in}}\ge5$ with
$d_{\mathrm{out}}=2$, and $(d_{\mathrm{in}},d_{\mathrm{out}})=(4,3)$.

\subsubsection{Equal frequencies in the remaining dimensions}
\label{app:dimension-exceptions}

For $d_{\mathrm{out}}=2$ and odd $d_{\mathrm{in}}\ge5$, number the features modulo $d_{\mathrm{in}}$.
Let $S$ sample $\{j,j+1\}$ and $E$ sample $\{j,j+2\}$, each with
probability $1/d_{\mathrm{in}}$. Both supports are Hamiltonian cycles and give every
feature frequency $2/d_{\mathrm{in}}$. They are disjoint for $d_{\mathrm{in}}\ge5$ and are
related by a permutation of feature labels, so their separate minima
coincide.

Write $v_j=(\cos\theta_j,\sin\theta_j)$ and
$z_j=e^{2\mathrm i\theta_j}$. The $S$ loss above the noise floor is
\[
\frac{2}{d_{\mathrm{in}}}\sum_j\cos^2(\theta_{j+1}-\theta_j)
=1+\frac{1}{d_{\mathrm{in}}}\sum_j\operatorname{Re}(\overline z_jz_{j+1}).
\]
The Fourier vectors $e^{2\pi\mathrm i rj/d_{\mathrm{in}}}$ diagonalize the cycle
adjacency matrix, with eigenvalues $2\cos(2\pi r/d_{\mathrm{in}})$.
Its smallest eigenvalue is $-2\cos(\pi/d_{\mathrm{in}})$, attained only at
$r=\pm m$, where $m=(d_{\mathrm{in}}-1)/2$. Since $\sum_j|z_j|^2=d_{\mathrm{in}}$, the
Rayleigh bound gives a minimum excess loss $1-\cos(\pi/d_{\mathrm{in}})$,
attained by either unit Fourier mode. At equality,
$z_j=Ae^{2\pi\mathrm i mj/d_{\mathrm{in}}}+Be^{-2\pi\mathrm i mj/d_{\mathrm{in}}}$.
The constraint $|z_j|=1$ requires
\[
|A|^2+|B|^2+
2\operatorname{Re}(A\overline B e^{4\pi\mathrm i mj/d_{\mathrm{in}}})=1
\quad\text{for every }j.
\]
Here $2m=-1\pmod{d_{\mathrm{in}}}$. Orthogonality of the constant and the two
nonzero Fourier modes forces $A\overline B=0$, so every optimum
is one of the two unit modes up to common phase.
Consequently every $S$ optimum has squared overlap
$a=(1-\cos(\pi/d_{\mathrm{in}}))/2$ on $S$ pairs and
$b=\cos^2(\pi/d_{\mathrm{in}})$ on $E$ pairs. Its loss difference is
\begin{equation}
c_{d_{\mathrm{in}},2}:=D_E-D_S=2(b-a)
=\cos(2\pi/d_{\mathrm{in}})+\cos(\pi/d_{\mathrm{in}})>0.
\label{eq:dimension-cycle-margin}
\end{equation}
Take $P$ to be the $E$ cycle. Every witness pair has positive overlap
$b$ and occurs only on $E$, while
$I_E^P-I_S^P=2b>c_{d_{\mathrm{in}},2}-\delta$.
\Cref{lem:dimension-endpoint} therefore gives a trap for
$0<\delta<c_{d_{\mathrm{in}},2}$. The endpoint losses need not have symmetric
cross-branch differences for this conclusion.

For $(d_{\mathrm{in}},d_{\mathrm{out}})=(4,3)$, give $S$ probability $1/8$ on the matching
$M_S=\{\{1,2\},\{3,4\}\}$ and $3/16$ on each other pair.
Define $E$ by transposing features $2$ and $3$, so its lower-weight
matching is $M_E=\{\{1,3\},\{2,4\}\}$.
Every feature has frequency $1/2$ under either law.
To characterize every $S$ optimum, put
$x=|\langle v_1,v_2\rangle|$ and
$y=|\langle v_3,v_4\rangle|$.
The pair frame matrices
$A=v_1v_1^\top+v_2v_2^\top$ and
$B=v_3v_3^\top+v_4v_4^\top$ have eigenvalues
$(1+x,1-x,0)$ and $(1+y,1-y,0)$.
Their trace product is at least $(1-x)(1-y)$, with equality
attained by placing their eigenvectors in opposite eigenvalue order.
For completeness, in eigenbases the trace product is
$\sum_{r,s}\lambda_r\mu_s O_{rs}^2$, where $O$ is orthogonal.
The matrix $(O_{rs}^2)$ is doubly stochastic, so the rearrangement
inequality gives the stated lower bound, attained by the reverse
permutation.
The exact minimum of $D_S-2\sigma^2$ is therefore the minimum on
$[0,1]^2$ of
\[
\frac38\left[\frac23(x^2+y^2)+(1-x)(1-y)\right].
\]
The bracket has Hessian eigenvalues $7/3$ and $1/3$. Its unique
minimizer is $x=y=3/7$, giving excess loss $3/14$.
At this point each frame has distinct eigenvalues. Equality in the
trace bound forces the larger-eigenvalue axes to be orthogonal and
the smaller-eigenvalue axes to coincide. To see the resulting
overlaps, representatives are
\[
\begin{aligned}
v_1&=\sqrt{5/7}\,e_1+\sqrt{2/7}\,e_2,&
v_2&=\sqrt{5/7}\,e_1-\sqrt{2/7}\,e_2,\\
v_3&=\sqrt{5/7}\,e_3+\sqrt{2/7}\,e_2,&
v_4&=\sqrt{5/7}\,e_3-\sqrt{2/7}\,e_2.
\end{aligned}
\]
The equality conditions show that every optimum has the same squared
overlaps: $9/49$ on $M_S$ and $4/49$ elsewhere, regardless of column
signs or common rotation. Hence at every $S$ optimum
\[
D_E-D_S=\frac5{196},\qquad
I_E^{M_S}-I_S^{M_S}=\frac9{196}.
\]
The pairs in $M_S$ have higher probability on $E$, and the separate
minimum losses are equal by relabeling. Applying
\cref{lem:dimension-endpoint} with $P=M_S$ proves the result for
$0<\delta<c_{4,3}:=5/196$.

\subsubsection{Why the dimension exclusions are necessary}
\label{app:dimension-obstructions}

If $d_{\mathrm{in}}\le d_{\mathrm{out}}$, mutually orthogonal unit columns attain the noise floor
on both branches and are globally optimal at every $p$.
At this optimum all cross-feature errors vanish. This contradicts
the allocation and reversal requirements of \cref{def:trap}, which
must hold at every optimum. Thus no trap of that definition can
occur, regardless of the pair laws.

If $d_{\mathrm{out}}=1$, every unit column is $1$ or $-1$ and every $g_{ij}=1$.
Both branch losses equal $2+2\sigma^2$ for every code and every pair
law. Actual and separately adapted rankings coincide, so reversal
and suboptimality cannot hold for the same alternative policy.
These exclusions and \cref{eq:dimension-unequal-margin} give the
classification $d_{\mathrm{in}}>d_{\mathrm{out}}\ge2$ when feature frequencies are unrestricted.

Under equal feature frequencies, the remaining excluded case is
$(d_{\mathrm{in}},d_{\mathrm{out}})=(3,2)$. For three features, the individual occurrence
probabilities $\mu_1,\mu_2,\mu_3$ determine the entire pair law:
\[
\omega_{12}=\frac{\mu_1+\mu_2-\mu_3}{2},\qquad
\omega_{13}=\frac{\mu_1+\mu_3-\mu_2}{2},\qquad
\omega_{23}=\frac{\mu_2+\mu_3-\mu_1}{2}.
\]
Equal marginals across branches therefore force equal pair laws and
$D_S=D_E$ at every code. No return reversal is possible.
The constructions above cover all other $d_{\mathrm{in}}>d_{\mathrm{out}}\ge2$ with $d_{\mathrm{in}}\ge4$,
completing the equal-frequency classification.

\subsubsection{Persistence when every pair can occur}
\label{app:dimension-full-support}

The preceding constructions need not retain their zero-probability
pairs. Fix one construction and a reward gap strictly inside its
stated interval. At the compact set of $S$ optima, each witness
overlap, the return-reversal margin, and the attribution margin has
a positive minimum. The finitely many occurrence contrasts for the
witness pairs are strictly positive as well. At every $E$ optimum,
the opposite return margin is at least $\delta$.

For two pair laws $\omega_B,\widetilde\omega_B$,
\cref{eq:dimension-branch-loss} gives the uniform bound
\[
\sup_V|D_B(V;\widetilde\omega_B)-D_B(V;\omega_B)|
\le2\|\widetilde\omega_B-\omega_B\|_1.
\]
If perturbed laws converge to the original laws, every limit of
their endpoint minimizers is an original endpoint minimizer.
Otherwise compactness and uniform convergence would contradict
optimality of the limiting code. The strict overlap, reversal, and
attribution inequalities therefore persist at every endpoint
optimum under sufficiently small perturbations.
The separate minimum losses are continuous, so the positive adapted
return difference persists too. The reference gap remains exactly
$\delta$ because the self-response and two-coordinate noise baseline
do not depend on the pair law.

The same compactness argument jointly in the laws and $p$ preserves
strict return signs in endpoint neighborhoods. The odds estimates
and inverse-time construction in \cref{app:dimension-endpoints}
then give local existence and uniform attraction for the perturbed
system. This argument does not require the two perturbed separate
minimum losses to remain equal.

In particular, let $u$ be uniform on all $\binom{d_{\mathrm{in}}}{2}$ unordered pairs
and set
\[
\omega_B^{(\varepsilon)}
=(1-\varepsilon)\omega_B+\varepsilon u.
\]
For every sufficiently small $\varepsilon>0$, both branches have
positive probability on every pair and retain the trap with the same
witness set. Each witness occurrence contrast is multiplied by
$1-\varepsilon$. The uniform law gives every feature frequency $2/d_{\mathrm{in}}$,
so this mixing also preserves the equal-frequency constraint in all
of our equal-frequency constructions.
Each explicit pair of branch laws is related by a permutation of
feature labels. Uniform mixing preserves that relation and hence
equality of the two separately minimized losses. The adapted-return
gap thus remains exactly $\delta$ in these full-support examples.
The permitted mixing size depends on the chosen dimensions, laws,
and strict reward margin. The proof establishes local attraction
and persistence, without asserting a global basin formula.

\subsection{Matching-Complement Phase Diagram}
\label{app:matching-proof}

Recall that the minimum fitting loss in \cref{thm:matching-complement} is
\begin{equation}
R_{\rm mc}^\star(p)=2\sigma^2+\frac{\min\{p,1-p\}}{d-1}.
\label{eq:matching-value}
\end{equation}
Removing the $d$ pairs in $M_B$ leaves
$N=\binom{2d}{2}-d=2d(d-1)$ equally likely active pairs. Each feature
belongs to $2d-2$ of these pairs, so its probability of appearing is
$(2d-2)/N=1/d$ under either branch. This equality follows from the
sampling rule and does not depend on the representation dimension.
With unit encoder columns and the stated moment and noise assumptions,
an active pair $\{i,j\}$ has expected control error
$2g_{ij}+2\sigma^2$. Averaging gives
\[
D_B=2\sigma^2+\frac{2}{N}
\sum_{\{i,j\}\notin M_B}g_{ij}.
\]
The dimension enters the optimization below: each excluded pair can
share one of $d$ orthogonal directions. The frame bound and fitting
weights determine which matching can carry the overlap at an optimum.

\begin{proof}[Proof of \cref{thm:matching-complement}]
Write
\begin{equation*}
g_S=\sum_{e\in M_S}g_e,
\qquad g_E=\sum_{e\in M_E}g_e,
\qquad g_0=\sum_{e\notin M_S\cup M_E}g_e,
\qquad g_{\rm tot}=g_S+g_E+g_0.
\end{equation*}
The two branch distortions are
\begin{equation*}
D_S=2\sigma^2+\frac{2}{N}(g_E+g_0),
\qquad
D_E=2\sigma^2+\frac{2}{N}(g_S+g_0).
\end{equation*}
Thus, up to the common factor \(2/N\), the nonconstant mixture objective is
\begin{equation*}
L^{\rm mc}_p=pg_S+(1-p)g_E+g_0.
\end{equation*}

The Welch bound for \(2d\) unit vectors in \(\R^d\) gives
\(g_{\rm tot}\ge d\) \citep{welch1974lower}. If \(p\le1/2\), then
\begin{equation*}
L^{\rm mc}_p-pg_{\rm tot}=(1-2p)g_E+(1-p)g_0\ge0,
\end{equation*}
so \(L^{\rm mc}_p\ge pd\). Collapse the two endpoints of every edge in \(M_S\)
onto one member of an orthonormal basis. Since \(M_S\cap M_E=\varnothing\),
this has \((g_S,g_E,g_0)=(d,0,0)\) and attains \(pd\).

For \(0<p<1/2\), equality forces \(g_{\rm tot}=d\) and \(g_E=g_0=0\). At \(p=0\),
zero loss directly forces \(g_E=g_0=0\). In either case, choose one
representative from each \(M_S\) pair. Representatives from distinct pairs
are orthogonal and form a basis. The mate of representative \(k\) is
orthogonal to the other \(d-1\) basis vectors and therefore lies on the
\(k\)th representative's line. Hence \(g_S=d\), and the squared Gram matrix is
one exactly on \(M_S\) and zero elsewhere. Exchanging \(S\) and \(E\) proves
the upper-half statement.

At \(p=1/2\),
\begin{equation*}
L^{\rm mc}_{1/2}=\tfrac12g_{\rm tot}+\tfrac12g_0,
\end{equation*}
so equality holds exactly when \(g_{\rm tot}=d\) and \(g_0=0\). Both matching-collapse
geometries attain equality and have distinct squared Gram matrices. Since
\(N=2d(d-1)\), multiplying the optimum by \(2/N\) and restoring the noise
baseline proves \cref{eq:matching-value}.

For every fixed representation,
\begin{equation*}
D_E-D_S=\frac{2}{N}(g_S-g_E).
\end{equation*}
The rigid open-phase geometries give the two displayed slopes. At a pure
\(S\) optimizer, the deployed return advantage is
\(q_S-q_E=1/(d-1)-\delta\); at a pure \(E\) optimizer, it is
\(q_E-q_S=1/(d-1)+\delta\). Here \(q_B=b_B-D_B\), as in the finite-action
formulation. These margins are the same for every optimizer and are
positive for \(|\delta|<1/(d-1)\), and \cref{thm:simplex-robustness} gives
the local stability claim. Both resident distortions equal \(2\sigma^2\), so
the high-minus-low endpoint return is \(\delta\).
\end{proof}

\subsection{Pair-Support Dimension Certificates}

The support graph records which feature pairs must be recovered together.
Its edges require orthogonality for a tied code and two vanishing cross-feature responses
for an untied code. These constraints connect the required code dimension to
orthogonal representations and fitting minrank
\citep{lovasz1979shannon,peeters1996orthogonal,baryossef2011index}.

\begin{proposition}[Pair-Support Dimension Certificates]
\label{prop:graph-certificates}
For a graph \(G=([n],\mathcal E)\), consider positive code dimensions and let
\(\mathrm{od}_\R(G)\) be the least
\(k\) for which nonzero real vectors \(v_i\in\R^k\) satisfy
\begin{equation*}
\{i,j\}\in\mathcal E\quad\Longrightarrow\quad \inner{v_i}{v_j}=0.
\end{equation*}
A tied unit-vector code has zero interference excess on every positive-weight
supported edge exactly when \(\mathrm{od}_\R(G)\le k\).

For the noiseless untied linear memory
\(h=\sum_{j\ \mathrm{active}}v_jZ_j\),
\(\widehat Z_i=w_i^\top h\) for freely optimized controller weights \(w_i\in\R^k\), define
\begin{equation*}
\mathrm{mr}_\R(G):=\min\left\{\operatorname{rank}(\Gamma_{\rm fit}): (\Gamma_{\rm fit})_{ii}=1,\ (\Gamma_{\rm fit})_{ij}=(\Gamma_{\rm fit})_{ji}=0\ 
\text{ for every }\{i,j\}\in\mathcal E\right\}.
\end{equation*}
Exact noiseless linear recovery on every supported pair in dimension \(k\)
exists exactly when \(\mathrm{mr}_\R(G)\le k\). If \(M\) is a perfect matching on
\([2d]\), then
\begin{equation*}
\mathrm{od}_\R(K_{2d}\setminus M)
=\mathrm{mr}_\R(K_{2d}\setminus M)=d.
\end{equation*}
For disjoint perfect matchings \(M_S,M_E\), the union of their complement
supports is \(K_{2d}\), and both thresholds are \(2d\). By contrast, the
union of two perfect matchings is bipartite and has edge-orthogonality
dimension at most two.
\end{proposition}

\begin{proof}
For the tied claim, zero positive weighted excess is equivalent to
orthogonality on every supported edge; nonzero vectors can then be normalized.
For the untied claim, put \(\Gamma_{\rm fit}=W^\top V\), where \(W\) has columns \(w_i\). On an active pair \(\{i,j\}\),
zero squared error is equivalent to
\((\Gamma_{\rm fit})_{ii}=(\Gamma_{\rm fit})_{jj}=1\) and \((\Gamma_{\rm fit})_{ij}=(\Gamma_{\rm fit})_{ji}=0\). Every rank-at-most-\(k\) real
matrix factors as \(W^\top V\) with \(k\) rows. Isolated vertices can
be assigned unit signal gain in any positive dimension, since they occur in
no supported pair. This proves the equivalence, including graphs with
isolated vertices.

For \(K_{2d}\setminus M\), selecting one vertex from each matched pair gives
a \(K_d\), so \(\mathrm{od}_\R\ge d\). Collapsing each matched pair
onto one of \(d\) orthogonal axes gives equality. The same ordering makes a
feasible fitting matrix with \(d\) all-ones \(2\times2\) diagonal blocks, so
\(\mathrm{mr}_\R\le d\); selecting one vertex per pair gives an identity principal
submatrix of size \(d\), proving the reverse inequality. The union of the two
complement supports is complete. A complete graph needs \(2d\) mutually
orthogonal vectors, while its fitting constraints force \(\Gamma_{\rm fit}=I_{2d}\). Finally,
the union of two matchings is a disjoint union of paths, even cycles, and
shared edges. It is bipartite, so its two color classes may reuse two
orthogonal axes.
\end{proof}

The tied certificate remains a certificate for zero interference excess
when channel noise is positive; the noise contribution to recovery error
remains. The untied certificate concerns exact recovery in the noiseless
model. Its fitting matrix is real and has two-sided edge zeros, without a
symmetry or positive-semidefiniteness constraint. These conventions differ
from directed index-coding minrank over a finite field.

\subsection{A Smooth Codimension-One Comparison}
\label{app:codimension-one}

The matching-complement family keeps singleton marginals equal and has a
nonsmooth midpoint. A separate lift preserves the continuous three-feature
branch-loss difference in any dimension. It uses \(d+1\) features in dimension
\(d\), with the original three features occupying a plane and the remaining
features occupying its orthogonal complement.

\begin{proposition}[Exact Codimension-One Lift]
\label{cor:codimension-one-lift}
Fix \(d\ge2\), and let \(\mathcal P_d\) be all unordered pairs from
\([d+1]\). In the same pair-valued tied-controller construction, now with noise
covariance \(\sigma^2I_d\), let \(S\) be uniform on
\(\mathcal P_d\setminus\{\{1,2\}\}\), let \(E\) be uniform on
\(\mathcal P_d\setminus\{\{2,3\}\}\), and take unit
\(v_1,\ldots,v_{d+1}\in\R^d\). For
\begin{equation*}
L^{\rm lift}_p=p g_{12}+(1-p)g_{23}
+\!\!\sum_{\{i,j\}\in\mathcal P_d\setminus
\{\{1,2\},\{2,3\}\}}\!\!g_{ij},
\end{equation*}
we have \(\min L^{\rm lift}_p=L^\star(p)\). Every minimizer is an optimizer from
\cref{thm:frame} in a two-dimensional subspace plus an orthonormal basis of
its orthogonal complement. With
\begin{equation*}
c_d=\frac{2}{\binom{d+1}{2}-1}=\frac{4}{(d-1)(d+2)},
\end{equation*}
the population objective and branch difference are
\begin{equation*}
(1-p)D_S+pD_E=2\sigma^2+c_dL^{\rm lift}_p,
\qquad
D_E-D_S=c_d(g_{12}-g_{23})
=c_dL^{\star\prime}(p)=c_d\cD(p),
\end{equation*}
using one-sided endpoint derivatives. Setting \(\delta=c_d\bar\delta\) makes
the adiabatic actor field \(c_d\) times \cref{eq:replicator}; for
\(|\bar\delta|<1\), it has the same two basins and separator
\(p_{\bar\delta}\), and \cref{cor:policy-floor} carries over. The raw gap
scale is \(c_d=\Theta(d^{-2})\).
\end{proposition}

\begin{proof}[Proof of \cref{cor:codimension-one-lift}]
The case \(d=2\) is \cref{thm:frame}, so suppose \(d\ge3\) and put
\(m=d-2\). The upper bound follows by placing an optimal three-line
configuration in a two-dimensional subspace and choosing
\(v_4,\ldots,v_{d+1}\) as an orthonormal basis of its orthogonal complement.

For the lower bound, let \(\Gamma=(\inner{v_i}{v_j})_{i,j\le3}\),
\(\Xi=\sum_{i=1}^3v_iv_i^\top\),
\(\Xi_\perp=\sum_{k=4}^{d+1}v_kv_k^\top\), and
\(\lambda_{\min}=\lambda_{\min}(\Gamma)\in[0,1]\). Decompose
\begin{align*}
L^{\rm lift}_p&=L_p(\Gamma)+g_\perp,\\
L_p(\Gamma)&=p\Gamma_{12}^2+\Gamma_{13}^2+(1-p)\Gamma_{23}^2,\\
g_\perp&=\operatorname{tr}(\Xi\Xi_\perp)
 +\tfrac12\bigl(\operatorname{tr}(\Xi_\perp^2)-m\bigr).
\end{align*}
Here \(g_\perp\) is exactly the sum of all squared inner products involving at
least one of the final \(m\) vectors.

If \(m=1\), then \(\Xi_\perp\) is a rank-one projector and
\(g_\perp=\operatorname{tr}(\Xi\Xi_\perp)\ge\lambda_{\min}\), because \(\Xi\) and \(\Gamma\) have the
same eigenvalues. For \(m\ge2\), let
\(\lambda_1\ge\cdots\ge\lambda_m\ge0\) be the largest \(m\) eigenvalues of \(\Xi_\perp\),
including zeros. They sum to \(m\). Von Neumann's trace inequality gives
\(\operatorname{tr}(\Xi\Xi_\perp)\ge\lambda_{\min}\lambda_m\). Writing
\(t=\lambda_m\in[0,1]\), Cauchy--Schwarz yields
\begin{align*}
g_\perp&\ge \lambda_{\min} t+\frac12\left(t^2+\frac{(m-t)^2}{m-1}-m\right)\\
&=\lambda_{\min} t+\frac{m}{2(m-1)}(1-t)^2
\ge \lambda_{\min}-\frac{m-1}{2m}\lambda_{\min}^2.
\end{align*}
Thus in both cases
\(g_\perp\ge\lambda_{\min}-c_m\lambda_{\min}^2\), where
\(c_m=(m-1)/(2m)<1/2\).

If \(\lambda_{\min}<1\), the deflated matrix
\begin{equation*}
\widetilde\Gamma=\frac{\Gamma-\lambda_{\min} I_3}{1-\lambda_{\min}}
\end{equation*}
is a rank-at-most-two correlation matrix. It is therefore the Gram matrix of
three unit vectors in \(\R^2\), and
\begin{equation*}
L_p(\Gamma)=(1-\lambda_{\min})^2L_p(\widetilde\Gamma)
\ge(1-\lambda_{\min})^2L^\star(p).
\end{equation*}
The same bound is immediate when \(\lambda_{\min}=1\). Put
\(\ell=L^\star(p)\). From \cref{eq:optimal-loss},
\(0\le\ell\le7/16\), and hence
\begin{equation*}
L^{\rm lift}_p-\ell\ge
\lambda_{\min}\bigl[1-2\ell+\lambda_{\min}(\ell-c_m)\bigr].
\end{equation*}
The bracket is affine in \(\lambda_{\min}\); at its endpoints it is at least
\(1/8\) and strictly greater than \(1/16\), respectively. Thus every
configuration with \(\lambda_{\min}>0\) has \(L^{\rm lift}_p>\ell\).

At a global minimizer, therefore, \(\lambda_{\min}=0\). The core has rank at most
two, so \(L_p(\Gamma)\ge L^\star(p)\), while \(g_\perp\ge0\) directly from its
sum-of-squares form. Equality forces \(L_p(\Gamma)=L^\star(p)\) and
\(g_\perp=0\). The latter says precisely that the final \(d-2\) vectors are
mutually orthogonal and orthogonal to the core. Their complement has dimension
two, and \cref{thm:frame} gives the claimed optimizer rigidity. The converse
is the upper-bound construction.

Finally, with \(M=\binom{d+1}{2}\), pair distortion remains
\(2(g_{ij}+\sigma^2)\) for noise covariance \(\sigma^2I_d\). Therefore
\begin{align*}
D_S&=2\sigma^2+\frac{2}{M-1}\sum_{e\ne\{1,2\}}g_e,\\
D_E&=2\sigma^2+\frac{2}{M-1}\sum_{e\ne\{2,3\}}g_e.
\end{align*}
Their on-policy average is \(2\sigma^2+c_dL^{\rm lift}_p\), where
\(c_d=2/(M-1)=4/((d-1)(d+2))\), and their difference is
\(c_d(g_{12}-g_{23})=c_d\cD(p)\). Substituting
\(\delta=c_d\bar\delta\) into \cref{eq:replicator} multiplies the original
vector field by \(c_d\), proving the phase-portrait claim. The policy-floor
argument is identical after this scaling.
\end{proof}

\section{Closed-Loop Dynamics and Interventions}
\label{app:dynamics}

The return-gap formula in \cref{sec:deployed-gap} leaves a dynamical
question: whether the reversed preference attracts nearby policies as
the code changes.

\paragraph{Stability relative to feasible policies.}
Stability is measured within the feasible policy set, so boundary
policies admit only feasible perturbations.

\begin{definition}[Stability Relative to the Feasible Policy Set]
\label{def:policy-stability}
Let $K$ be the feasible policy set and $x^\star\in K$ an equilibrium
of the specified dynamics. Assume trajectories exist from every
sufficiently nearby initial point in $K$ and that all such maximal
trajectories are defined for every $t\ge0$.
\begin{enumerate}[label=(\roman*)]
\item \emph{Lyapunov stability:} For every $\eta>0$, there is $r>0$
such that every admissible trajectory with
$\|x(0)-x^\star\|<r$ satisfies
$\|x(t)-x^\star\|<\eta$ for all $t\ge0$.
\item \emph{Local asymptotic stability:} The equilibrium is Lyapunov
stable, and there is $r_0>0$ such that every admissible trajectory
with $\|x(0)-x^\star\|<r_0$ satisfies $x(t)\to x^\star$ as
$t\to\infty$.
\item \emph{Instability:} The equilibrium is not Lyapunov stable.
Failure to attract nearby trajectories alone does not imply instability.
\end{enumerate}
All initial points and trajectories are restricted to $K$.
\end{definition}

For exact adaptation, admissible trajectories are absolutely continuous
solutions of $\dot x=F(x,q(\theta))$ almost everywhere, paired with
measurable $\theta(t)\in\cM_0(x(t))$. When stability is claimed
uniformly over optimizer selections, the radii above are common to all
such solutions, and attraction is also uniform: for each $\eta>0$
there is $T_\eta<\infty$ such that
\[
\|x(0)-x^\star\|<r_0
\quad\Longrightarrow\quad
\|x(t)-x^\star\|<\eta\quad\text{for all }t\ge T_\eta,
\]
with $T_\eta$ independent of the initial point in this neighborhood
and of the admissible selection.

For a flow with unique trajectories, the basin of $x^\star$ is the set
of initial points whose trajectories converge to it. In the two-step
task, $K=[0,1]$, so stability at $0$ and $1$ is one-sided. Under
\cref{thm:trap}, their basins are $[0,p_\delta)$ and $(p_\delta,1]$.
The point $p_\delta$ is itself stationary, but arbitrarily small
perturbations to either side eventually leave a fixed neighborhood of
it. The full basin description therefore goes beyond local stability.

Recall from \cref{sec:deployed-gap} that under exact adaptation the
branch policy follows
\begin{equation}
\dot p=p(1-p)[\delta-\cD(p)].
\label{eq:replicator}
\end{equation}
The full order parameter is
\begin{equation}
\cD(p)=
\begin{cases}
1, &0\le p<p_-,\\
\displaystyle \frac{1-2p}{4}\bigl([p(1-p)]^{-2}-1\bigr),
&p_-\le p\le p_+,\\
-1, &p_+<p\le1.
\end{cases}
\label{eq:D-piecewise}
\end{equation}

\paragraph{Slope of the Order Parameter.}
For \(p\in(p_-,p_+)\), the order parameter in \cref{eq:D-piecewise} has the derivative
\begin{equation}
\cD'(p)=-\frac12\bigl([p(1-p)]^{-2}-1\bigr)
-\frac{(1-2p)^2}{2[p(1-p)]^3}<0.
\label{eq:D-derivative}
\end{equation}
so \(\cD\) decreases strictly across the central phase.

The optimal branch-loss difference determines the sign of every interior
actor update. Its two plateaus give the endpoint stability margins, and its
strict decrease in the central phase gives the unique unstable separator.

\begin{proof}[Proof of \cref{thm:trap}]
The function \(\cD\) is continuous, equals \(+1\) below \(p_-\), equals
\(-1\) above \(p_+\), and decreases strictly between them by
\cref{eq:D-derivative}. Hence every \(\delta\in(-1,1)\) has a unique
preimage \(p_\delta\in(p_-,p_+)\). The actor gap
\(\delta-\cD(p)\) is negative below \(p_\delta\) and positive above \(p_\delta\), which proves the basin statement.

Near zero,
\begin{equation*}
\dot p=(\delta-1)p+O(p^2),
\end{equation*}
so zero is asymptotically stable when \(\delta<1\). Writing \(x=1-p\), near
one we have
\begin{equation*}
\dot x=-(\delta+1)x+O(x^2),
\end{equation*}
so one is asymptotically stable when \(\delta>-1\). At the interior fixed
point,
\begin{equation*}
\left.\frac{d\dot p}{dp}\right|_{p=p_\delta}
=p_\delta(1-p_\delta)[-\cD'(p_\delta)]>0,
\end{equation*}
so it is unstable.

At \(p=0\), both pairs used by \(S\) are orthogonal and
\(J^\star(0)=b_S-2\sigma^2\). At \(p=1\), both pairs used by \(E\) are
orthogonal and \(J^\star(1)=b_E-2\sigma^2\). Their difference is \(\delta\).

For \(0<\delta<1\), these facts establish every clause of
\cref{def:trap}. Take \(x^\star=(1,0)\), \(y=(0,1)\), and
\(P=\{\{1,2\}\}\). The pair's co-activation probability rises from
zero under \(x^\star\) to \(1/2\) under \(y\), and
\cref{thm:frame} gives \(g_{12}=1\) at every optimum at \(x^\star\).
For each such optimum, with \(\Delta x=y-x^\star\),
\begin{equation*}
\Delta x^\top q=\delta-1<0,\qquad
\Delta x^\top q^\perp=\delta>0,\qquad
\Delta x^\top I^P=1.
\end{equation*}
Thus interference from this pair reverses the preference. The optimal
squared Gram entries determine the actor gap uniquely at every \(p\), so
all optimizer selections induce the same flow. Near zero this flow is
\(\dot p=(\delta-1)p(1-p)\); trajectories exist and the stability
already proved is uniform over those selections. Finally,
\(J^\star(y)-J^\star(x^\star)=\delta>0\) supplies the better
adapted alternative.
\end{proof}

\paragraph{Other binary policy updates.}
\label{app:binary-policy-updates}
The phase portrait in \cref{thm:trap} is determined by the sign of the
return gap. Under the same task and exact-adaptation assumptions, fix
$m\in C^1([0,1])$ with $m(0)=m(1)=0$ and $m(p)>0$ for $0<p<1$.
Replacing the natural policy gradient by
\[
\dot p=m(p)[\delta-\cD(p)]
\]
preserves all equilibrium, basin, and trap conclusions of the theorem.
Indeed, $\cD$ is Lipschitz on $[0,1]$, so the new flow has unique
solutions. Its zero velocity at both endpoints makes $[0,1]$ invariant
and gives global forward existence. For $-1<\delta<1$, the velocity
is negative on $(0,p_\delta)$ and positive on $(p_\delta,1)$.
Trajectories in each interval are monotone, and their limits must be
zeros of the velocity. They therefore converge to $0$ and $1$,
respectively. This proves asymptotic stability of both endpoints and
instability of $p_\delta$. Every optimal-code selection induces the
same flow, so stability remains uniform over those selections.
The feature allocation, interference attribution, and adapted-return
deficit are unchanged. Convergence rates may depend on $m$.

Natural policy gradient uses $m(p)=p(1-p)$. For ordinary policy gradient,
parameterize the branch probability by $p=(1+e^{-z})^{-1}$.
Differentiating the current expected return with respect to $z$, with
$V$ held fixed for this derivative, gives
\[
\dot z=p(1-p)[q_E(V)-q_S(V)],\qquad
\dot p=[p(1-p)]^2[\delta-\cD(p)].
\]
Thus exact expected gradient flow on a binary logit is included with
$m(p)=[p(1-p)]^2$. Stability here concerns action probabilities in the
closed interval $[0,1]$. From a finite initial logit, the pure policies
are approached as $z\to-\infty$ or $z\to+\infty$, rather than attained
as finite-parameter equilibria. This argument does not extend the
statement to sampled, finite-step, or entropy-regularized updates.

\paragraph{Adapted-return ascent.}
The optimized loss satisfies
$L^{\star\prime}(p)=g_{12}(p)-g_{23}(p)=\cD(p)$.
Differentiating \cref{eq:scalar-adapted-return} therefore gives
$J^{\star\prime}(p)=\delta-\cD(p)=\Delta_q(p)$, and along
\cref{eq:replicator},
\[
\frac{d}{dt}J^\star(p(t))=p(1-p)\Delta_q(p)^2\ge0.
\]
Since $L^\star$ is the minimum of functions affine in $p$, it is concave
and $J^\star$ is convex. For $|\delta|<1$, the endpoint slopes make both
endpoints local maxima. When $0<\delta<1$, ascent can thus converge to
$p=0$ even though $J^\star(1)-J^\star(0)=\delta>0$.
At $\delta=0$, interference and bistability remain, but the adapted
endpoints tie and the strict trap conditions fail.

\subsection{Positive Visitation}
\label{app:positive-visitation}

\begin{corollary}[Positive-visitation Trap]
\label{cor:policy-floor}
For \(0<\epsilon<p_-\), exact adaptation at \(p_\epsilon(s)\) and
\(\dot s=(1-2\epsilon)s(1-s)[\delta-\cD(p_\epsilon(s))]\) retain two stable endpoints
for \(|\delta|<1\), separated by
\((p_\delta-\epsilon)/(1-2\epsilon)\).
For \(0<\delta<1\), the low endpoint retains the collision and is worse
by \((1-2\epsilon)\delta\).
\end{corollary}

\begin{proof}[Proof of \cref{cor:policy-floor}]
Because \(0<\epsilon<p_-<p_\delta<p_+<1-\epsilon\), the affine map
\(p_\epsilon\) sends \([0,1]\) onto an interval containing \(p_\delta\).
The unique crossing of \(\cD\) with \(\delta\) gives a unique sign change
at the stated \(s_\delta\). The flow is negative below that crossing and
positive above, so each constrained endpoint attracts a neighborhood. In the low endpoint
phase, \(p=\epsilon<p_-\), so \cref{thm:frame} gives the exact \(1\)-\(2\)
collision. Finally, \(L^\star(p)=L^\star(1-p)\), and the noise term is the
same under both action distributions. Therefore the two adapted returns differ only in
base reward:
\begin{equation*}
J^\star(1-\epsilon)-J^\star(\epsilon)
=(1-2\epsilon)(b_E-b_S)=(1-2\epsilon)\delta.
\end{equation*}
\end{proof}

\subsection{Matching-Complement Actor Field}
\label{app:matching-dynamics}

Let \(c=1/(d-1)\). By \cref{thm:matching-complement}, every optimizer has
\(D_E-D_S=c\) for \(0\le p<1/2\) and \(D_E-D_S=-c\) for
\(1/2<p\le1\). Hence the exact fast-best-response actor obeys
\begin{equation*}
\dot p=p(1-p)(\delta-c)<0\quad\text{on }(0,1/2),
\qquad
\dot p=p(1-p)(\delta+c)>0\quad\text{on }(1/2,1)
\end{equation*}
whenever \(|\delta|<c\). Every initialization in the left open half therefore
converges to zero, and every initialization in the right open half converges
to one. At \(p=1/2\), the two matching-collapse geometries are both optimal
and have slopes \(c\) and \(-c\); their actor drifts have opposite signs.
The midpoint field therefore depends on the optimizer-selection rule.
The two available slopes do not by themselves imply that an
optimizer with an intermediate slope, or a stationary actor at the midpoint,
exists.

Under \(p_\epsilon(s)=\epsilon+(1-2\epsilon)s\),
\(0<\epsilon<1/2\), the constrained endpoints lie in opposite rigid phases
and inherit the same inward signs. Symmetry of \cref{eq:matching-value} gives
\begin{equation*}
J^\star_{\rm mc}(1-\epsilon)-J^\star_{\rm mc}(\epsilon)
=(1-2\epsilon)\delta.
\end{equation*}
Finally, the one-sided derivatives of \cref{eq:matching-value} at \(1/2\)
are \(c\) and \(-c\). The differentiability assumption of
\cref{prop:general-stability} fails, so neither that smooth-separator result nor
the three-feature finite-rate saddle result applies at the matching midpoint.

\subsection{Finite-Rate Representation Adaptation}
\label{app:finite-rate}

For the visitation floor write \(p_\epsilon(s)=\epsilon+(1-2\epsilon)s\).
With two line angles \(z\) relative to \(v_3\),
\(\Delta_D(z)=g_{12}-g_{23}\), and symmetric positive-definite \(M\),
consider the simultaneous adaptation dynamics
\begin{equation}
\dot z=-\kappa M\nabla_zL_{p_\epsilon(s)}(z),\qquad
\dot s=(1-2\epsilon)s(1-s)[\delta-\Delta_D(z)].
\label{eq:finite-rate-flow}
\end{equation}

\begin{proposition}[Finite-Rate Local Persistence]
\label{prop:finite-rate}
For \(0\le\epsilon<p_-\), \(|\delta|<1\), and every \(\kappa>0\),
the two endpoint policy-code pairs are locally asymptotically stable
after removing their common rotation. Each central global minimizer at
\(p_\delta\), paired with
\(s_\delta=(p_\delta-\epsilon)/(1-2\epsilon)\), is a hyperbolic saddle
with one unstable eigenvalue. For \(0<\delta<1\), the low sink has
return deficit \((1-2\epsilon)\delta\).
\end{proposition}

The three-feature model permits a local stability calculation when the
representation and actor move simultaneously. Two line angles describe the
representation after removal of the common rotation, so the joint system has
three coordinates. The statement in \cref{prop:finite-rate} is
relative to \((\R/\pi\mathbb Z)^2\times[0,1]\); the removed
rotation gives a family of equivalent codes, not an additional attracting
coordinate.

Fix the common-rotation gauge
\begin{equation*}
v_3=(1,0),\quad
v_1=(\cos\beta_1,\sin\beta_1),\quad
v_2=(\cos\beta_2,\sin\beta_2),
\end{equation*}
write \(z=(\beta_1,\beta_2)\), and take line angles modulo \(\pi\). Then
\begin{align}
L_p(z)&=p\cos^2(\beta_1-\beta_2)+\cos^2\beta_1
 +(1-p)\cos^2\beta_2,\nonumber\\
\Delta_D(z)&=\cos^2(\beta_1-\beta_2)-\cos^2\beta_2.
\label{eq:angle-loss-gap}
\end{align}

\begin{proof}[Proof of \cref{prop:finite-rate}]
Write \(c=1-2\epsilon>0\). At the low and high endpoints, direct
differentiation of \cref{eq:angle-loss-gap} gives
\begin{align*}
H_-&:=\nabla_z^2L_\epsilon(\pi/2,\pi/2)
=2\begin{pmatrix}1-\epsilon&\epsilon\\
\epsilon&1-2\epsilon\end{pmatrix},\\
H_+&:=\nabla_z^2L_{1-\epsilon}(\pi/2,0)
=2\begin{pmatrix}2-\epsilon&-(1-\epsilon)\\
-(1-\epsilon)&1-2\epsilon\end{pmatrix}.
\end{align*}
Both leading diagonal entries are positive, and
\begin{equation*}
\det H_-=\det H_+=4(1-3\epsilon+\epsilon^2)>0
\end{equation*}
exactly when \(\epsilon<p_-\). Thus \(H_-,H_+\succ0\).
Also \(\Delta_D(\pi/2,\pi/2)=1\) and \(\Delta_D(\pi/2,0)=-1\). The low Jacobian
eigenvalues are those of \(-\kappa M H_-\), together with
\(c(\delta-1)\). Using \(x=1-s\), the high actor eigenvalue is
\(-c(\delta+1)\), and its representation eigenvalues are those of
\(-\kappa M H_+\). Since \(M H_\pm\) is similar to the positive-definite
symmetric matrix \(M^{1/2}H_\pm M^{1/2}\), all eigenvalues have negative real
part for \(\delta\in(-1,1)\). This proves local stability relative to the
policy interval.

The central optimum has a positive-definite Hessian after removal of the
common rotation. In \cref{lem:weighted-frame}, the eliminated angle has
positive curvature and the profiled objective has \(\phi''(t)>0\).
The interior condition also gives \(dt/dv\ne0\), so the remaining Schur
complement is positive. For a smooth representative \(z^\star(p)\) of this
optimum, put
\begin{equation*}
H=\nabla_z^2L_{p_\delta}(z^\star(p_\delta))\succ0,
\quad g=\nabla_z\Delta_D(z^\star(p_\delta)),
\quad m=s_\delta(1-s_\delta)>0.
\end{equation*}
Because \(\partial_p\nabla_zL_p=\nabla_z\Delta_D=g\), differentiating the
first-order condition gives
\begin{equation*}
Hz^{\star\prime}(p)+g=0,
\qquad
\cD'(p)=-g^\top H^{-1}g.
\end{equation*}
Equation~\eqref{eq:D-derivative} is strictly negative at \(p_\delta\), so
\(g\ne0\).

Because \(\partial_s\nabla_zL_{p_\epsilon(s)}=c g\), the interior
Jacobian in coordinates \((z,s)\) is
\begin{equation*}
A=\begin{pmatrix}
-\kappa M H&-\kappa c M g\\
-c m g^\top&0
\end{pmatrix}.
\end{equation*}
With \(T=\operatorname{diag}(M^{1/2},\sqrt{m/\kappa})\), it is similar to
the symmetric matrix
\begin{equation*}
T^{-1}AT=
\begin{pmatrix}
-\kappa M^{1/2}H M^{1/2}
&-c\sqrt{\kappa m}\,M^{1/2}g\\
-c\sqrt{\kappa m}\,g^\top M^{1/2}&0
\end{pmatrix}.
\end{equation*}
The upper-left block is negative definite and its Schur complement is
\(c^2m g^\top H^{-1}g>0\). Sylvester inertia gives exactly two negative and
one positive eigenvalue. Hence the interior equilibrium is a hyperbolic
saddle. The endpoint return difference is
\((1-2\epsilon)\delta\), as in \cref{cor:policy-floor}.
\end{proof}

In the adiabatic flow \cref{eq:replicator}, every \(p\in[0,p_-]\) is fixed
at \(\delta=1\), and every \(p\in[p_+,1]\) is fixed at \(\delta=-1\).
These plateaus are neutral because the deployed gap vanishes throughout the
corresponding range of \(p\).

\subsection{Replay and Entropy}

\Cref{fig:mitigation-mechanisms} compares the three interventions in
\cref{sec:interventions} for the tied task under exact adaptation,
with $0<\delta<1$. \Cref{prop:gram-threshold} gives the overlap-penalty
thresholds.

\begin{figure}[t]
\centering
\includegraphics[width=\textwidth]{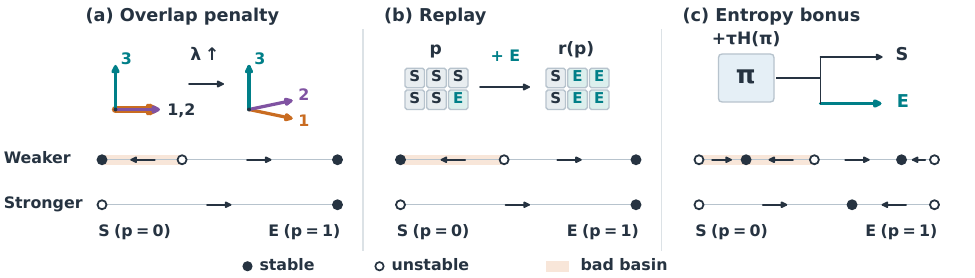}
\caption{Policy basins under three interventions in the tied task.
Schematic phase lines assume exact adaptation and $0<\delta<1$;
endpoint policies remain fixed. Eliminating collisions can precede
basin removal, while strong entropy leaves an interior attractor.}
\label{fig:mitigation-mechanisms}
\end{figure}

Replay changes the action weights used to fit the representation while the actor
still evaluates its current policy. The two basins survive exactly when
the range of fitting frequencies straddles the original separator.

\begin{proposition}[Replay Phase Condition]
\label{prop:replay}
Fix \(\delta\in(-1,1)\) and \(\bar p\in[0,1]\). Suppose the code is optimized
at fitting frequency
\begin{equation*}
r(p)=\rho p+(1-\rho)\bar p,
\qquad \rho\in[0,1],
\end{equation*}
while the actor follows
\begin{equation*}
\dot p=p(1-p)[\delta-\cD(r(p))].
\end{equation*}
For \(\rho>0\), a bistable separator exists if and only if
\begin{equation}
(1-\rho)\bar p<p_\delta<\rho+(1-\rho)\bar p.
\label{eq:replay-condition}
\end{equation}
If \(r(0)\ge p_\delta\), the low endpoint
loses relative asymptotic stability. If \(\rho=0\), the interior gap is
constant. Every interior trajectory then approaches the endpoint favored by
that gap, or remains fixed when the gap is zero.
\end{proposition}

\begin{proof}
For \(\rho>0\), \(r\) maps \([0,1]\) monotonically onto
\([(1-\rho)\bar p,\rho+(1-\rho)\bar p]\). The actor gap changes sign
inside \((0,1)\) exactly when this image contains \(p_\delta\) in its
interior. Monotonicity of \(\cD\) then gives the same basin signs as in
\cref{thm:trap}. If \(r(0)>p_\delta\), the gap is positive at zero and on its right.
If \(r(0)=p_\delta\), strict decrease in the central phase makes the gap
positive immediately to the right of zero. Both cases exclude relative
asymptotic stability of the low endpoint. At \(\rho=0\), the gap is
independent of \(p\), which gives the constant-gap alternatives.
\end{proof}

Entropy changes the actor objective directly and can move the stable
policies into the interior. For \(p\in(0,1)\), define
\begin{equation*}
f_{\delta,\tau}(p)
=\delta-\cD(p)+\tau\log\frac{1-p}{p},
\qquad
\dot p=p(1-p)f_{\delta,\tau}(p),
\end{equation*}
with the vector field continuously extended to the endpoints.

\begin{proposition}[Entropy-Regularized Trap]
\label{prop:entropy-trap}
Let \(\varphi=(1+\sqrt5)/2\). If \(0<\delta<1\),
\begin{equation*}
0<\tau<\frac{1-\delta}{\log\varphi},\qquad \tau\le\frac{15}{8},
\end{equation*}
then the entropy-regularized flow has exactly three interior equilibria. Its
sinks are
\begin{equation*}
p_L=\frac{1}{1+e^{(1-\delta)/\tau}}<p_-,\qquad
p_H=\frac{1}{1+e^{-(1+\delta)/\tau}}>p_+,
\end{equation*}
and a unique hyperbolic source \(p_U\in(p_-,1/2)\) separates their basins.
The low sink retains the exact \(1\)-\(2\) collision and has lower
environmental return, excluding entropy. If \(\tau\ge5/2\), then for every
finite \(\delta\) there is instead one interior equilibrium attracting every
initialization in \((0,1)\). The continuously extended flow fixes zero and one.
\end{proposition}

\begin{proof}[Proof of \cref{prop:entropy-trap}]
For fixed deployed \(V\), the entropy-regularized actor objective is
\begin{equation*}
(1-p)q_S(V)+pq_E(V)+\tau H(p),
\qquad H'(p)=\log\frac{1-p}{p}.
\end{equation*}
The same Bernoulli natural-gradient calculation as in
\cref{eq:replicator}, with \(V\) held fixed, gives the stated flow.

Recall \(\chi=p(1-p)\). For \(p\in(p_-,p_+)\),
\cref{eq:D-derivative} gives
\begin{equation}
\chi f_{\delta,\tau}'(p)=K(\chi)-\tau,\qquad
K(\chi)=\frac12\left(\chi^{-2}-3\chi^{-1}-\chi\right).
\label{eq:entropy-curvature}
\end{equation}
Here \(\chi\in(\sqrt5-2,1/4]\). Direct differentiation shows that \(K\)
decreases with \(\chi\), with the one-sided boundary value
\begin{equation*}
K(1/4)=\frac{15}{8},\qquad K(\sqrt5-2)=\frac52.
\end{equation*}
Thus \(f_{\delta,\tau}\) is strictly increasing in the central phase when
\(\tau\le15/8\), except for an isolated zero derivative at
\(p=1/2\) when equality holds.

Since \(p_-=\varphi^{-2}\) and \(p_+=\varphi^{-1}\),
\begin{align*}
f_{\delta,\tau}(p_-)&=\delta-1+\tau\log\varphi<0,\\
f_{\delta,\tau}(1/2)&=\delta>0.
\end{align*}
There is therefore one central zero \(p_U\in(p_-,1/2)\), and it is
hyperbolically unstable. In the low phase \(\cD=1\), so \(f\) decreases
strictly from \(+\infty\) and has the unique zero
\begin{equation*}
p_L=\frac{1}{1+e^{(1-\delta)/\tau}}.
\end{equation*}
The phase inequality gives \(p_L<p_-\). In the high phase \(\cD=-1\), the
same calculation gives the unique sink \(p_H>p_+\) stated in the proposition.
The signs are
\begin{equation*}
+\quad p_L\quad-\quad p_U\quad+\quad p_H\quad-,
\end{equation*}
which proves the basin statement.

Excluding entropy, the environmental return is
\(J^\star(p)=b_S+p\delta-L^\star(p)-2\sigma^2\). Put
\(x=p_L\) and \(y=1-p_H\). The explicit logits give
\(0<y<x<1/2\), so
\begin{equation*}
J^\star(p_H)-J^\star(p_L)
=(x-y)+\delta(1-x-y)>0.
\end{equation*}
The low sink lies in the exact \(1\)-\(2\) collision phase.

If \(\tau\ge5/2\), \cref{eq:entropy-curvature} is nonpositive throughout
the central phase, while the two outer derivatives are strictly negative.
The continuous function \(f\) is therefore strictly decreasing from
\(+\infty\) to \(-\infty\), giving one interior attractor. Finally,
\(p(1-p)\log((1-p)/p)\to0\) at both endpoints, so the continuously extended
field fixes them.
\end{proof}

\subsection{Orthogonality Regularization}
\label{app:orthogonality}

An orthogonality penalty changes the frame weights and reduces the
largest branch-loss difference available to sustain a bad basin. The exact
threshold depends on that difference, evaluated under the unregularized
environmental reward.

\begin{proposition}[Exact Orthogonality Threshold]
\label{prop:gram-threshold}
Suppose the representation minimizes
\(L_p(V)+\lambda(g_{12}+g_{13}+g_{23})\), \(\lambda\ge0\), while the actor
uses the unregularized branch gap. Define
\begin{equation*}
s_\lambda=
\begin{cases}
1,&0\le\lambda\le1,\\
(3\lambda+1)/(4\lambda^2),&\lambda>1.
\end{cases}
\end{equation*}
Every endpoint optimizer has deployed branch-distortion difference
\(D_E-D_S=+s_\lambda\) at zero and \(D_E-D_S=-s_\lambda\) at one. If
\(0<\delta<s_\lambda\), the flow has two stable
endpoints and one unstable separator, and the low endpoint is worse by
\(\delta\). For fixed \(0<\delta<1\), let
\begin{equation*}
\lambda_{\rm crit}(\delta)
=\frac{3+\sqrt{9+16\delta}}{8\delta}.
\end{equation*}
The low endpoint has an open basin exactly when
\(\lambda<\lambda_{\rm crit}\). At equality it is nonhyperbolic and repelling
from the interior, while every positive initialization converges to one. Above
the threshold the actor gap has the uniform margin \(\delta-s_\lambda>0\).
\end{proposition}

\begin{figure}[H]
  \centering
  \includegraphics[width=0.85\textwidth]{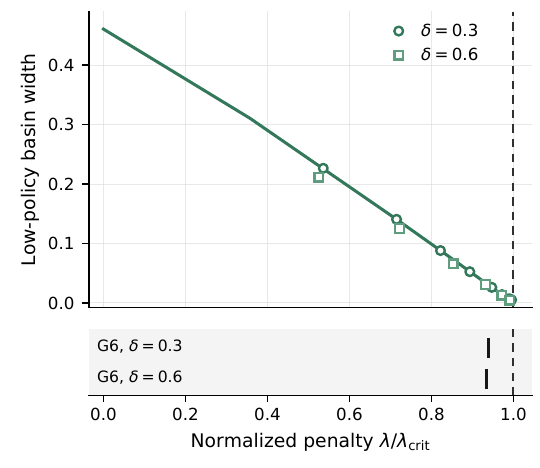}
  \caption{Estimated low-policy basin width under orthogonality regularization
at \(\kappa=80\), with penalty normalized by the analytical
\(\lambda_{\rm crit}(\delta)\). The curve is the post-hoc
\(\delta=0.3\) diagnostic; open markers are follow-up estimates for
\(\delta=0.3,0.6\). The separate lower panel marks the registered G6
escape thresholds from one fixed initialization. G6 failed its registered
criterion; these penalty estimates do not measure basin width. The dashed
line marks the theoretical basin-removal threshold.}
  \label{fig:e2-ortho-appendix}
\end{figure}

\begin{proof}[Proof of \cref{prop:gram-threshold}]
The regularized frame weights are
\begin{equation*}
A=p+\lambda,\qquad B=1+\lambda,\qquad C=1-p+\lambda.
\end{equation*}
When \(\lambda=0\) and \(p=0\), the objective is
\(g_{13}+g_{23}\). Its minimum zero forces both \(v_1\) and \(v_2\) to be
orthogonal to \(v_3\), hence
\((g_{12},g_{13},g_{23})=(1,0,0)\). For \(\lambda>0\), all weights are
positive. At \(p=0\), the \(1\)-\(2\) collision condition from
\cref{lem:weighted-frame} is
\begin{equation*}
\lambda\le\frac{1+\lambda}{2},
\end{equation*}
which is equivalent to \(\lambda\le1\). Together with the direct boundary
argument, every endpoint optimizer therefore has
\((g_{12},g_{13},g_{23})=(1,0,0)\) for \(0\le\lambda\le1\). For
\(\lambda>1\), the
central weighted-frame formulas give
\begin{equation*}
g_{12}=\frac{(1+\lambda)^2}{4\lambda^2},\qquad
g_{13}=g_{23}=\frac{\lambda-1}{4\lambda},
\end{equation*}
and hence \(g_{12}-g_{23}=s_\lambda\). Feature exchange gives
\(-s_\lambda\) at \(p=1\). The squared Gram geometry is unique in every
phase, so these statements hold for every optimizer.

The coefficient-\(B\) collision is impossible because
\begin{equation*}
\frac{AC}{A+C}\le\frac{A+C}{4}<B.
\end{equation*}
For \(0\le\lambda\le1\), applying the lemma on \(p\in(0,1)\), then using the
direct endpoint argument and feature exchange, gives symmetric outer intervals,
possibly reduced to endpoints at \(\lambda=1\), on which the deployed branch-distortion difference
\(\cD_\lambda=g_{12}-g_{23}\) equals \(+1\) and \(-1\). For
\(\lambda>1\), the \(A\)-collision condition already fails at zero and its
left side increases while its right side decreases with \(p\); the
\(C\)-condition is symmetric. Hence the whole interval is central. Put
\begin{equation*}
S=1+2\lambda,\qquad Q=(p+\lambda)(1-p+\lambda).
\end{equation*}
Envelope differentiation of \cref{eq:weighted-value} gives, in every central
phase,
\begin{align*}
\cD_\lambda(p)
&=\frac{1-2p}{4}\left(\frac{BS^2}{Q^2}-\frac1B\right),\\
\cD_\lambda'(p)
&=-\frac12\left(\frac{BS^2}{Q^2}-\frac1B\right)
-\frac{BS^2(1-2p)^2}{2Q^3}<0.
\end{align*}
The first bracket is positive because \(Q\le S^2/4<BS\). Therefore
\(\cD_\lambda\) decreases from \(+s_\lambda\) to \(-s_\lambda\), strictly
outside any collision plateau. If \(0<\delta<s_\lambda\), it crosses
\(\delta\) exactly once, giving the two sinks and one source.

The environmental return excludes the orthogonality penalty, so its endpoint
comparison uses the branch distortions themselves. Feature-exchange
symmetry gives
\begin{equation*}
D_E(V_\lambda^\star(1))=D_S(V_\lambda^\star(0)),
\end{equation*}
so the high-minus-low environmental return is exactly \(\delta\).

For \(\lambda>1\), \(s_\lambda=(3\lambda+1)/(4\lambda^2)\) decreases
strictly. Solving \(s_\lambda=\delta\) gives the stated
\(\lambda_{\rm crit}>1\). Below \(\lambda_{\rm crit}\), \(\delta<s_\lambda\) gives an open low basin. Above \(\lambda_{\rm crit}\),
\begin{equation*}
\delta-\cD_\lambda(p)\ge\delta-s_\lambda>0,
\end{equation*}
so every interior trajectory converges to one. At equality,
\(\cD_\lambda(0)=\delta\) but strict central decrease gives
\(\cD_\lambda(p)<\delta\) for every \(p>0\). Thus the exact endpoint remains
fixed, is nonhyperbolic, and repels every interior initialization.
\end{proof}

The replicator prefactor \(p(1-p)\) keeps the exact endpoints fixed after
these interventions. When orthogonal capacity, routing, or sufficient replay
removes the open bad basin, the low endpoint loses relative asymptotic
stability. An actor initialized exactly at zero still needs mutation or a
positive reinitialization to leave that endpoint. A visitation floor instead
excludes zero probability of \(E\) from the policy class.

\section{Untied Linear Controller}
\label{app:untied}

The untied controller has independently optimized output weights $w_i$, with
\(\widehat Z_i=w_i^\top h\) for each feature. The encoder retains the unit-column
constraint, and the pair laws remain those of \cref{sec:exact}.
Write \(D_B(V,W)\) for the branch distortion under these encoder and controller
matrices. The population loss is
\((1-p)D_S(V,W)+pD_E(V,W)\).

Each output weight vector can be optimized separately. Write
\(\omega_{ij}=\Pr(G=\{i,j\})\) for the pair probabilities of
\cref{sec:task-construction}. Let \(\eta=\sigma^2>0\),
let \(r_i=\sum_{j\ne i}\omega_{ij}=\Pr(i\in G)\), and let
\(\omega_{j\mid i}=\omega_{ij}/r_i\) be the conditional probability that feature
\(j\) accompanies active feature \(i\). The branch laws give
\(r_i\in[1/2,1]\), so these conditional weights are defined at every
\(p\). For fixed \(V\), define
\begin{equation*}
M_i=\eta I+v_iv_i^\top+\sum_{j\ne i}\omega_{j\mid i}v_jv_j^\top.
\end{equation*}
The matrix \(M_i\) is the second-moment matrix of the noisy representation \(h\) conditional
on feature \(i\) being active. Its positive noise term gives the unique
controller optimum \(w_i^\star=M_i^{-1}v_i\) and the reduced loss
\begin{equation}
F_{\eta,p}(V)=\sum_i r_i\bigl(1-v_i^\top M_i^{-1}v_i\bigr).
\label{eq:untied-objective}
\end{equation}

To obtain this optimum, expand the contribution of feature \(i\) to the
population loss:
\begin{align*}
\mathcal L_i
&=r_i(w_i^\top v_i-1)^2
+\sum_{j\ne i}\omega_{ij}(w_i^\top v_j)^2
+r_i\eta\|w_i\|^2\\
&=r_i\left[1-2w_i^\top v_i+w_i^\top M_iw_i\right].
\end{align*}
The quadratic is strictly convex in \(w_i\), and its first-order condition
is \(M_iw_i=v_i\). Substitution gives \cref{eq:untied-objective}.
For \(\{i,j,k\}=\{1,2,3\}\), direct inversion in two dimensions gives a
formula using only the squared Gram entries:
\begin{equation*}
F_{\eta,p}(V)=\sum_i r_i
\frac{
\eta(\eta+1)+\omega_{j\mid i}\omega_{k\mid i}(1-g_{jk})
}{
\eta(\eta+2)+\omega_{j\mid i}\omega_{k\mid i}(1-g_{jk})
+\omega_{j\mid i}(1-g_{ij})+\omega_{k\mid i}(1-g_{ik})
}.
\end{equation*}

\subsection{Endpoint Optima at Any Positive Noise Level}
\label{app:untied-endpoints}

At a pure policy, both supported pairs can be represented orthogonally.
This geometry attains the separate lower bound for every controller output and
therefore determines all global optima, even though the controller directions
are unconstrained.

\begin{lemma}[Untied Endpoint Optima]
\label{prop:untied-endpoint}
For every \(\eta>0\), all global minimizers at \(p=0\) satisfy
\(v_1=\pm v_2\perp v_3\), \(w_i^\star=v_i/(1+\eta)\), and
\(D_E-D_S=(1+\eta)^{-2}\). At \(p=1\), the slope has the opposite sign.
Thus the low code reverses the orthogonal-capacity gap whenever
\(0<\delta<(1+\sigma^2)^{-2}\).
\end{lemma}

\begin{proof}[Proof of \cref{prop:untied-endpoint}]
Since
\begin{equation*}
M_i\succeq\eta I+v_iv_i^\top,
\end{equation*}
inverse monotonicity gives
\begin{equation*}
v_i^\top M_i^{-1}v_i
\le v_i^\top(\eta I+v_iv_i^\top)^{-1}v_i
=\frac1{1+\eta}.
\end{equation*}
Each normalized per-feature loss is therefore at least
\(\eta/(1+\eta)\). Equality holds exactly when the additional
positive-semidefinite term
\(\sum_{j\ne i}\omega_{j\mid i}v_jv_j^\top\) annihilates
\((\eta I+v_iv_i^\top)^{-1}v_i=v_i/(1+\eta)\).
Thus every positively weighted interferer must be orthogonal to \(v_i\).

At \(p=0\), features 1 and 2 each co-activate with feature 3. Equality for all rows therefore forces \(v_1,v_2\perp v_3\), and in \(\R^2\) this forces \(v_1=\pm v_2\). Since all \(r_i\) are positive, every global optimizer attains each bound and therefore has this geometry. The conditional interferer weights sum to one,
so \(M_i=(1+\eta)I\) for each feature and
\(w_i=v_i/(1+\eta)\).

The two pairs in branch \(S\) are orthogonal and have total distortion
\begin{equation*}
D_S=\frac{2\eta}{1+\eta}.
\end{equation*}
The colliding pair \(\{1,2\}\) has distortion
\begin{equation*}
\frac{2(\eta^2+\eta+1)}{(1+\eta)^2},
\end{equation*}
which exceeds that of an orthogonal pair by \(2/(1+\eta)^2\). Branch \(E\) averages
one colliding and one orthogonal pair, proving
\(D_E-D_S=(1+\eta)^{-2}\). The exchange of features 1 and 3 swaps the two branches and gives the
opposite difference at \(p=1\).
\end{proof}

The endpoint calculation extends to a neighborhood of each pure policy
because the optimum set is compact after optimization of the controller.
The same endpoint margins give local stability at every fixed positive noise level without a
phase diagram for all \(p\).
The stability neighborhoods need not be uniform in \(\eta\).

\begin{proof}[Proof of \cref{cor:untied-local}]
For fixed \(\eta>0\), the controller optimum obeys
\(\|w_i^\star\|\le\eta^{-1}\), so restricting each output weight vector to that
compact ball leaves every global optimum unchanged. By
\cref{prop:untied-endpoint}, every optimizer has endpoint slopes
\(\Delta_{D,0}=s_\eta\) and \(\Delta_{D,1}=-s_\eta\). Applying
\cref{prop:selection-stability} gives the stability claim. Feature-exchange
symmetry gives
\begin{equation*}
F_\eta^\star(0)=F_\eta^\star(1)=\frac{2\eta}{1+\eta},
\end{equation*}
where \(F_\eta^\star(p)=\min_VF_{\eta,p}(V)\). The high-minus-low
endpoint return is therefore \(\delta\).
\end{proof}

\subsection{A Global Phase Diagram at Sufficiently Large Noise}
\label{app:untied-global}

At large noise, the first encoder-dependent term of the optimized untied
loss is the tied weighted-frame objective. A uniform expansion transfers
its unique central separator to the untied model. Concavity of the full
population optimum then controls the actor signs outside this central
interval.

\begin{proposition}[Untied-controller Global Phase Portrait]
\label{thm:untied-global}
For some finite \(\eta_0\), every \(\eta\ge\eta_0\) and
\(0<\delta<(8\eta^2)^{-1}\) yield two stable endpoints and one unstable
separator under the globally optimized untied model. The low endpoint is
worse by \(\delta\), uniformly over global-minimizer selection.
\end{proposition}

\begin{proof}[Proof of \cref{thm:untied-global}]
Write
\begin{equation*}
A_i=v_iv_i^\top+\sum_{j\ne i}\omega_{j\mid i}v_jv_j^\top.
\end{equation*}
Fix the common rotation by taking \(\beta_1=0\) and use the remaining line angles \((\beta_2,\beta_3)\). The denominators of the conditional
weights satisfy \(r_i\ge1/2\), so \(A_i\) and its derivatives in \((p,\beta_2,\beta_3)\) are uniformly bounded. For sufficiently large
\(\eta\), the inverse has the uniform Neumann expansion
\begin{equation*}
(\eta I+A_i)^{-1}
=\eta^{-1}I-\eta^{-2}A_i+\eta^{-3}A_i^2+O(\eta^{-4}),
\end{equation*}
including two derivatives in \((p,\beta_2,\beta_3)\). Since
\begin{equation*}
v_i^\top A_iv_i=1+\sum_{j\ne i}\omega_{j\mid i}g_{ij},
\end{equation*}
and \(\sum_i r_i=2\), summing the expansion gives
\begin{align*}
F_{\eta,p}(V)
&=2-\frac2\eta
+\frac1{\eta^2}\left[2+
\sum_i\sum_{j\ne i}\omega_{ij}g_{ij}\right]
+O(\eta^{-3})\\
&=2-\frac2\eta+\frac{2+L_p(V)}{\eta^2}+O(\eta^{-3}),
\end{align*}
with uniform \(C^2\) remainder.

Fix \(I=[0.4,0.6]\), a compact interval strictly inside the central phase.
Lemma~\ref{lem:weighted-frame} gives a unique minimizer for each \(p\in I\)
modulo common orthogonal transformations and column signs. The curvature in
the eliminated angle is positive. At the profiled optimum,
\(\phi''(t)>0\) and \(dt/dv\ne0\), so the remaining Schur complement is
also positive. The two-angle Hessian is therefore positive definite, with
its smallest eigenvalue bounded away from zero on \(I\).
At \(p=1/2\), the optimal squared Gram entries are
\begin{equation*}
(g_{12},g_{13},g_{23})=(3/8,1/16,3/8),
\end{equation*}
and the gauge-fixed Hessian is
\begin{equation*}
\begin{pmatrix}1/2&-1/4\\-1/4&2\end{pmatrix}\succ0.
\end{equation*}

After rescaling by \(\eta^2\) and subtracting the encoder-independent
terms, the untied objective converges to \(L_p\) in \(C^2\).
The uniform Hessian bound and the implicit function theorem give one nearby
smooth minimizer for each representative of the tied optimum. These
representatives differ only by symmetries of both objectives. Compactness
of the angle space and uniqueness of the limiting optimum give a uniform
positive value gap outside neighborhoods of that optimum. For sufficiently
large \(\eta\), every global minimizer must lie in those neighborhoods.
Thus some finite \(\eta_0\) makes every global minimizer on \(I\) share
one smooth squared Gram geometry satisfying
\begin{equation*}
g^\eta(p)=g^{\mathrm{tied}}(p)+O(\eta^{-1}).
\end{equation*}
Envelope differentiation of \(F_\eta^\star\) and the uniform expansion give
\begin{equation*}
s_\eta(p):=F_\eta^{\star\prime}(p)
=\eta^{-2}\cD(p)+O(\eta^{-3}),
\qquad
s_\eta'(p)=\eta^{-2}\cD'(p)+O(\eta^{-3}).
\end{equation*}
The error bounds are uniform on \(I\). Since \(\cD'\) is bounded strictly
below zero on \(I\), enlarge \(\eta_0\) so that \(s_\eta'<0\). Feature
exchange symmetry gives \(F_\eta^\star(p)=F_\eta^\star(1-p)\), hence
\(s_\eta(1/2)=0\). Moreover,
\begin{equation*}
\cD(0.4)=\frac{589}{720},
\end{equation*}
so \(\eta_0\) can also be chosen so that
\(s_\eta(0.4)>1/(4\eta^2)\).

For \(0<\delta<1/(8\eta^2)\), continuity and strict decrease give a unique
\(p_{\eta,\delta}\in(0.4,0.5)\) with \(s_\eta(p_{\eta,\delta})=\delta\). The rest of the interval for \(p\) is controlled by the affine dependence
of the full loss on \(p\) for fixed \((V,W)\). As above,
\(M_i\succeq\eta I\) gives \(\|w_i^\star\|\le\eta^{-1}\), so a compact
controller class contains every optimum. Lemma~\ref{lem:concavity} therefore
makes \(F_\eta^\star\) concave, and every optimized deployed branch-loss
difference is one of its supergradients. To the left of \(0.4\), each such
slope is at least \(s_\eta(0.4)>\delta\). To the right of \(0.6\), each
is at most \(s_\eta(0.6)<0<\delta\). On \(I\), strict decrease gives
the sole crossing already identified. The actor gap consequently has one
sign on each side of \(p_{\eta,\delta}\), for every optimizer selection. Finally,
\begin{equation*}
\frac{1}{8\eta^2}<\frac{1}{(1+\eta)^2}
\end{equation*}
for all sufficiently large \(\eta\). The exact endpoint slopes from
\cref{prop:untied-endpoint} therefore give inward drift near both endpoints.
The interior crossing is unstable because \(s_\eta'<0\). Endpoint distortions
are equal by symmetry, so the high endpoint returns \(\delta\) more than the low one.
Every interior initialization below the separator therefore converges to
zero, and every initialization above it converges to one, independently of
the selected global optimizer.
\end{proof}

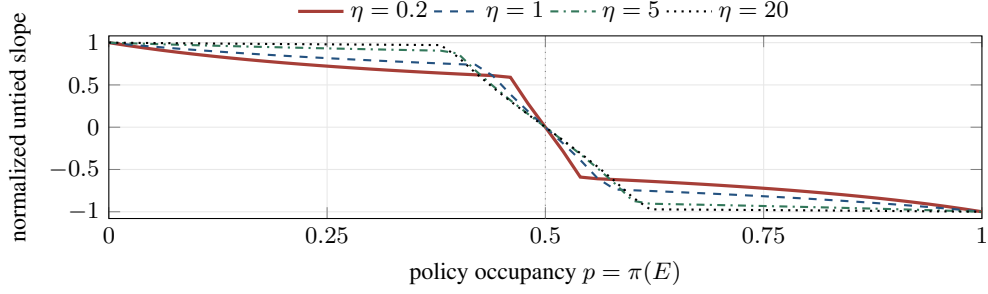
\begin{figure}[t]
    \centering
    \input{figures/untied_phase_diagram}
    \caption{Numerical untied-controller branch-loss differences for
    \(\eta\in\{0.2,1,5,20\}\), divided by the exact endpoint magnitude
    \((1+\eta)^{-2}\). A deterministic angle search gives a monotone
    crossing of zero at all four noise levels. These finite-noise calculations
    do not give a value for the asymptotic threshold \(\eta_0\) in
    \cref{thm:untied-global}.}
    \label{fig:untied-numerical}
\end{figure}

At \(\eta=0\), the controller optimum can be expressed using
\(M_i^\dagger\), but the controller may cease to be unique and the compactness
argument above is unavailable. The fixed-positive-noise endpoint result and
the large-noise phase result therefore leave the noiseless phase diagram
unresolved.

\section{General Fast-adaptation Results}
\label{app:general}

Throughout the exact-response arguments, the parameter class is compact,
the per-action losses are continuous, and the fitting loss is affine in
the action distribution (\cref{sec:setting}). The finite-rate result
introduces additional smoothness assumptions.

\subsection{The Two-Action Envelope Identity}
\label{app:two-action-envelope}

For two actions, the slope of the optimized distortion determines the deployed
value gap. For actions \(S,E\), write \(x=(1-p,p)\), set
\(\delta:=b_E-b_S\), and define
\begin{align*}
\Delta_D(\theta)&:=D_E(\theta)-D_S(\theta),\\
R^\star(p)&:=\min_{\theta\in\Theta}
\bigl[D_S(\theta)+p\,\Delta_D(\theta)\bigr],\\
\cM(p)&:=\argmin_{\theta\in\Theta}
\bigl[D_S(\theta)+p\,\Delta_D(\theta)\bigr].
\end{align*}
Thus \(J^\star(p)=b_S+p\delta-R^\star(p)\), while
\(\cM(p)=\cM_0((1-p,p))\) is the full set of optimal parameter choices.
The next lemma identifies the envelope slopes even when that set contains
more than one optimizer.

\begin{lemma}[Two-Action Fast-Adaptation Geometry]
\label{lem:concavity}
The function \(R^\star\) is concave and continuous. Its right derivative for
\(0\le p<1\) and left derivative for \(0<p\le1\) satisfy
\begin{equation*}
R_+^{\star\prime}(p)=\min_{\theta\in\cM(p)}\Delta_D(\theta),
\qquad
R_-^{\star\prime}(p)=\max_{\theta\in\cM(p)}\Delta_D(\theta).
\end{equation*}
At every interior differentiability point, all
\(\theta\in\cM(p)\) have the same slope \(\Delta_D(\theta)\), and the exact
branch gap under the current parameters satisfies
\begin{equation}
J^{\star\prime}(p)=\delta-R^{\star\prime}(p)
=\delta-\bigl[D_E(\theta)-D_S(\theta)\bigr].
\label{eq:general-gap}
\end{equation}
\end{lemma}

\begin{proof}[Proof of \cref{lem:concavity}]
Each parameter choice contributes an affine function of \(p\),
\begin{equation*}
R_\theta(p)=D_S(\theta)+p\,\Delta_D(\theta).
\end{equation*}
The minimum of these functions is concave. Indeed, for
\(p_0,p_1\in[0,1]\) and \(\lambda\in[0,1]\),
\begin{align*}
R^\star(\lambda p_0+(1-\lambda)p_1)
&=\min_\theta
\bigl[\lambda R_\theta(p_0)+(1-\lambda)R_\theta(p_1)\bigr]\\
&\ge \lambda\min_\theta R_\theta(p_0)
+(1-\lambda)\min_\theta R_\theta(p_1).
\end{align*}
Compactness and continuity ensure attainment of the minimum. They also
bound \(\Delta_D\) uniformly, so the envelope is Lipschitz and hence continuous,
including at the endpoints.

Fix \(p<1\) and take \(h>0\) with \(p+h\le1\). An optimizer at \(p\)
provides an upper bound at the neighboring probability \(p+h\),
\begin{equation*}
R^\star(p+h)-R^\star(p)
\le h\,\Delta_D(\theta).
\end{equation*}
This bound holds for every \(\theta\in\cM(p)\), giving
\begin{equation*}
\limsup_{h\downarrow0}
\frac{R^\star(p+h)-R^\star(p)}{h}
\le \min_{\theta\in\cM(p)}\Delta_D(\theta).
\end{equation*}
For the reverse bound, choose a sequence \(h\downarrow0\) realizing the
lower limit of the difference quotients and choose
\(\theta_h\in\cM(p+h)\). Compactness gives a convergent subsequence with
limit \(\bar\theta\); continuity and optimality imply
\(\bar\theta\in\cM(p)\). The neighboring optimizer satisfies
\begin{equation*}
R^\star(p+h)-R^\star(p)
\ge h\,\Delta_D(\theta_h).
\end{equation*}
The subsequential limit is at least the smallest slope among optimizers at
\(p\), which proves the right-derivative formula. The same argument for a
negative increment reverses the slope ordering and gives the left derivative.
At an interior differentiability point, the two extrema therefore coincide,
so every optimizer supplies the same deployed gap in
\cref{eq:general-gap}.
\end{proof}

At interior differentiability points, the two-action exact-response
replicator consequently takes the form
\begin{equation}
\dot p=p(1-p)J^{\star\prime}(p).
\label{eq:general-flow}
\end{equation}
\Cref{thm:fast-envelope-ascent} gives the corresponding ascent identity
for any number of actions, including at nonsmooth optima.

\subsection{Exact Envelope Ascent and Its Scope}
\label{app:envelope-ascent}

We establish that adapted return is nondecreasing under exact fitting
and the unperturbed replicator actor, and explain why this does not
exclude convergence to a lower-return attractor. We also characterize
the location of attracting equilibria. The proof handles changes of
optimizer without assuming a differentiable optimal-value function.

Write $C(x)=\operatorname{diag}(x)-xx^\top$, so the replicator actor is
$\dot x=C(x)q(\theta)$. We also establish the pointwise directional bound
\begin{equation}
J^{\star\prime}(x;C(x)q(\theta))
\ge q(\theta)^\top C(x)q(\theta)
=\operatorname{Var}_x[q(\theta)],\qquad \theta\in\cM_0(x).
\label{eq:envelope-directional}
\end{equation}
The equality along trajectories is stronger than this bound at an
arbitrary point: it uses both time directions at almost every point of
an absolutely continuous solution.

\begin{theorem}[Exact Fast-response Envelope Ascent]
\label{thm:fast-envelope-ascent}
In this finite-action model, $J^\star$ is convex and Lipschitz.
Every absolutely continuous replicator trajectory with a measurable
$\theta(t)\in\cM_0(x(t))$ satisfies
\begin{equation}
\frac{d}{dt}J^\star(x(t))
=\operatorname{Var}_{x(t)}[q(\theta(t))]
=\sum_{a\in\operatorname{supp}x(t)}
\frac{\dot x_a(t)^2}{x_a(t)}\ge0
\quad\text{a.e.}
\label{eq:envelope-energy}
\end{equation}
Every simplex face is invariant. If an admissible exact-optimizer
selection generates a semiflow, every locally asymptotically stable
equilibrium relative to the full simplex is a vertex.
\end{theorem}

\begin{proof}[Proof of \cref{thm:fast-envelope-ascent}]
The payoff of an optimal representation is a subgradient of the adapted
return. To make this statement in the ambient space, extend \(J^\star\) to
\(\R^A\) by the same maximum over \(x^\top q(\theta)\). The set
\(q(\Theta)\) is compact and bounded, so the extension is finite, convex,
and Lipschitz. For \(x\in\Delta_A\), the subdifferential and directional
derivative of this maximum satisfy
\begin{align}
\partial J^\star(x)
&=\operatorname{co}\{q(\theta):\theta\in\cM_0(x)\},
\label{eq:envelope-subdifferential}\\
J^{\star\prime}(x;w)
&=\max_{\theta\in\cM_0(x)}q(\theta)^\top w.
\label{eq:envelope-directional-formula}
\end{align}
For any selected optimizer \(\theta\in\cM_0(x)\), its replicator
direction \(w=C(x)q(\theta)\) consequently obeys
\begin{equation*}
J^{\star\prime}(x;w)
\ge q(\theta)^\top C(x)q(\theta).
\end{equation*}
The quadratic form on the right is the variance of action returns under \(x\),
which proves \cref{eq:envelope-directional}.

The replicator mobility also determines the behavior at the boundary. For
any payoff vector \(g\),
\begin{equation*}
[C(x)g]_a=x_a(g_a-x^\top g).
\end{equation*}
Each coordinate therefore solves a scalar linear equation with bounded
coefficient along the trajectory. A coordinate starting at zero remains zero,
so every simplex face is invariant. For a positive coordinate,
\begin{equation*}
\frac{d}{dt}\log x_a=q_a(\theta)-x^\top q(\theta)
\end{equation*}
almost everywhere. Bounded payoffs also prevent that coordinate from reaching
zero at any finite time, so a trajectory retains its initial support.

Along an absolutely continuous trajectory, the scalar function
\(J^\star(x(t))\) is absolutely continuous because \(J^\star\) is
Lipschitz. At almost every time, both \(x(t)\) and \(J^\star(x(t))\) have
time derivatives and the selected dynamics hold. At such a time, fix
\(g=q(\theta)\in\partial J^\star(x)\). The subgradient inequality for
forward increments bounds the scalar derivative from below by
\(g^\top\dot x\); backward increments bound the same derivative from
above. Hence
\begin{equation*}
\frac{d}{dt}J^\star(x(t))=g^\top\dot x=g^\top C(x)g.
\end{equation*}
The coordinate formula then identifies the variance and squared-speed
expressions in \cref{eq:envelope-energy}. This argument requires time
differentiability of the scalar composition, without requiring ambient
differentiability of the envelope.

Convexity turns this ascent identity into a restriction on attracting points.
Suppose a locally asymptotically stable equilibrium \(\bar x\) has support
\(S\) with \(|S|\ge2\), under an exact-optimizer selection generating a
semiflow. Its support face \(F_S\) is invariant. Every sufficiently nearby
\(y\in F_S\) must converge to \(\bar x\), so monotonicity and continuity
give
\begin{equation*}
J^\star(y)\le J^\star(\bar x).
\end{equation*}
The point \(\bar x\) lies in the relative interior of its face and is a
relative local maximum of the convex envelope. For every sufficiently small
tangent displacement \(w\), local maximality bounds both
\(J^\star(\bar x+w)\) and \(J^\star(\bar x-w)\) by
\(J^\star(\bar x)\). The convex midpoint inequality forces equality in
both bounds. Thus the envelope is constant in a relative neighborhood of
\(\bar x\).

At any point \(y\) in a smaller such neighborhood, the subgradient
inequality applies to both signs of the feasible directions
\(e_i-e_j\), \(i,j\in S\). Local constancy forces every
\(g\in\partial J^\star(y)\) to have equal coordinates on \(S\), and
therefore \(C(y)g=0\). All of these nearby points are stationary under
every exact optimizer selection. They cannot converge to a distinct
\(\bar x\), contradicting local attraction.
\end{proof}

The semiflow qualification matters because a measurable optimizer selector
alone need not yield unique trajectories or continuous dependence on initial
conditions. The differential inclusion
\(\dot x\in C(x)\partial J^\star(x)\) supplies a formulation with a
nonempty compact convex set of velocities and upper semicontinuous dependence
on \(x\). The preceding proof also rules out strong local asymptotic
stability of a mixed point for this inclusion, where attraction must hold for
all nearby solutions. Any claim about a particular single-valued dynamics
still needs its optimizer selection specified.

\paragraph{A Strict Directional Inequality at an Optimizer Tie.}
An optimizer tie can make the selected payoff underestimate the envelope's
pointwise directional derivative. Take two available payoff vectors
\(q=(1,-1,0)\) and \(q'=2q\), and let \(x=(1/3,1/3,1/3)\). Both maximize
return at \(x\). The choice \(q\) produces \(v=C(x)q=q/3\), for which
\begin{equation*}
q^\top v=2/3,
\qquad
J^{\star\prime}(x;v)=\max\{q^\top v,q'^\top v\}=4/3.
\end{equation*}
The branch \(q'\) becomes strictly better immediately along the forward
direction, so \(q\) cannot remain the selected exact optimizer there. Such
a switching time permits a strict pointwise inequality while preserving the
almost-everywhere identity along an exact-optimizer trajectory.

\paragraph{Nonattracting Equilibria on a Flat Envelope.}
The theorem excludes asymptotically stable mixed points, but permits stable
points without attraction. For two actions with \(x=(p,1-p)\), let the
available payoff vectors be \(q^0=(0,0)\) and \(q^1=(1,-1)\). Their
envelope is
\begin{equation*}
J^\star(p)=\max\{0,2p-1\}.
\end{equation*}
At every \(p<1/2\), the unique active vector \(q^0\) gives zero velocity.
Each interior point in this interval is Lyapunov stable because a sufficiently
small neighborhood consists entirely of stationary points. These points fail
to attract neighboring initial conditions and have lower return than
\(p=1\). At the tie \(p=1/2\), the selection \(q^0\) also gives a
stationary trajectory. The distinction between stability and attraction is
therefore necessary in \cref{thm:fast-envelope-ascent}.

\paragraph{Nonlinear Return in Raw MDP Policy Probabilities.}
The affine-mixture assumption concerns the coordinates used in the theorem.
For comparison, consider a single nonterminal state where a stationary policy
terminates with reward one with probability \(p\) and otherwise returns to
the same state with reward zero. With discount \(0<\gamma<1\), its value satisfies
\begin{equation*}
V(p)=\frac{p}{1-\gamma+\gamma p},
\qquad
V''(p)=\frac{-2\gamma(1-\gamma)}{(1-\gamma+\gamma p)^3}<0.
\end{equation*}
The strictly negative second derivative rules out the required affine
structure in this raw policy probability. State--action occupancy coordinates may express
return linearly, but a change to those coordinates also changes the actor
field. A replicator in policy probabilities therefore requires a separate
calculation before the theorem can be applied in those coordinates.

\paragraph{Envelope Descent During Finite-Rate Adaptation.}
The exact-response hypothesis can also fail while the affine-mixture
structure remains intact. Let \(x=(p,1-p)\), \(\theta\in[-1,1]\), and
both base rewards be zero, with distortions
\begin{equation*}
D_1(\theta)=(\theta-1)^2,
\qquad
D_2(\theta)=(\theta+1)^2.
\end{equation*}
The optimal representation \(\theta^\star(p)=2p-1\) gives
\begin{equation*}
R^\star(p)=4p(1-p),
\qquad
J^\star(p)=-4p(1-p),
\qquad
q_1(\theta)-q_2(\theta)=4\theta.
\end{equation*}
Now let the representation follow gradient flow at rate \(\kappa>0\)
and drive the actor with its current payoff,
\begin{equation*}
\dot\theta=-2\kappa[\theta-(2p-1)],
\qquad
\dot p=4\theta p(1-p).
\end{equation*}
At \(p=3/4\), \(\theta=-1/2\), the representation favors the action
opposite to the one favored by the adapted-return slope,
\begin{equation*}
J^{\star\prime}(p)=2,
\qquad
\dot p=-3/8,
\qquad
\frac{d}{dt}J^\star(p)=-3/4<0.
\end{equation*}
This state has strictly decreasing optimized return because the current
representation has not reached its optimum. The local stability conclusions
of \cref{prop:general-finite-rate} remain available under its Hessian and
coupling assumptions; they do not require a global envelope ascent identity.

\subsection{Joint Deficit and Scarcity Necessity}
\label{app:joint-deficit}

Recall the three bounds in \cref{thm:joint-deficit}.
For any $\theta\in\cM_\varepsilon(x)$, the weighted excess satisfies
\begin{equation}
\sum_a x_a\psi_a(\theta)
\le \Psi(x)+\varepsilon\le\varepsilon_J+\varepsilon.
\label{eq:joint-budget}
\end{equation}
For actions with $x_ax_b>0$, this gives
\begin{equation}
\left|(q_a(\theta)-q_b(\theta))
-(\bar q_a-\bar q_b)\right|
\le\frac{\Psi(x)+\varepsilon}{\min\{x_a,x_b\}}.
\label{eq:joint-gap-bound}
\end{equation}
If $|\xi_i-\xi_j|\le\zeta$, $\bar q_i>\bar q_j$, and
$\widetilde q_i\le\widetilde q_j$, the reversal requires
\begin{equation}
\Psi(x)+\varepsilon
\ge x_i\bigl(\bar q_i-\bar q_j-\zeta\bigr)_+.
\label{eq:scarcity-necessary}
\end{equation}
In the three-feature model, $\Psi(p)=L^\star(p)$ and
$\varepsilon_J=7/16$, attained by an optimal code at $p=1/2$.

\begin{proof}[Proof of \cref{thm:joint-deficit}]
The quantity \(\Psi(x)\) measures the smallest weighted excess above the
actionwise minima. Let \(\theta_J\) attain the joint approximation radius
\(\varepsilon_J\), as guaranteed by compactness and continuity. The
nonnegative component excesses give
\begin{equation*}
0\le \Psi(x)
\le \sum_a x_a\psi_a(\theta_J)
\le\varepsilon_J.
\end{equation*}
For an approximate optimizer \(\theta\in\cM_\varepsilon(x)\), the
same baseline subtracted from its loss gives
\begin{equation*}
\sum_a x_a\psi_a(\theta)
=R_x(\theta)-\sum_a x_aD_a^\star
\le \Psi(x)+\varepsilon.
\end{equation*}
Together these inequalities prove \cref{eq:joint-budget}.

A pairwise ordering depends on the difference of two excesses. Assume
\(x_ax_b>0\). If \(\psi_a\ge \psi_b\), nonnegativity bounds that difference by
\begin{equation*}
\psi_a-\psi_b\le \psi_a
\le\frac{\Psi(x)+\varepsilon}{x_a}.
\end{equation*}
If \(\psi_b\ge \psi_a\), the corresponding bound uses \(x_b\). The payoff
error relative to the ideal ordering is exactly
\begin{equation*}
(q_a-q_b)-(\bar q_a-\bar q_b)=-(\psi_a-\psi_b).
\end{equation*}
Thus the larger excess alone controls the error, giving
\cref{eq:joint-gap-bound} without an additional factor of two.

A reversal gives a directional bound using the probability of choosing the better
ideal action. If \(\bar q_i>\bar q_j\),
\(|\xi_i-\xi_j|\le\zeta\), and \(\widetilde q_i\le\widetilde q_j\), then
\begin{equation*}
\psi_i-\psi_j \ge \bar q_i-\bar q_j+\xi_i-\xi_j
\ge \bar q_i-\bar q_j-\zeta. \end{equation*}
When the last expression is positive, the weighted-excess bound and
\(\psi_j\ge0\) imply
\begin{equation*}
\Psi(x)+\varepsilon\ge x_i\psi_i \ge x_i(\psi_i-\psi_j)
\ge x_i(\bar q_i-\bar q_j-\zeta). \end{equation*}
When that expression is nonpositive, the positive-part bound in
\cref{eq:scarcity-necessary} follows from \(\Psi(x)+\varepsilon\ge0\).

If \(\varepsilon_J=0\), the representation \(\theta_J\) attains every
component minimum simultaneously. Hence
\(R^\star(x)=\sum_a x_aD_a^\star\) throughout the simplex. At a full-support
action distribution, every exact optimizer has \(\sum_a x_a\psi_a=0\); the positive
weights and nonnegative excesses force \(\psi_a=0\) for every action.
Its unperturbed deployed values therefore equal the ideal values.
\end{proof}

\paragraph{A Sharp Bound at Positive Visitation.}
The coefficient \(1/\min\{x_a,x_b\}\) cannot be improved for an arbitrary
compact class, even under exact optimization. Fix a positive excess scale
\(\varepsilon>0\), let \(x=(p,1-p)\) with
\(0<p\le1/2\), and consider a finite class with excess vectors
\begin{equation*}
(\varepsilon,\varepsilon),\qquad
(\varepsilon/p,0),\qquad
(0,\Upsilon),
\end{equation*}
where \(\Upsilon>\varepsilon/(1-p)\). Each component has minimum zero,
and \(\varepsilon_J=\varepsilon\). The first two representations both
minimize weighted excess at value \(\Psi(x)=\varepsilon\). The second
therefore attains the bound with pairwise error
\(\varepsilon/p=\Psi(x)/\min_a x_a\).

\paragraph{Unconstrained Excess on an Unvisited Action.}
Positive visitation is essential to the componentwise-optimality conclusion.
A class containing excess vectors \((0,0)\) and \((0,\Upsilon)\), with \(\Upsilon>0\),
has \(\varepsilon_J=0\). Both representations are nevertheless exact
optimizers at the first-action vertex because the second excess has zero
weight. If \(0<\bar q_2-\bar q_1<\Upsilon\), the second representation reverses
the deployed ordering at that vertex despite the jointly optimal alternative.
An absence claim for every boundary optimizer therefore needs an additional
selection condition; positive visitation removes this ambiguity. This example
concerns the ordering at a vertex and does not establish local attraction.

\subsection{Replay with a Protected Fitting Weight}
\label{app:general-replay}

\begin{proof}[Proof of \cref{cor:general-replay}]
Apply \cref{thm:joint-deficit} at the fitting weights $r=r(x)$.
Every $\theta\in\cM_\varepsilon(r)$ satisfies
\[
\sum_a r_a\psi_a(\theta)\le\Psi(r)+\varepsilon
\le\varepsilon_J+\varepsilon.
\]
Since each excess is nonnegative and $r_k\ge\beta$, this implies
$\psi_k(\theta)\le(\varepsilon_J+\varepsilon)/\beta$. For every $j\ne k$,
\begin{align*}
\widetilde q_k-\widetilde q_j
&=(\bar q_k-\bar q_j)-\psi_k+\psi_j+\xi_k-\xi_j\\
&\ge\Delta-\frac{\varepsilon_J+\varepsilon}{\beta}-\zeta=m.
\end{align*}
The positive gap is uniform in the policy and in all admissible fitting
selections. Compactness bounds the true returns, and the relative-error
bound keeps all return differences bounded. Thus any initially positive
policy coordinate stays positive at finite times. For $x_j(0)>0$,
the replicator gives
\[
\frac{d}{dt}\log\frac{x_j}{x_k}
=\widetilde q_j-\widetilde q_k\le-m
\quad\text{a.e.}
\]
Integration proves the odds bound. Coordinates initially equal to zero
remain zero, so the same bound holds for them. Summing the ratios yields
\[
1-x_k(t)\le\frac{R_0e^{-mt}}{1+R_0e^{-mt}},
\qquad R_0=\frac{1-x_k(0)}{x_k(0)},
\]
and hence convergence to $e_k$. This is a conclusion for each admissible
solution defined on $t\ge0$, not an existence claim for arbitrary
optimizer selections.
\end{proof}

\paragraph{A guarantee with unavoidable representation error.}
The condition does not require $\varepsilon_J=0$. In the tied
three-feature task, $\varepsilon_J=7/16$ and the uniquely best separately
adapted action is $E$ whenever $\delta>0$. With exact fitting, exact
returns, and uniform reference data $\nu=(1/2,1/2)$, the replay mixture
has $r_E(x)\ge\alpha/2$. Consequently,
$m=\delta-7/(8\alpha)$, which is positive when $\alpha\delta>7/8$.
Taking $\alpha=\delta=19/20$ gives $m=11/380>0$ at the original capacity,
although ordinary on-policy fitting has the self-confirming trap of
\cref{thm:trap}. Thus the guarantee covers a nontrivial mixture of
current and retained data in a model with an existing trap. It is
conservative: at $\alpha=1$, uniform fitting gives the actual return gap
$\delta$ for every $\delta>0$, including values for which the bound
does not certify recovery.

\paragraph{Policy recovery and deployed return.}
The conclusion selects the best separately adapted action; it need not
attain that action's separately adapted return while replay remains in
effect. For example, take a finite parameter class with distortion vectors
$(0,2)$, $(2,0)$, and $(0.1,0.1)$, and base rewards $b=(1,0)$.
Both separate distortion minima are zero, while $\varepsilon_J=0.1$.
Under uniform replay, the third parameter choice is uniquely optimal.
With $\beta=1/2$ and zero errors, the certificate gives $m=0.8$ and
the actor converges to action 1, but its deployed return is $0.9$
rather than $\bar q_1=1$. Recovering the latter requires refitting the
model to the selected action. Finally, if $x_k(0)=0$, the replicator
never introduces that action, even when its observations are present
in replay.

\Cref{fig:replay-ranking} uses this finite class at the same policy
$x=(0.02,0.98)$ in both panels. Fitting to $x$ selects distortion
$(2,0)$ and gives deployed returns $(-1,0)$, reversing the separately
adapted ordering $\bar q=(1,0)$. Mixing in uniform reference data with
$\alpha=0.4$ gives $r=(0.212,0.788)$, selects distortion $(0.1,0.1)$,
and restores the ordering with returns $(0.9,-0.1)$. The guaranteed
floor is $\beta=0.2$, distinct from the displayed weight $r_1=0.212$.
With exact fitting and values, \cref{cor:general-replay} gives
$m=1-0.1/0.2=0.5>0$, while the actual return gap in this panel is one.

\begin{figure}[t]
\centering
\includegraphics[width=\textwidth]{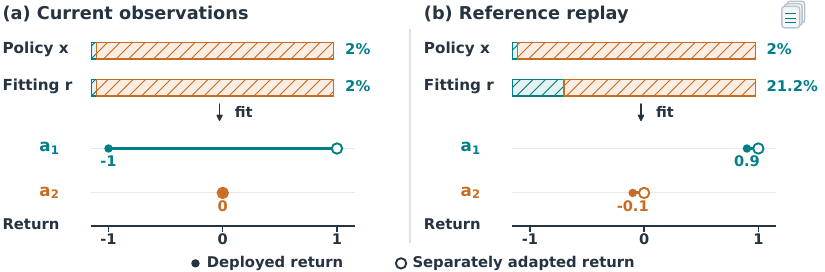}
\caption{Replay restores return ordering with residual representation
error. Both panels use the same policy and finite parameter class, with
40\% uniform reference replay added in (b). Percentages denote $a_1$'s
weight. Hollow markers use separate fits, whereas solid markers use one
shared fit (construction in \cref{app:general-replay}).}
\label{fig:replay-ranking}
\end{figure}

\subsection{Local Persistence and Two-Action Basins}
\label{app:local-persistence}

The deficit bound limits a possible reversal. A strict invasion margin adds
the local sign information needed for attraction, uniformly over the selected
optimizer.

\begin{proof}[Proof of \cref{thm:simplex-robustness}]
Compactness transfers the endpoint margin to sufficiently accurate
optimizers at nearby action distributions. To see this, let \(x_n\to e_i\),
\(\varepsilon_n\downarrow0\), and
\(\theta_n\in\cM_{\varepsilon_n}(x_n)\). Every convergent subsequence of
\(\theta_n\) has a limit in \(\cM_0(e_i)\). Indeed, approximate optimality
means that, for every \(\phi\in\Theta\),
\begin{equation*}
R_{x_n}(\theta_n)\le R_{x_n}(\phi)+\varepsilon_n.
\end{equation*}
The limit satisfies \(D_i(\bar\theta)\le D_i(\phi)\), which establishes
the claimed endpoint optimality.

Fix \(0\le\zeta<\gamma_i\) and set
\(\underline\gamma_i:=(\gamma_i-\zeta)/2>0\). There are a neighborhood \(U_i\) and a
tolerance \(\bar\varepsilon>0\) for which every
\(x\in U_i\), \(0\le\varepsilon\le\bar\varepsilon\),
\(\theta\in\cM_\varepsilon(x)\), and \(j\ne i\) satisfy
\begin{equation*}
q_i(\theta)-q_j(\theta)\ge(\gamma_i+\zeta)/2.
\end{equation*}
Otherwise a sequence approaching the vertex with tolerances tending to zero
would, by the preceding compactness argument, yield an endpoint optimizer
violating the definition of \(\gamma_i\). Relative value errors reduce
each resident advantage by at most \(\zeta\). The replicator dynamics
\(\dot x_a=x_a(\widetilde q_a-x^\top\widetilde q)\), with
\(\widetilde q_a=q_a(\theta)+\xi_a\), therefore give, wherever
\(x_i,x_j>0\),
\begin{equation}
\frac{d}{dt}\log\frac{x_j}{x_i}
=q_j(\theta)+\xi_j-q_i(\theta)-\xi_i\le-\underline\gamma_i,
\qquad j\ne i.
\label{eq:odds-contraction}
\end{equation}
An invader with zero mass remains absent, so only positive invaders require a
logarithm.

To turn this conditional rate into a stable neighborhood, let
\(r_j=x_j/x_i\) and \(r=\sum_{j\ne i}r_j\), and choose a sufficiently
small odds sublevel set inside \(U_i\). While a solution remains in that
set, every ratio obeys \(r_j(t)\le r_j(0)e^{-\underline\gamma_i t}\). The total odds cannot
increase, so the solution cannot leave the set, and
\begin{equation*}
1-x_i(t)=\frac{r(t)}{1+r(t)}\le r(0)e^{-\underline\gamma_i t}.
\end{equation*}
On this neighborhood, \(r(0)=(1-x_i(0))/x_i(0)\) is bounded by a fixed
multiple of \(1-x_i(0)\). The estimate therefore proves local exponential
stability with constants uniform over all admissible selections and errors.

At a pure policy, every exact optimizer has
\(D_i(\theta)=R^\star(e_i)\). Its resident return is consequently the
selection-independent value \(b_i-R^\star(e_i)\), and the stated strict
comparison makes the lower-return stable vertex strictly suboptimal.

The same margins also separate the optimizer sets of distinct stable
vertices. A representation belonging to both \(\cM_0(e_i)\) and
\(\cM_0(e_k)\) would have to satisfy both
\(q_i(\theta)>q_k(\theta)\) and \(q_k(\theta)>q_i(\theta)\), which is
impossible.
\end{proof}

The pure-$S$ margins in \cref{thm:trap,thm:matching-complement,thm:untied}
are $1-\delta$, $(d-1)^{-1}-\delta$, and $s_\eta-\delta$, respectively.
Each is strictly positive in its trap regime, so
\cref{thm:simplex-robustness} preserves local attraction to $S$ under
sufficiently small fitting and relative value errors.

For two actions, the resident advantage at the low endpoint is
\(\Delta_{D,0}-\delta\), while the advantage at the high endpoint is
\(\delta-\Delta_{D,1}\). These margins specialize
\cref{thm:simplex-robustness} without requiring a unique optimizer or a
differentiable envelope.

\begin{corollary}[Selection-Robust Endpoint Criterion]
\label{prop:selection-stability}
Define
\begin{equation*}
\Delta_{D,0}=\min_{\theta\in\cM(0)}\Delta_D(\theta),\qquad
\Delta_{D,1}=\max_{\theta\in\cM(1)}\Delta_D(\theta).
\end{equation*}
If \(\Delta_{D,1}<\delta<\Delta_{D,0}\), both endpoints are locally exponentially stable
relative to \([0,1]\), with neighborhoods and rates uniform over all
optimizer-selected trajectories. If additionally
\(\delta>R^\star(1)-R^\star(0)\), then
\(J^\star(0)<J^\star(1)\).
\end{corollary}

\begin{proof}
In the notation of \cref{thm:simplex-robustness}, the worst-selection
invasion margins are \(\gamma_S=\Delta_{D,0}-\delta>0\) and
\(\gamma_E=\delta-\Delta_{D,1}>0\). Applying that theorem at both
vertices with exact fitting and zero payoff error gives the uniform
local exponential stability. The return comparison follows from
\(J^\star(1)-J^\star(0)=\delta-[R^\star(1)-R^\star(0)]\).
\end{proof}

A differentiable envelope determines the intervening basins as well as the
endpoint signs. A strictly decreasing distortion slope makes the deployed
gap cross zero once.

\begin{proposition}[Endpoint-Slope Criterion]
\label{prop:general-stability}
Assume \(R^\star\) is differentiable on \((0,1)\), its derivative decreases
strictly and continuously from \(\Delta_{D,0}\) to \(\Delta_{D,1}\), and
\(\Delta_{D,1}<\delta<\Delta_{D,0}\). Then \cref{eq:general-flow} has a unique unstable interior
fixed point and both endpoints are asymptotically stable relative to
\([0,1]\). If additionally
\(\delta>R^\star(1)-R^\star(0)\), the low endpoint has strictly lower return.
\end{proposition}

\begin{proof}
Continuity and strict monotonicity give a unique \(p_\delta\) with
\(R^{\star\prime}(p_\delta)=\delta\). The gap
\(J^{\star\prime}=\delta-R^{\star\prime}\) is negative below this point
and positive above. Every initial probability in \((0,p_\delta)\) therefore
moves toward zero, and every initial probability in \((p_\delta,1)\) moves
toward one. The interior equilibrium is unstable. The strict limiting gap
signs give local endpoint attraction, as in
\cref{prop:selection-stability}. Their return difference remains
\begin{equation*}
J^\star(1)-J^\star(0)
=\delta-[R^\star(1)-R^\star(0)].
\end{equation*}
\end{proof}

For any continuous mobility \(m(p)>0\) in the interior, the chain rule
along \(\dot p=m(p)J^{\star\prime}(p)\) gives
\begin{equation*}
\frac{d}{dt}J^\star(p(t))
=m(p)[J^{\star\prime}(p)]^2\ge0.
\end{equation*}
For replicator mobility \(m(p)=p(1-p)\), this is the differentiable
two-action form of \cref{eq:envelope-energy}. At a nonsmooth tie, its
pointwise directional version remains an inequality, while the replicator
identity still holds almost everywhere along an absolutely continuous
exact-optimizer trajectory. The endpoint return comparison explains how
ascent of a convex envelope can support stable boundary points with unequal
returns.

\subsection{Local Geometry with Finite-Rate Adaptation}

A finite adaptation rate adds representation coordinates to the dynamics.
The following result preserves the local sink and saddle classifications when
the representation loss has positive curvature at the equilibrium. It makes
no assumption that the representation tracks its optimum along the full
trajectory.

\begin{proposition}[Finite-Rate Local Geometry]
\label{prop:general-finite-rate}
Let \(z\) range over an open set \(\mathcal Z\subseteq\R^d\). Suppose
\(D_a\in C^2(\mathcal Z)\) and
\begin{equation*}
\dot z=-\kappa P(x,z)\nabla_zR_x(z),\qquad
\dot x_a=x_a\left(q_a(z)-\sum_bx_bq_b(z)\right),
\end{equation*}
where \(\kappa>0\) and \(P\) is \(C^1\), symmetric, and positive definite
near the equilibria. At a pure policy \(e_i\), let \(z_i\in\mathcal Z\) be an
isolated global representation optimum with
\(H_i:=\nabla_z^2D_i(z_i)\succ0\). If
\(q_i(z_i)>q_j(z_i)\) for every \(j\ne i\), then \((e_i,z_i)\) is a
hyperbolic sink for every \(\kappa>0\).

At an interior equilibrium, use the first \(A-1\) policy probabilities as
coordinates \(y\), write
\(x(y)=(y_1,\ldots,y_{A-1},1-\mathbf 1^\top y)\), let
\(H=\nabla_z^2R_{x(y)}(z)\succ0\), and form
\begin{equation*}
G=\bigl[\nabla_z(D_1-D_A),\ldots,
\nabla_z(D_{A-1}-D_A)\bigr].
\end{equation*}
If \(G\) has full column rank, the linearization has exactly \(d\) negative
and \(A-1\) positive real eigenvalues. Thus the joint equilibrium is a
hyperbolic saddle of unstable dimension \(A-1\), for every \(\kappa>0\).
\end{proposition}

\begin{proof}
At a pure policy, the invader probabilities give local policy coordinates.
The representation gradient vanishes at \(z_i\), so derivatives of the
preconditioner contribute zero to the linearization. An invader equation has
zero derivative with respect to \(z\) because its policy mass is zero.
The Jacobian at \((e_i,z_i)\) is therefore block triangular.

Its representation block is \(-\kappa PH_i\), which is similar to the
negative definite symmetric matrix
\(-\kappa P^{1/2}H_iP^{1/2}\). Every representation eigenvalue is negative,
and the remaining eigenvalues are the strict invader disadvantages
\(q_j(z_i)-q_i(z_i)<0\). Thus the pure equilibrium is a hyperbolic sink
for every \(\kappa>0\).

At an interior equilibrium, put \(n=A-1\) and use
\(x(y)=(y_1,\ldots,y_n,1-\mathbf 1^\top y)\). The reduced policy mobility
is
\begin{equation*}
C(y)=\operatorname{diag}(y)-yy^\top.
\end{equation*}
All action probabilities are positive, making this reduced matrix positive
definite. In particular, with \(x_A=1-\sum_{a=1}^ny_a>0\),
Cauchy--Schwarz gives, for every \(w\ne0\),
\begin{equation*}
w^\top C w
=\sum_a y_aw_a^2-\left(\sum_a y_aw_a\right)^2
\ge x_A\sum_a y_aw_a^2>0.
\end{equation*}
Let \(a_j=q_j-q_A\). The reduced actor equation is
\(\dot y=C(y)a(z)\), and all payoff differences vanish at an interior
equilibrium. The loss-gradient coupling then gives
\begin{equation*}
\partial_y\nabla_zR_{x(y)}=G,
\qquad \nabla_za=-G^\top.
\end{equation*}
All quantities evaluated at the equilibrium consequently yield the Jacobian
\begin{equation*}
J=\begin{pmatrix}
-\kappa PH&-\kappa PG\\
-CG^\top&0
\end{pmatrix}.
\end{equation*}
The change of basis
\(T=\operatorname{diag}(P^{1/2},\kappa^{-1/2}C^{1/2})\) makes
\(T^{-1}JT\) symmetric. Its upper-left block is
\(-\kappa P^{1/2}HP^{1/2}\), and its off-diagonal block is
\(-\sqrt\kappa P^{1/2}GC^{1/2}\). The representation block is negative
definite, while its Schur complement is
\begin{equation*}
C^{1/2}G^\top H^{-1}GC^{1/2}\succ0.
\end{equation*}
The Schur complement is positive definite precisely because \(G\) has
full column rank. Sylvester's law of inertia therefore gives exactly
\(d=\dim z\) negative and \(A-1\) positive eigenvalues, with no zero
eigenvalues. Similarity to a symmetric matrix makes all of these eigenvalues
real, proving the stated saddle classification for every \(\kappa>0\).
\end{proof}

\section{Executable Algebra Checks}
\label{app:numerical}

The supplementary scripts check algebraic identities and numerical
implementations underlying the figures. These checks supplement the proofs.
With seed 20260808, \texttt{scripts/check\_theory.py} regenerates
\texttt{figures/phase\_data.csv} and compares the pair-loss formula with
800,000 Gaussian samples. It compares 24 randomized weighted-frame
solutions with \(1201\times1201\) angle grids and checks the phase
boundaries and envelope identity. The untied rational loss is compared with 300 direct matrix inversions and with endpoint slopes at six noise levels.

The same script reconstructs central angle representatives and checks their
stationary gradients and finite-rate Jacobian inertia at three values of \(p\)
and three values of \(\kappa\). It also checks the policy-floor symmetry
identity. For the codimension-one lift, the script verifies the exact block
construction for \(d=2,\ldots,8\) and checks the spectral lower bound on 480 deterministic random frames.

The separate deterministic script
\texttt{scripts/untied\_phase\_sweep.py} eliminates the controller exactly,
searches a \(361\times361\) global angle grid followed by three local
refinements, and writes \texttt{figures/untied\_phase\_data.csv}. Across
\(\eta\in\{0.2,1,5,20\}\), the largest optimized-value symmetry residual was
\(2.0\times10^{-10}\), the largest slope antisymmetry residual was
\(6.6\times10^{-5}\), and no monotonicity violation was observed on the
51-point grid of \(p\) values.

The recorded maximum errors were \(8.47\times10^{-4}\) for pair-loss Monte
Carlo, \(1.39\times10^{-5}\) for frame value,
\(2.36\times10^{-3}\) for squared Gram entries,
\(7.4\times10^{-11}\) for the envelope identity,
\(1.1\times10^{-16}\) for the lifted block value (with minimum random-frame
objective and spectral-bound slacks \(0.32\) and \(7.2\times10^{-3}\)),
\(1.7\times10^{-16}\) for finite-rate central stationarity,
\(5.6\times10^{-16}\) for untied direct inversion, and
\(5.3\times10^{-7}\) for the exact endpoint slope. These endpoint checks
use \(\eta\in\{0.05,0.2,1,5,20,100\}\), separate from the four noise
levels in the sweep over \(p\). The smallest signed
untied slope magnitude at \(p\in\{0.4,0.6\}\) was
\(8.0\times10^{-5}\).

\section{Finite-sample Geometry and Deployed Control}
\label{app:sampled-check}

\begin{figure}[!ht]
  \centering
  \includegraphics[width=\textwidth]{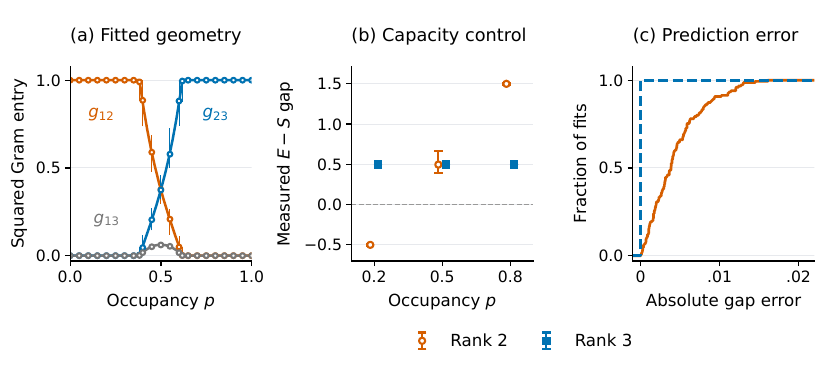}
  \caption{Finite-sample fits recover the predicted geometry and preference reversal at low \(p\); orthogonal capacity removes the reversal.
Curves in (a) are analytic predictions; markers and bars in (a--b) show
medians and 10--90\% ranges over 50 datasets per value of \(p\).}
\label{fig:e1-recovery}
\end{figure}

\begin{figure}[t]
  \centering
  \includegraphics[width=\textwidth]{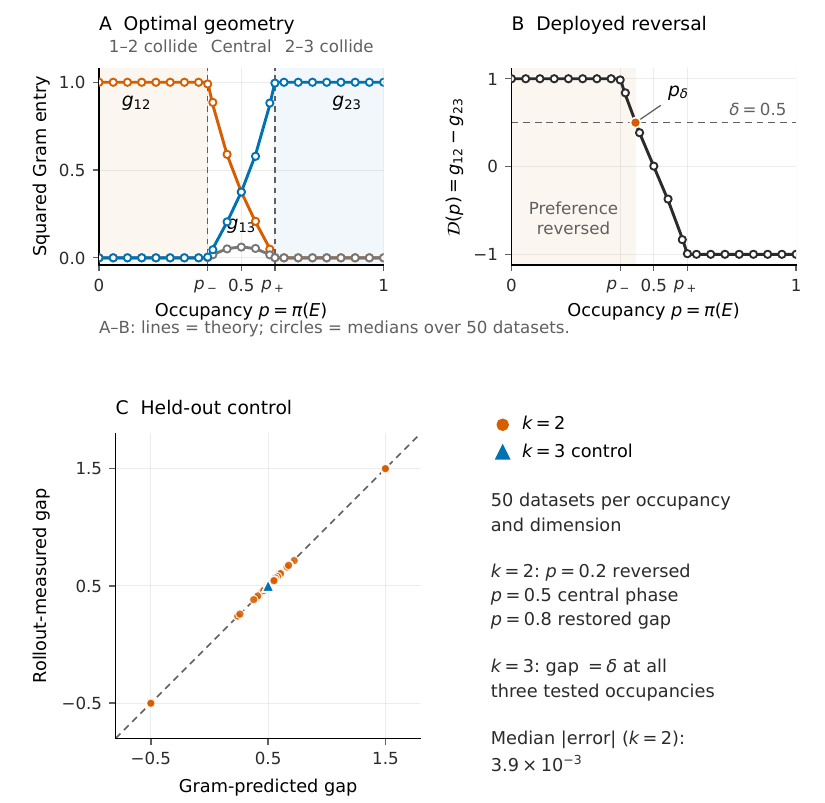}
  \caption{E1 recovers the solved chain from the action distribution to deployed control.
  Panel A shows the finite-sample squared Gram entries, with circles marking
the medians over 50 datasets. The medians lie on the three theory curves, and
the collision moves from pair \(\{1,2\}\) to pair \(\{2,3\}\) across the
golden-ratio boundaries \(p_\pm\), as \cref{thm:frame} predicts. Panel B
  then reads the same fits through the order parameter \(\cD(p)\), which
  crosses the task preference \(\delta=0.5\) at the separator \(p_\delta\),
  so the deployed preference is reversed on the shaded side. Panel C closes the chain by comparing the Gram-predicted gap with fresh
forced-rollout measurements. The scarce-rank fits at \(p=0.2\) and \(p=0.8\)
lie on the identity line at the reversed and restored gaps, while the central-phase fits spread along the line. The orthogonal \(k=3\) control stays at
gap \(\delta\), and the median absolute bridge error for rank-two codes is \(3.9\times10^{-3}\).}
  \label{fig:e1-solved}
\end{figure}

\paragraph{Frozen Protocol.}
E1 tests finite-sample optimization of the tied representation objective.
The protocol and pass/fail criteria were fixed before the primary run in
\texttt{experiments/sampled\_phase/PROTOCOL.md}. Each dataset contains
\(N=4096\) episodes from the branch/pair generator of \cref{sec:exact},
with independent standard-normal active feature values and \(\sigma=0\).
This noise choice preserves the predicted geometry because tied-code noise
adds only a \(V\)-independent constant in \cref{eq:branch-loss}.
If pair
\(\{i,j\}\) is active, its sampled reconstruction loss is
\begin{equation*}
\ell(V)=(v_i^\top v_j)^2(Z_i^2+Z_j^2).
\end{equation*}
Thus each fixed raw dataset has empirical objective
\begin{equation*}
\widehat L(V)=\widehat A g_{12}+\widehat B g_{13}+\widehat C g_{23},
\qquad
\E(\widehat A,\widehat B,\widehat C)=(p,1,1-p).
\end{equation*}
The implementation first applies the tied encoder and controller on raw samples, and an
automatic check requires its loss and angle gradient to equal those of this exact sufficient-statistic reduction. The maximum observed discrepancies were
\(6.7\times10^{-16}\) for both loss and gradient.

The grid of \(p\) values contains \(\{0,0.05,\ldots,1\}\) and the two population
phase thresholds, for 23 values of \(p\) with 50 independent datasets each. The
common-rotation gauge fixes \(v_3=(1,0)\), and the other two line angles
start independently and uniformly on \([0,\pi)\). Each dataset uses eight
initializations and unregularized full-batch gradient descent, with step
size \(0.05\), at most 20,000 steps and projected-gradient tolerance
\(10^{-10}\). All initializations are retained. The independently derived
weighted-frame formula supplies an exact empirical optimum after training,
separating optimization error from sampling-induced movement of the phase
boundaries. The main markers select the lowest empirical-loss initialization
to match the theorem's global-best-response assumption. The single-start
rate is reported separately.

\paragraph{Results and Controls.}
All frozen criteria were met. Among 1,150 independent datasets,
1,148 best-of-eight solutions had empirical suboptimality below \(10^{-8}\);
the two remaining errors were \(1.1\times10^{-8}\) and \(1.9\times10^{-8}\),
both at a population phase threshold. Across all 9,200 individual starts,
\(99.75\%\) met the same certificate. The median maximum squared-Gram error
relative to the exact empirical optimum was \(3.3\times10^{-20}\); its 90th,
99th, and maximum values were \(5.5\times10^{-10}\), \(1.6\times10^{-6}\), and
\(4.6\times10^{-4}\), respectively.

At \(p=0.2\), every dataset yielded \(g_{12}>0.99\) and
\(g_{13},g_{23}<0.01\); every dataset at \(p=0.8\) yielded the symmetric
\(2\)-\(3\) collision. Consequently all learned two-dimensional codes at
\(p=0.2\) reversed the deployed gap for the pre-fixed \(\delta=0.5\). We also
trained three unit vectors in \(\R^3\) on the same datasets sampled at interior values of \(p\)
and with the same optimization budget. All 150 controls had every squared Gram
entry below \(10^{-2}\) and retained a positive deployed gap. We exclude
three-dimensional endpoint controls because one angle is empirically
unidentified when its pair is absent.

Independent forced-branch evaluation used \(65{,}536\) fresh episodes per
branch for each of the 150 two-dimensional fits at
\(p\in\{0.2,0.5,0.8\}\) and their 150 three-dimensional controls. The
median absolute discrepancy between a code's
measured \(D_E-D_S\) and the Gram prediction \(g_{12}-g_{23}\) was
\(3.9\times10^{-3}\) for the two-dimensional codes, with 90th percentile
\(9.3\times10^{-3}\). For the orthogonal three-dimensional controls these
values were \(1.5\times10^{-22}\) and \(5.9\times10^{-22}\).

\Cref{fig:phase} shows sampling variation and the independent rollout
comparison. The population curve in panel (a) is exact; open marks and
10--90\% bars summarize independent finite-sample fits. Panels (b)--(c)
use fresh forced-branch rollouts to test whether the fitted geometry predicts
deployed distortion.

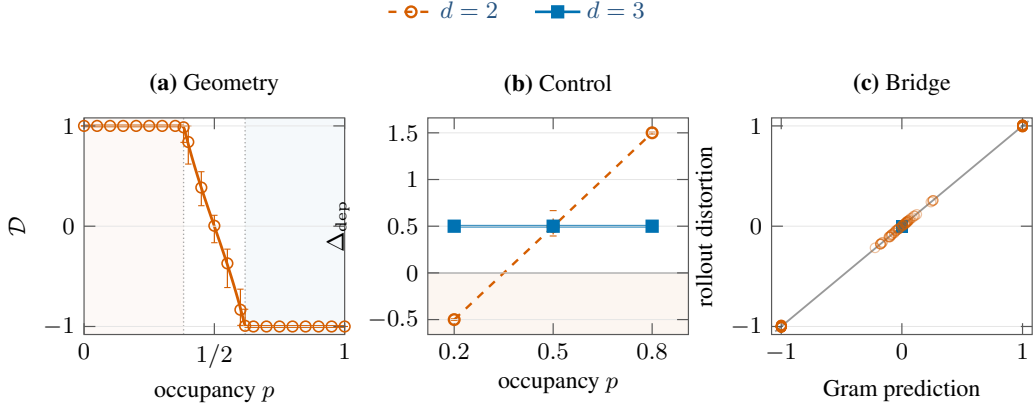
\begin{figure}[t]
    \centering
    \input{figures/phase_diagram}
    \caption{Finite-sample recovery of the solved chain: (a) collision
    transfer, (b) gap reversal at low \(p\) and its removal by orthogonal
    capacity, and (c) agreement between Gram prediction and fresh rollout.
    Marks and bars are medians and 10--90\% ranges over 50 datasets; panel (c)
    shows every dataset. The actor is not trained.}
    \label{fig:phase}
\end{figure}

\input{experiments/sampled_phase/outputs/sample_size_table_rows}
\begin{table}[!h]
\centering
\caption{Sampling variation in the empirical squared-Gram optimum. Entries
are the median / 90th percentile of the maximum absolute difference between
empirical-optimum and population-optimum squared Gram entries over 100
independent datasets per cell. The table shows the two phase thresholds and
the central value of \(p\); the released data also include four values
away from the thresholds.}
\label{tab:sample-size}
\small
\begin{tabular}{rccc}
\toprule
\(N\) & \(p=p_-\) & \(p=1/2\) & \(p=p_+\) \\
\midrule
\SampleSizeTableRows
\bottomrule
\end{tabular}
\end{table}

The sample-size study draws 100 datasets for each combination of
\(N\in\{256,1024,4096,16384\}\) and
\(p\in\{0.2,0.35,p_-,0.5,p_+,0.65,0.8\}\), for 2,800 datasets.
It evaluates the exact empirical optimum directly to isolate sampling
variation, without another optimization run or a fitted convergence rate.
The complete, unfiltered outputs include every dataset, initialization,
empirical certificate, held-out branch estimate, three-dimensional control, and
sample-size run. The experiment is deliberately limited to optimization of the
solved representation model. It does not test joint actor learning or establish
prevalence in nonlinear deep RL agents.

\section{Coupled Updates and Two-step Neural Control}
\label{app:e2-e3-gates}

\begin{table}[t]
\centering\footnotesize\setlength{\tabcolsep}{3.5pt}\renewcommand{\arraystretch}{1.1}
\caption{Registered quantities of the three experiments. Bounds are the
frozen pass criteria, and every verdict is reported, including the two
failures.}
\label{tab:results}
\begin{tabularx}{\textwidth}{@{}lXll@{}}
\toprule
Experiment & Quantity & Bound & Measured \\
\midrule
E1 & fits within \(10^{-8}\) of the exact empirical optimum & \(\ge0.99\) & 1,148 of 1,150 \\
E1 & median forced-rollout gap error, rank two & --- & \(3.9\times10^{-3}\) \\
E2 & basin agreement with \(\operatorname{sign}(p_0-p_\delta)\) & \(\ge0.99\) & 746 of 746 \\
E2 & separator error over 19 preferences, \(\kappa=80\) & \(\le0.01\) & 0.0044 \\
E2 & floor separator error, 4 points & \(\le0.01\) & \(6.3\times10^{-4}\) \\
E2 & replay transition error, 2 references & \(\le0.02\) & 0.0123 \\
E2 & orthogonality gate, fixed initialization & --- & fail \\
Bridge & constrained-class boundary error, 6 cells & \(\le0.01\) & \(\le0.007\) \\
Bridge & full-rank control trapped & 0 & 0 of 11,700 \\
Bridge & registered return gate, constrained arm & --- & fail \\
E3 & low-start rank-two seeds trapped & LCB \(>0.5\) & 96 of 100 (LCB 0.901) \\
E3 & high-start and rank-64 seeds trapped & 0 & 0 of 100 each \\
E3 & rank cost of return, low start & LCB \(>0\) & 0.225 [0.219, 0.230] \\
E3 & additional low-start cost & LCB \(>0\) & 0.133 [0.127, 0.137] \\
\bottomrule
\end{tabularx}
\end{table}

\Cref{fig:e3-distributions} retains the full endpoint distributions
behind \cref{fig:empirical-mechanism,tab:ppo-results}, including the four
low-start rank-two runs that finish favoring $E$. The main-text bars use
all 100 training seeds in each condition. Their 95\% percentile bootstrap
intervals use 20,000 within-condition resamples and quantify uncertainty
in the group means.
The cross-talk contrast is \(x_{12}-x_{23}\), and forced values include
the task preference \(\delta=0.15\). Both forced and deployed returns use
mean terminal controls, with stochastic routing for deployment. The
cutoff \(0.5\) classifies endpoints and is not a proved separator of
neural dynamics. Ranks share training seeds within each start, whereas
low and high starts use independent seed sets.

\begin{figure}[t]
  \centering
  \includegraphics[width=\textwidth]{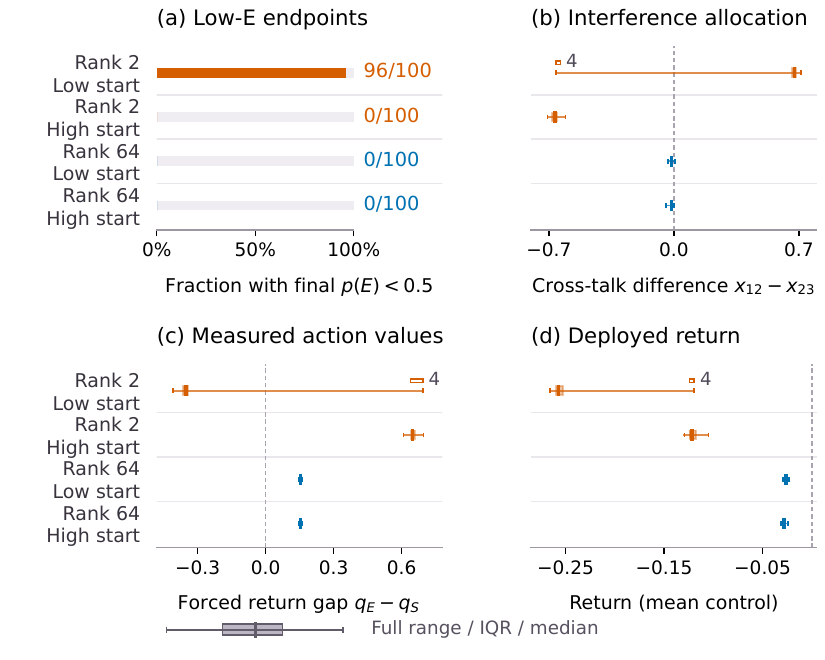}
  \caption{All 100 seeds per arm. Boxes, center lines, and whiskers show
interquartile ranges, medians, and full ranges. Outlined ranges marked 4
identify low-start rank-two runs ending at high \(E\).}
\label{fig:e3-distributions}
\end{figure}

\begin{figure}[t]
  \centering
  \includegraphics[width=\textwidth]{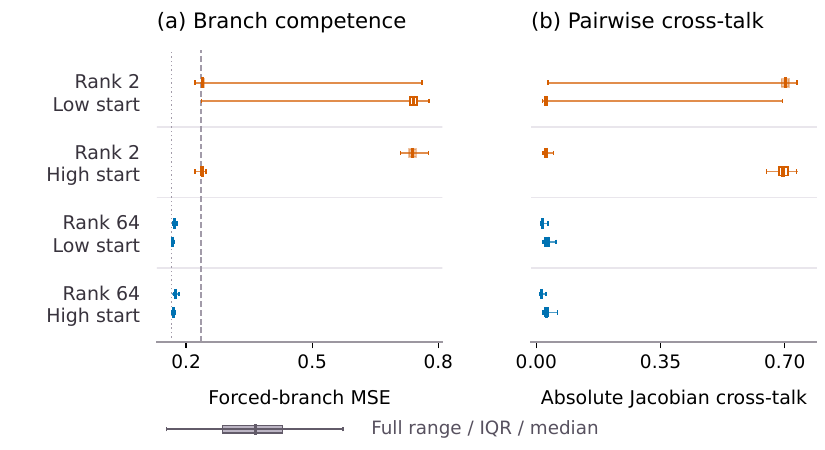}
  \caption{Branch competence and absolute cross-talk for the same 400
  PPO runs as \cref{fig:e3-bridge}. In each row, upper filled bands and
  lower open bands show (a) forced-branch reconstruction MSE on $S$ and
  $E$, respectively, and (b) cross-talk $x_{12}$ and $x_{23}$, on the
  pairs exclusive to $E$ and $S$, respectively. Cross-talk is a
  symmetrized absolute Jacobian response, not an encoder angle.
  Dotted and dashed MSE guides mark the Bayes floor ($0.165$) and an
  achievable rank-two reference ($0.235$); the latter is not a proved
  neural-class optimum. Boxes show interquartile ranges, center lines
  medians, and whiskers full ranges over all 100 training seeds per arm,
  including the four alternate-endpoint runs. Reconstruction uses the
  mean terminal control.}
  \label{fig:e3-final}
\end{figure}

\subsection{Closed-loop Basins and the Sampled Bridge Cohort}
\begin{figure}[t]
  \centering
  \includegraphics[width=\textwidth]{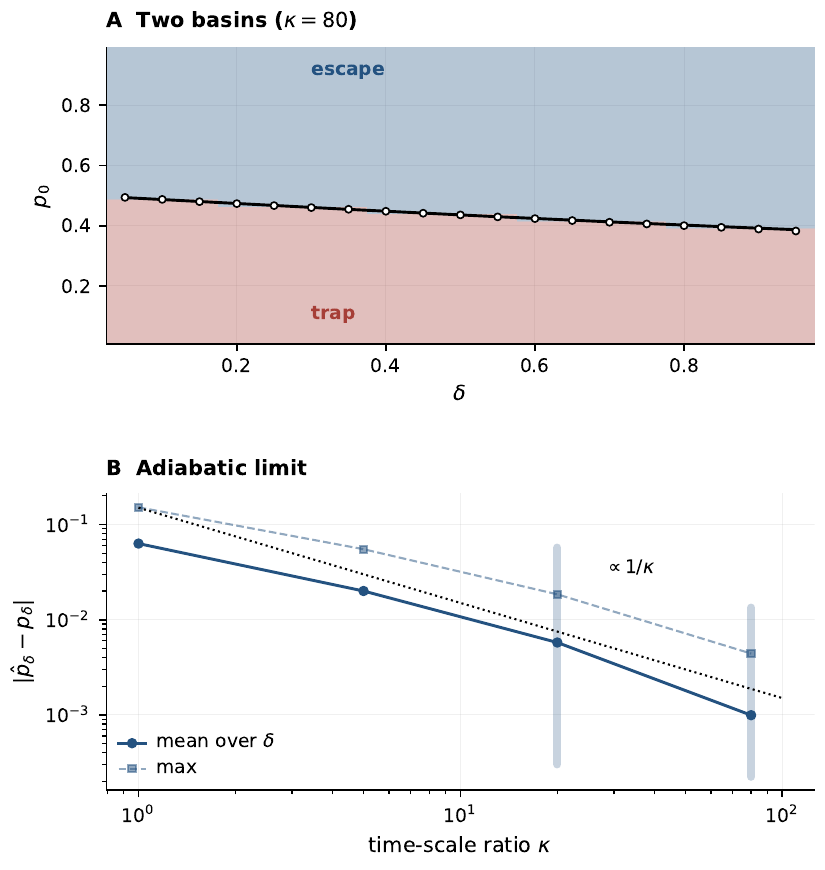}
  \caption{Closed-loop dynamics in the solved class. Panel A draws the two basins with
the analytic separator on the \(\kappa=80\) grid as an illustration, while
the registered basin gate runs on the \(\kappa=20\) grid
(\cref{app:e2-e3-gates}). Panel B shows the separator error decreasing roughly as \(1/\kappa\) over the tested update counts, with bands over ten representation initializations.}
  \label{fig:e2-core}
\end{figure}

The solved-class loop still gives the actor an exact representation map, so
a bridge cohort replaces that map with a sampled learner while keeping the
environment fixed. Three arms differ only in what the representation may express: a two-dimensional encoder constrained to the solved class, a two-dimensional
unconstrained encoder, and a full-rank three-dimensional control. Each arm
trains on fresh seeds under frozen gates, for 35{,}100 records in total.

The learned boundaries match their class predictions
(\cref{fig:e2-neural}). The constrained arm meets the solved-class boundary
within \(\pm0.007\) at all six registered cells. The unconstrained arm asks
more of the theory, because its population reference is optimized numerically and carries no global certificate; the measured boundaries
still meet this reference within \(\pm0.023\) at all six cells. Capacity
closes the comparison, with the full-rank control ending trapped in 0 of
11{,}700 runs. The visitation-floor intervention transfers as well, moving each boundary in
the predicted direction at all six registered pairs. As a development check
beyond the direction gate, the three constrained-class shift magnitudes also
cover their point predictions.

Within the runs, the collision and gap links of the mechanism hold, with
one qualified stratum. The predicted collision pair dominates in 1.000 of
trapped and 0.999 of escaped two-dimensional runs. The collision then predicts the forced gap with the predicted sign in 0.992 and 1.000 of the pooled runs.
Eleven of the twelve strata lie at 0.99 or above, and the one stratum at
0.718 coincides with the anomalous return stratum recorded in
\cref{app:e2-e3-gates}. The
registered return gate fails in the constrained arm, and the diagnosis is
an estimand error rather than a mechanism failure. This gate compared returns under a balanced reference action distribution, where the constrained class is
provably symmetric between the two endpoints, so it tested a quantity the
theory itself sets to zero. A development check of the on-policy return deficit agrees with the floor prediction \(-(1-2\epsilon)\delta\) at all twelve cells. The nominal return result of the unconstrained arm has no confirmatory reading either, since its analysis pooled the floor levels and resampled records rather than seeds, so the confirmatory return chain stays open pending a registered rerun.

\begin{figure}[t]
  \centering
  \includegraphics[width=\textwidth]{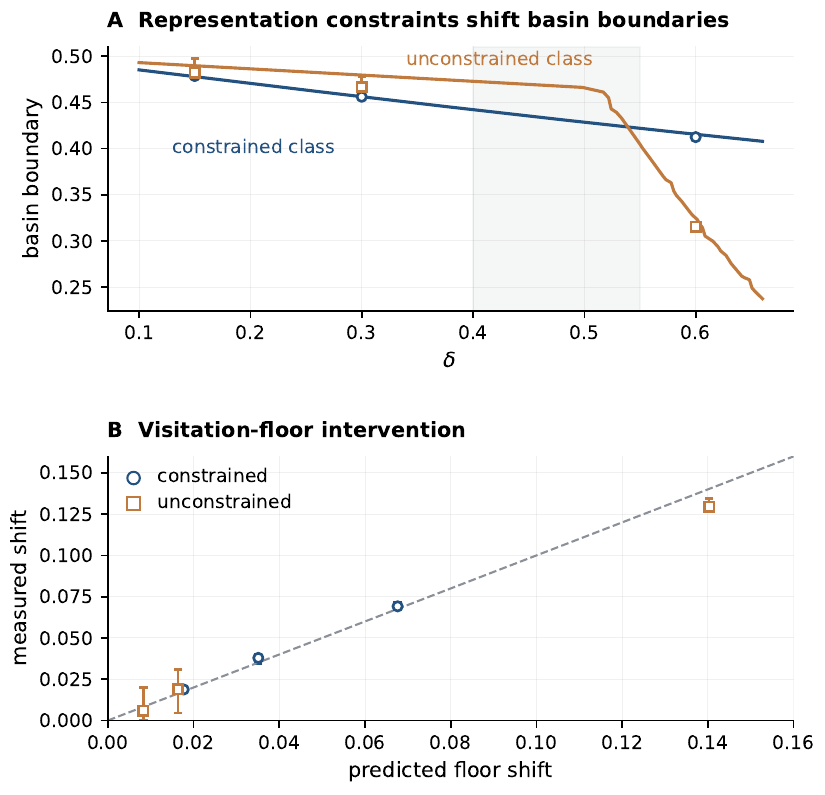}
  \caption{Fresh-seed cohort of the minimal learned-representation bridge.
  Representation constraints shift the basin boundary toward each class's
population prediction, with the capacity control at 0 of 11{,}700 trapped
runs (A). The visitation-floor intervention moves the boundary in the
predicted direction (B), where the dashed equality line is a secondary
reference and the registered gate tests direction.}
  \label{fig:e2-neural}
\end{figure}

E2 used a deterministic grid with one fixed representation initialization;
its multiple-initialization and fresh-solve comparisons are development
checks. The bridge and final PPO cohorts reserved fresh seeds after their
gate definitions, tolerances and analysis code were fixed. Rank partners
within each cohort deliberately reuse seed identities for paired comparisons.
The final PPO executable was also deployed from the frozen commit to a
directory keyed by its hash.

The bridge reporting rule marks a target added after execution as a
development check. A boundary interval with more than five percent non-finite
bootstrap resamples is unresolved and cannot pass its registered gate.
\Cref{tab:e2-gates,tab:bridge-gates,tab:e3-gates} report the registered
verdicts; the tables do not promote a later explanatory analysis to a pass
of the original criterion.

\begin{table}[t]
\centering
\caption{Registered gates for the solved-class loop E2. All separator and
threshold errors are absolute deviations from the computed values of
\cref{sec:dynamics,sec:interventions}.}
\label{tab:e2-gates}
\small
\setlength{\tabcolsep}{5pt}
\renewcommand{\arraystretch}{1.12}
\begin{tabularx}{\textwidth}{@{}l >{\raggedright\arraybackslash}X l l l@{}}
\toprule
Gate & Registered estimand & Bound & Measured & Verdict \\
\midrule
G1 & trapped runs, three-dimensional control & $0$ & $0$ of $779$ & pass \\
G2 & basin agreement with $\operatorname{sign}(p_0-p_\delta)$ & $\ge0.99$ & $746$ of $746$ & pass \\
G3 & separator error over $19$ preferences, $\kappa=80$ & $\le0.02$ & $0.0044$ & pass \\
G4 & critical temperature inside the computed bracket & $3$ of $3$ & $3$ of $3$ & pass \\
G5 & floor-mapped separator error, four points & $\le0.02$ & $<0.00063$ & pass \\
G6 & fixed-start escape point against $\lambda_{\rm crit}$ & $\le0.1$ & $0.169$, $0.101$ & fail \\
G7 & replay threshold error, two replay anchors & $\le0.02$ & $0.0122$, $0.0046$ & pass \\
\bottomrule
\end{tabularx}
\end{table}

The E2 grid has 19 task preferences and 41 initial values of \(p\), for 779
cells per arm. G2 excludes the 33 cells satisfying
$|p_0-p_\delta|\le0.02$, leaving 746 eligible cells. The frozen G2 line
wrote $\kappa=80$ while the grid specification and executed code used
$\kappa=20$; the reported verdict uses that executed grid. G3 separately
uses the $\kappa=80$ separator estimates. G4 also requires reference-sink
error at $\tau=0.8$ below $0.01$, with measured error at numerical
precision.

G6 contains two subchecks. Its fixed-start escape thresholds differ from
$\lambda_{\rm crit}$ by $0.1690$ and $0.1007$ at
$\delta=0.3$ and $0.6$, exceeding the $0.1$ tolerance. The endpoint
subcheck at penalty $\lambda=1.5$ also fails for the low-start
$\delta=0.6$ run. It measures $-0.6111$ against the low-endpoint target
$+0.6111$, an absolute discrepancy of $1.2222$, consistent with the
trajectory reaching the opposite endpoint. The other endpoint comparisons
agree at numerical precision. The basin-width follow-up in
\cref{fig:e2-ortho-appendix} is development evidence and does not replace
either failure. Its final measured widths are $0.0052$ and $0.0046$;
linear extrapolation from the last three penalty values gives zero-width
points $2.824$ and $1.526$, compared with theoretical thresholds $2.798$
and $1.523$.

\begin{table}[t]
\centering
\caption{Registered gates for the sampled bridge, with 35,100 records
from 50 fresh seeds. Per-side entries read trapped/escaped and pool the
two-dimensional arms. Boundary intervals resample seeds jointly across
matched arms and floors and use the 2.5th and 97.5th percentiles. The return-gate resampling defect is described
below.}
\label{tab:bridge-gates}
\small
\setlength{\tabcolsep}{5pt}
\renewcommand{\arraystretch}{1.12}
\begin{tabularx}{\textwidth}{@{}l >{\raggedright\arraybackslash}X >{\raggedright\arraybackslash}p{0.23\textwidth} l@{}}
\toprule
Gate & Registered estimand and bound & Measured & Verdict \\
\midrule
B1 & full-rank control trap rate $\le0.005$ & $0$ of $11{,}700$ & pass \\
B2 & constrained boundary error CI within $\pm0.02$, six cells & $6$ of $6$, widest $\pm0.007$ & pass \\
B3 & unconstrained boundary error CI within $\pm0.03$, six cells & $6$ of $6$, all within $\pm0.023$ & pass \\
B4 & floor shift CI lower bound $>0$, six pairs & $6$ of $6$ & pass \\
B5 & collision-pair order $\ge0.95$ per side & $1.000$/$0.999$ & pass \\
B6 & forced-gap sign $\ge0.95$ per side & $0.992$/$1.000$ & pass \\
B7 & balanced-return difference CI upper bound $<0$ & constrained $0$ of $3$ & fail \\
\bottomrule
\end{tabularx}
\end{table}

The bridge crosses three representation arms, three task preferences,
two floors and 39 initial values of \(p\), with the same 50 seed identities
used throughout these paired comparisons. Each arm therefore contributes
11,700 records, and all 35,100 records are finite. The boundary is the
interpolated 50\% trapped-rate crossing on the grid of initial \(p\) values;
its 95\% percentile interval uses 2,000 seed-bootstrap resamples. B5 evaluates collision
ordering on 10,562 trapped and 12,838 escaped two-dimensional runs. Its
counts are 10,562 and 12,820, respectively. B6 uses the same denominators,
with 10,482 negative forced gaps among trapped runs and 12,838 positive
forced gaps among escaped runs.

B6 passes in its registered pooled form, but the trapped-run sign rate is
$201/280=0.718$ in the unconstrained stratum with $\delta=0.6$ and
$\epsilon=0.25$. The other eleven trapped-run strata have rates at least
$0.99$. The failed stratum also contains the anomalous return comparison.
A recorded, untested explanation concerns the steep dependence on \(p\) of
the unconstrained population reference. That reference was optimized
numerically and passed the recorded solver-stability checks, but carries
no global optimality certificate.

B7 compares trapped and escaped codes under a balanced reference action distribution,
pooling floor levels within each class and task preference. In the
constrained class, the two ideal endpoint codes have equal balanced
population return by symmetry, so a strictly negative difference is not
the theory's prediction. All three constrained comparisons fail. The
unconstrained nominal pass is also unsuitable for confirmatory attribution
because floor levels were pooled and records were resampled as independent
units despite repeated seed identities. The on-policy return comparison
reported in the main text is a development analysis. A valid registered
return comparison remains uncompleted.

\begin{table}[t]
\centering
\caption{Registered gates for the final PPO cohort, with 100 fresh seeds
per main arm. Lower bounds are the lower endpoints of exact two-sided
95\% binomial intervals. Return contrasts use 95\% seed-bootstrap
intervals. The two endpoint-conditioned entries in F1, F5 and F6 concern
96 trapped low-start runs and 100 escaped high-start runs.}
\label{tab:e3-gates}
\small
\setlength{\tabcolsep}{5pt}
\renewcommand{\arraystretch}{1.12}
\begin{tabularx}{\textwidth}{@{}l >{\raggedright\arraybackslash}X >{\raggedright\arraybackslash}p{0.23\textwidth} l@{}}
\toprule
Gate & Frozen estimand and bound & Measured & Verdict \\
\midrule
F1 & occupied-branch error $\le1.10\times$ achievable rank-two reference; lower bound $>0.8$ & $96$ of $96$, $100$ of $100$ & pass \\
F2 & trap rate given low start, lower bound $>0.8$ & $0.96$, lower bound $0.901$ & pass \\
F3 & escape rate given high start, lower bound $>0.8$ & $1.00$, lower bound $0.964$ & pass \\
F4 & rank-64 trapped runs, at most $1$ of $100$ per arm & $0$ and $0$ & pass \\
F5 & cross-talk order on both sides, lower bound $>0.5$ & $96$ of $96$, $100$ of $100$ & pass \\
F6 & forced-gap sign on both sides, lower bound $>0.5$ & $96$ of $96$, $100$ of $100$ & pass \\
F7 & three return contrasts, CI lower bound $>0$ & $0.225$, $0.131$, $0.133$ & pass \\
tail & final-window policy drift at most $0.1$ in every arm & all arms & pass \\
\bottomrule
\end{tabularx}
\end{table}

The neural actor uses a nonlinear width-64 encoder with three noisy
target inputs and five nuisance inputs. A fixed rank-2 or rank-64
projector feeds an independently parameterized
controller; trainable parameter counts match across ranks.
The final PPO cohort has 400 main runs, from two projector ranks and two
initialization conditions with 100 seeds each. A further 30 low-start
rank-two runs use the full root-head learning rate as a secondary
comparator, for 430 records in total. The primary configuration has
$\delta=0.15$, $\epsilon=0.05$, $\sigma=0.30$, 400 updates
(819,200 two-step episodes per run) and
root-head learning-rate multiplier $0.3$. The unmodified branch
probabilities initialize at $q_0=0.05$ or $0.9$ before application of the
visitation floor. A trapped endpoint has final probability $p(E)$ below $0.5$;
the threshold classifies outcomes and is not an estimated dynamical
separator. The tail check bounds maximum-minus-minimum $p(E)$ over the
last quarter of updates by $0.1$ in every run, including the comparator.

The rank-two competence reference is the explicit linear construction
with error
\[
\frac{\sigma^2}{1+\sigma^2}
+\frac{2\sigma^2}{1+2\sigma^2}=0.235111\ldots
\qquad (\sigma=0.30).
\]
One coordinate stores the always-active feature of the evaluated branch;
the other stores the sum of its two mutually exclusive features. The
construction supplies a feasible comparison and does not certify global
optimality over the learned nonlinear class. F1 uses this reference on
branch $S$ for trapped low-start runs and branch $E$ for escaped
high-start runs. The four escaped low-start runs remain in the reported
arm summaries and return comparisons, but are not in these conditional
F1, F5 or F6 denominators.

The Jacobian cross-talk statistic in F5 averages
$(|\partial\widehat Z_i/\partial x_j|
+|\partial\widehat Z_j/\partial x_i|)/2$ over 4,096 fresh observations
with equally weighted branches, where $x_i$ is sensory feature input $i$.
Competence and forced-gap evaluation use 20,000 fresh observations per
branch. Deployment return uses 20,000 fresh episodes with the learned
branch probabilities and the mean terminal action.

F7 tests three contrasts. The first is the rank-64-minus-rank-two return
difference under low initialization, $0.225\,[0.219,0.230]$. The second
is the rank-two high-minus-low difference,
$0.131\,[0.125,0.135]$. The third subtracts the high-initialization rank
difference from the low-initialization rank difference,
giving $0.133\,[0.127,0.137]$. The high-initialization rank difference
is $0.092$, so rank restriction has a cost under both starts. A causal
interpretation of the additional low-start cost requires ordinary
capacity effects to be comparable across starts; the contrast does not
uniquely identify a cross-talk path.

Low-start seeds are 6000--6099 and high-start seeds are 6100--6199. The
rank partners within each initialization share six bound random streams
for initialization, data, observation noise, action sampling, optimization
and evaluation. The bootstrap preserves these rank pairings and resamples
the independent low- and high-start seed groups separately, using 2,000
resamples. The secondary comparator reuses low-start seeds 6000--6029.

Two pilots preceded the final cohort and failed their registered gates.
The first used a mismatched full-information competence reference for the
rank-two arm. The second passed six of seven gates, but its trap count of
28 out of 30 failed the exact confidence-bound criterion. Both verdicts
are retained. The final cohort then ran once on fresh seeds, with planned
power $0.965$ at the pilot-observed trap rate. Its later success does not
change either pilot verdict.

\section{Diagnostics across Action Distributions}
\label{app:cross-occupancy-diagnostic}

The diagnostic tests whether a change in training exposure alters
representation interference and thereby changes deployed feedback. A new
application must specify its feature reference, intervention, measurement
support, tolerances and experimental unit before the chain can be interpreted.
For a confirmatory application, these choices and the handling of missing
measurements are fixed before outcome access. The completed experiments
above test their stated subsets of this procedure.

\begin{enumerate}[leftmargin=*]
\item Compare controlled exposures. Choose at least two declared action distributions and
  hold singleton marginals, compute, evaluator support, and all other factors
  fixed. Record the complete intended denominator. If
  this isolation or support fails, return \emph{measurement invalid}.

\item Validate the representation measurement. Before attributing a change
  in measured geometry to competent feature coding, require the registered
  competence, effect-norm, additivity, and
  partner-context checks. A failed validity gate returns \emph{measurement
  invalid}, not absent superposition. Retain undefined cells as missing under
  the registered rule rather than changing the denominator.

\item Measure the forward effects. Test the predeclared contrast between action distributions
  in collision allocation, then support-matched joint injury or cross-talk,
  then the deployed gap after controller-access checks and fresh forced rollouts.
  Return the earliest unsupported link as the \emph{first failed rung}, and a later effect leaves that rung failed.

\item Test feedback and a matched intervention. Starting from the registered fresh
  actor state, test signed early motion, finite-budget persistence, and the
  declared perturbation restoration. Then apply the matched intervention in capacity, routing,
  the reference action distribution, or replay. Return \emph{chain supported}
  only when the mediator and downstream gap, policy, and return move in the
  registered order. Reward improvement without mediator movement is not
  diagnostic.
\end{enumerate}

An actor floor can restore visitation while retaining collision, so its
effect on the action distribution alone does not establish a repair of the representation.
The output \emph{chain supported} identifies the measured sequence in the
specified architecture and training budget. Exact global representation
response, asymptotic stability, a global basin, spontaneous discovery of
the relevant features and prevalence each require additional evidence.
For the registered studies reported here, the infrastructure-recovery rule
permits continuation from a checkpoint of the same job while preserving all
scientific state. It excludes reinitialization, replacement and
outcome-driven relaunches.

\section{Sampled Co-activation Studies}
\label{app:matching-pilot}
\label{app:matched-sampled}

\subsection{The four-feature matching task}

This exploratory study tests whether changing co-activation alone can
change the learned allocation and the ensuing policy updates. It uses the
four-feature, two-dimensional instance of
\cref{thm:matching-complement}. The independent replication unit is a
dataset and training-stream seed. We use four such seeds, numbered
0--3, with four random initializations nested within each seed and paired
across two training exposures, giving 32 fits and 32 continuations.
Every initialization is retained.

\paragraph{Task and matched samples.}
Let the latent features be independent standard Gaussians. Branch \(S\)
draws its active pair uniformly from \(\{13,14,23,24\}\), and branch
\(E\) from \(\{12,14,23,34\}\). Each feature therefore has activation
probability \(1/2\) on either branch. For a unit-column code
\(V\in\mathbb R^{2\times4}\), the encoder maps
\(X=\mathbf1_G\odot Z\) to \(h\), and the tied controller maps \(h\) to \(a\):
\[
 h=VX+\nu,\qquad a=V^\top h,\qquad
 \nu\sim\mathcal N(0,0.3^2I_2).
\]
The noise is in the latent channel; \(0.3\) is its standard deviation.
Reward is \(b_B-\sum_{i\in G}(a_i-Z_i)^2\), with
\(b_S=0\) and \(b_E=0.5\).

For each seed, each fixed exposure \(r\in\{0.4,0.6\}\) contains 4,000
control observations. At \(r=0.4\), each of the four \(S\) pairs occurs
600 times and each \(E\) pair 400 times; the counts exchange at
\(r=0.6\). Every feature occurs exactly 2,000 times in either dataset.
The two conditions also use the same pool of 2,000 active Gaussian
amplitudes for each feature, independently permuted into their active
slots, and share the channel-noise array. Thus both feature counts and
empirical active-amplitude distributions match exactly, while joint
co-activation changes. This is a stratified paired design, rather than
independent sampling of every active mask.

\paragraph{Fitting and fresh control evaluation.}
Each initialization normalizes the columns of an independent Gaussian
matrix and is shared across the two exposures. Fitting takes 2,000
projected SGD updates of size \(0.03\), differentiating the sampled
reconstruction loss through both uses of \(V\), projecting each gradient
onto the unit-column tangent spaces, and renormalizing after each update.
Each update selects one of ten fixed 400-observation batches with
replacement. Every batch has the same eight-stratum proportions as its
dataset; paired conditions use the same batch index.

The population optimum is \(0.58\) at either exposure. Across all 32
fits, the largest evaluated population excess is
\(4.444\times10^{-4}\), and the largest absolute error in a squared
Gram entry relative to the predicted matching is \(0.002418\).
These comparisons concern the known population objective; they do not
certify optimization of each noisy empirical objective. Analytic
objectives and Gram predictions are used for evaluation only.

Before further adaptation, each fitted code is held fixed for fresh
forced-branch evaluation. Each branch receives 40,000 control observations
in 20 independently generated blocks of 2,000, with its four pairs equally
represented within each block. Evaluation samples are shared across
the paired fits within a dataset seed and are separate from fitting data.
The resulting payoff-gap means are \(-0.495803\) and \(1.495990\) for
\(r=0.4\) and \(0.6\). Across the four seed means, the corresponding
ranges are \([-0.501893,-0.487099]\) and
\([1.487152,1.506790]\). These are descriptive ranges, not confidence
intervals. The Monte Carlo blocks quantify conditional evaluation noise;
they are not additional training seeds.

\paragraph{Sampled continuation from a common policy.}
Every fitted code begins a 2,000-update continuation at actual
\(p_E=0.5\). At each update, 100 branch draws follow the current policy;
each draw supplies one fresh observation from each of that branch's four
pairs. This gives 400 stratified control observations for one encoder
update of size \(0.03\). A separate batch with the same design then
measures rewards using the updated code and the same current policy.
Thus an encoder batch and an actor batch each have 100 branch draws,
not 400 independent branch choices. Their random streams are separate,
and the two exposure conditions use paired random numbers.

The actor uses branchwise payoff estimates with exponential averaging
coefficient \(0.9\), initialized by the first observed branch mean and
retained when a branch is absent. Its update is
\[
 u_{t+1}=\operatorname{clip}_{[-40,40]}
 \bigl(u_t+0.02(\widehat q_E-\widehat q_S)\bigr),\qquad
 p_E=0.05+0.9\operatorname{sigmoid}(u).
\]
The drive is zero until both branches have been observed. This is a
floored sampled-payoff replicator, rather than PPO or an ordinary logit
policy-gradient update. Neither encoder nor actor receives an analytic
optimizer or population payoff. Scalar policy and gap traces are saved
at every update, with complete code diagnostics every 20 updates.

All 16 continuations from \(r=0.4\) finish near \(p_E=0.05\), and all
16 from \(r=0.6\) near \(0.95\). The final returns average
\(-0.205028\) and \(0.244966\). Here return is the exact environmental
expectation of the saved terminal code and its actual stochastic policy,
\[
 J(V,p)=0.5p-(1-p)D_S(V)-pD_E(V),
\]
not a fresh rollout estimate or a fully adapted pure-policy benchmark.
For the ideal matching codes at these floored endpoints, the returns are
\(-0.205\) and \(0.245\), separated by \(0.45\); the separately
adapted pure-policy gap is \(0.5\).

\begin{table}[t]
\centering
\small
\setlength{\tabcolsep}{4pt}
\caption{Matching study by dataset seed and exposure. Gap, policy and
return entries average the four nested initializations; the last column
is the largest squared-Gram error among those four fits. The Monte Carlo
gap is measured before continuation, and the analytic gap uses the same
fitted codes. No initialization or evaluation block is treated as an
independent dataset seed.}
\label{tab:matching-pilot-seeds}
\begin{tabular}{@{}rr rrr rr@{}}
\toprule
Seed & \(r\) & MC gap & Analytic gap & Final \(p_E\) & Final \(J\) & Max Gram error \\
\midrule
0 & 0.4 & -0.492603 & -0.499766 & 0.05 & -0.205030 & 0.000787 \\
0 & 0.6 & 1.487152 & 1.499311 & 0.95 & 0.244970 & 0.001044 \\
1 & 0.4 & -0.487099 & -0.498816 & 0.05 & -0.205024 & 0.002418 \\
1 & 0.6 & 1.492229 & 1.498647 & 0.95 & 0.244970 & 0.001726 \\
2 & 0.4 & -0.501893 & -0.499512 & 0.05 & -0.205016 & 0.000907 \\
2 & 0.6 & 1.497788 & 1.498831 & 0.95 & 0.244974 & 0.001812 \\
3 & 0.4 & -0.501618 & -0.498992 & 0.05 & -0.205044 & 0.001384 \\
3 & 0.6 & 1.506790 & 1.499331 & 0.95 & 0.244949 & 0.001521 \\
\bottomrule
\end{tabular}
\end{table}

\paragraph{Equal loss at the symmetric exposure.}
A separate assay directly constructs
\(C_S=(e_1,e_1,e_2,e_2)\) and
\(C_E=(e_1,e_2,e_1,e_2)\), where \(e_1,e_2\) are orthonormal.
At \(r=0.5\), both are globally optimal with loss \(0.68\).
Their branch distortions are respectively \((0.18,1.18)\) and
\((1.18,0.18)\), giving payoff gaps \(-0.5\) and \(1.5\).
Using the same 20-block forced-control design, measured balanced losses
are \(0.679465\) and \(0.679633\), and gaps are \(-0.499199\) and
\(1.497135\). An orthogonal four-dimensional code gives exact
distortions \((0.18,0.18)\) and gap \(0.5\), with measured gap
\(0.497293\). This assay concerns explicitly constructed codes, not
learned optima, and demonstrates different actionwise loss allocations
under the same aggregate objective. It does not assert that the midpoint
is a stable trap. The learned continuations establish the observed
finite-rate feedback in this controlled class; four development seeds do
not establish its prevalence in neural reinforcement learning.

\subsection{Learning across dimensions}
\label{app:dimension-learning}

The four settings in \cref{fig:matching-sampled}b instantiate the
equal-frequency partition construction of
\cref{app:dimension-partitions} with
$(d_{\mathrm{in}},d_{\mathrm{out}})\in\{(16,2),(16,4),(32,2),(32,4)\}$.
For each setting, partition the features into $d_{\mathrm{out}}$ groups
of size $m=d_{\mathrm{in}}/d_{\mathrm{out}}$. Branch $S$ draws uniformly
from the $N=d_{\mathrm{in}}(d_{\mathrm{in}}-m)/2$ pairs joining distinct
groups. Branch $E$ uses the partition obtained by exchanging one feature
between the first two groups. Every feature has probability
$2/d_{\mathrm{in}}$ on either branch. With unit encoder columns, tied
controller weights, and independent Gaussian channel noise of standard
deviation $0.3$, each branch has minimum population loss $0.18$.
We set $b_S=0$ and $b_E=\delta=2(m-1)/N$, half the excess loss that
an $S$-optimal code incurs on $E$. These choices give strict return
reversal at the $S$ optimum and a better separately adapted $E$ policy.
Because the branch laws and reward gap change with dimension, this
comparison tests the feedback across constructions rather than a
monotonic effect of representation size.

\paragraph{Matched fitting data.}
Each setting uses 20 dataset seeds with two code initializations per
seed, paired between pure-$S$ and pure-$E$ fitting. These are separate
branch datasets, rather than the $r=0.4,0.6$ mixtures in the four-feature
experiment. Pair counts exactly follow each branch law. The observation
count is the smallest multiple of $N$ giving at least 4,096 active
observations per feature. Each feature uses the same Gaussian value pool
in the two datasets, independently permuted into its active slots, and
the channel-noise arrays are shared. Thus feature counts and empirical
active-value distributions match exactly. Initial encoders are Gaussian
with normalized columns. We take 5,000 projected Adam updates on the
empirical active-coordinate squared error, with learning rate decaying
from $0.03$ to $0.0015$ on a cosine schedule. Sufficient statistics
accelerate this finite-sample objective, with loss and gradient checked
against raw observations. Population optima are evaluation references
and do not enter training. All 320 fitted codes are retained.

\paragraph{Joint code and policy learning.}
Every fitted code begins a 4,096-update continuation from $p(E)=0.5$.
Each update fits the code on 512 fresh policy-collected observations
using projected SGD with learning rate $0.03$. The current controller
then receives 1,024 separate observations from each branch to estimate
its returns. These forced-branch samples update only the policy, so both
returns remain measurable even when a branch is rarely selected.
The estimates use exponential averaging with coefficient $0.95$, and
the policy logit follows
\[
 u_{t+1}=\operatorname{clip}_{[-12,12]}
 \bigl(u_t+0.2(\widehat q_E-\widehat q_S)\bigr),
 \qquad p(E)=\operatorname{sigmoid}(u).
\]
This sampled-payoff update does not use analytic returns. Its numerical
limits are approximately $6.14\times10^{-6}$ and $1-6.14\times10^{-6}$.
Codes, probabilities, and expected returns are recorded every 32
updates. Separate local continuations start from $0.05$ for $S$-fitted
codes and $0.95$ for $E$-fitted codes, giving 640 continuations in total.
All runs are retained. In \cref{fig:matching-sampled}b, the first 1,024
updates show the transition from the common initial policy. Curves
average the two initializations within each dataset seed and then the
20 seeds, with pointwise 95\% percentile intervals from 5,000 bootstrap
resamples of seed means.
The full 4,096-update common-start curves appear in
\cref{fig:dimension-learning-full}.

\paragraph{Fitting, return reversal, and continuation.}
\Cref{tab:dimension-learning} reports the largest initial population
fitting excess, mean initial branch-return gaps, and final return
advantage of the $E$-fitted group. All 40 $S$-fitted codes per setting
initially favor $S$, and all 40 $E$-fitted codes favor $E$. Both groups
retain their respective return-gap signs at every saved continuation
checkpoint and approach opposite pure policies under both start
protocols. The $E$-fitted group earns more at the same capacity in every
setting. Initial and final branch returns are also checked with 16 fresh
evaluation blocks of 4,096 observations per branch. The table uses
analytic expected returns of the learned codes and actual final
policies, not returns obtained by exact refitting during continuation.
These finite trajectories test persistence during learning and do not
estimate the theoretical basin boundaries.

\begin{table}[t]
\centering
\small
\setlength{\tabcolsep}{5pt}
\caption{Nearly optimal fits yield opposite return rankings across four
dimension settings. Final return differences compare the common-start
continuations after 4,096 updates.}
\label{tab:dimension-learning}
\begin{tabular}{@{}lrrrr@{}}
\toprule
\shortstack{Features $\to$\\dimensions} &
\shortstack{Maximum\\fitting excess} &
\shortstack{Initial $E-S$\\$S$-fitted} &
\shortstack{Initial $E-S$\\$E$-fitted} &
\shortstack{Final return\\$E$-fitted $-$ $S$-fitted} \\
\midrule
$16\to2$ & $9.25\times10^{-5}$ & $-0.21873$ & $0.65623$ & $0.21875$ \\
$16\to4$ & $2.30\times10^{-4}$ & $-0.06246$ & $0.18746$ & $0.06250$ \\
$32\to2$ & $7.50\times10^{-5}$ & $-0.11718$ & $0.35155$ & $0.11719$ \\
$32\to4$ & $1.98\times10^{-4}$ & $-0.03643$ & $0.10935$ & $0.03646$ \\
\bottomrule
\end{tabular}
\end{table}

\begin{figure}[t]
\centering
\includegraphics[width=\textwidth]{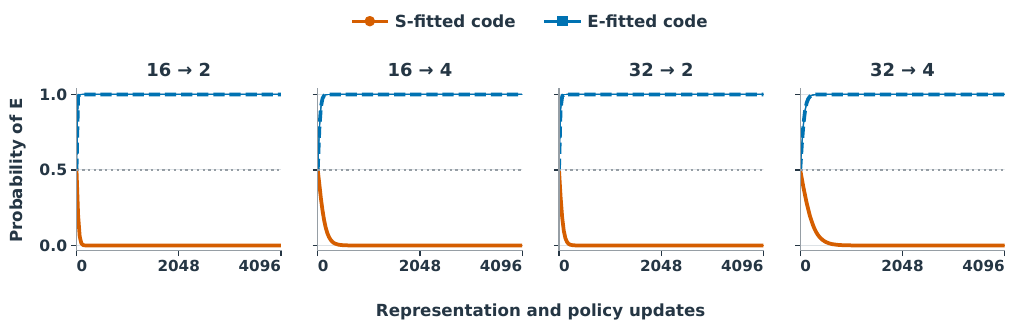}
\caption{The branch preferences in \cref{fig:matching-sampled}b persist
throughout continued learning. Curves and bands summarize all 40
continuations per fitting condition in each setting.}
\label{fig:dimension-learning-full}
\end{figure}

\section{Sequential Control with Learned Representations}
\label{app:sequential-neural}

This experiment studies capacity and training exposure beyond the solved
two-step construction. It retains known sensory factors so that control
errors can be measured, while introducing repeated routing decisions,
stochastic transitions, and reward-only neural PPO.

\subsection{Task and learner}
\label{app:sequential-task}

The environment has stages \(s\in\{0,1,2,3\}\), two exits, and an
intrinsic budget of 32 phases. The first stage \(s=0\) is called Hub.
A routing phase selects \(S\) or \(E\);
the following control phase supplies three real actions. On successful
control, \(S\) terminates with bonus \(1+0.5s\), while \(E\) advances
one stage, terminating with bonus 3 at stage 3. Failure returns to routing
at the current stage. Thus four stages need not mean four decisions or
four steps. Every phase costs \(0.02\). A control phase also costs its
realized error
\[
L=\sum_{i\in G}(a_i-Z_i)^2,
\qquad \Pr(\text{success}\mid L)=\exp(-L/0.5).
\]
The success event is sampled from this probability. Budget exhaustion is
terminal, with zero value bootstrap; remaining budget is observed.

The independent latent coordinates are standard Gaussian. Branch \(S\)
draws \(G\) uniformly from \(\{13,23\}\), and \(E\) from
\(\{12,13\}\). The eight sensory inputs contain
\(X_i=Z_i\mathbf 1\{i\in G\}+\epsilon_i\), with independent
\(\epsilon_i\sim\mathcal N(0,0.3^2)\), and five independent Gaussian
nuisance coordinates. The actor receives neither the active mask nor
latent targets. Stage, phase, branch and normalized remaining budget form
nine context coordinates; branch context is zero during routing.

An \(8\to64\) tanh encoder is followed by a fixed orthogonal rank-\(k\)
projector and gain \(\sqrt{64/k}\). The projected representation and context feed
two 64-unit tanh layers, a categorical route head, and a three-dimensional
Gaussian control head. The sensory representation is zero during routing. Ranks 2 and
64 have the same trainable parameter count and matched initial actor
parameters; the fixed projector restricts the learned sensory map, not
the number of weights. Route probabilities are
\(0.05+0.9\operatorname{softmax}(\ell)_E\). A learned diagonal control
standard deviation starts at \(e^{-0.5}\), with log standard deviations
clipped to \([-5,1]\). The independent critic has two 64-unit tanh layers
and receives observable sensory and context inputs. Its inputs contain
no latent targets, and its gradients do not update the actor encoder.

Training uses 128 environments, 128 phases per rollout, four PPO epochs,
minibatches of 1,024, Adam learning rate \(3\times10^{-4}\) and
\(\epsilon_{\rm Adam}=10^{-5}\), discount \(0.99\), GAE parameter
\(0.95\), clipping ratio \(0.2\), entropy coefficient \(0.001\), and
global gradient-norm clipping at \(0.5\). The value term is
\(0.5\times\tfrac12(V-\widehat R)^2\).
Advantages are normalized over the rollout. Likelihoods and entropies use
only the action head active in each phase. Environment rewards and
transitions provide no differentiable path or supervised target to the
learner.

\subsection{Ordinary training and matched forks}
\label{app:sequential-design}

\begin{figure}[t]
\centering
\includegraphics[width=\textwidth]{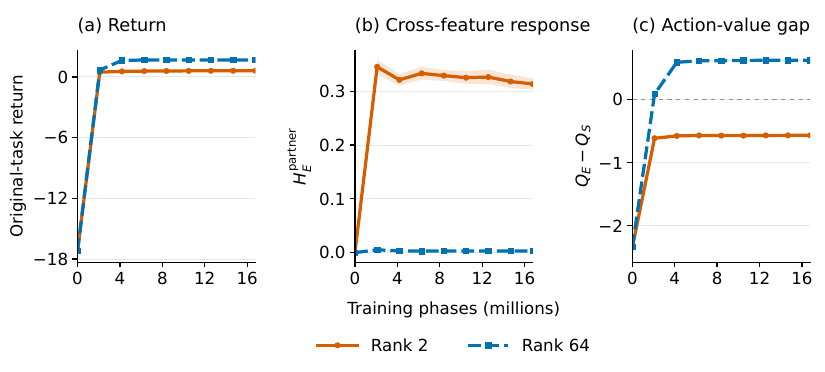}
\caption{Complete trajectories for \cref{fig:sequential-baseline},
including initialization. Both ranks use \(p_0=0.1\) and 20 training
seeds. Means and pointwise 95\% seed-bootstrap intervals use the
same seeds and saved checkpoints throughout. Initial returns are
strongly negative before either controller learns the task.}
\label{fig:sequential-full-training}
\end{figure}

Ordinary training compares both ranks at actual initial advance
probabilities \(p_0=0.1\), \(0.5\), and \(0.9\), with 20 seeds per condition
and 16,777,216 phases per run. \Cref{tab:sequential-baseline} gives all
six results. All 60 narrow estimated root rollout gaps are negative,
whereas all 60 wide gaps are positive. No tested initial probability
produces a second narrow outcome class. This neural configuration
therefore does not reproduce the initialization-selected bistability of
the solved model.

\begin{table}[t]
\centering\small
\caption{Ordinary training from Hub, 20 seeds per cell and 16,777,216
phases per run. Entries are stochastic original-Hub return with
pointwise 95\% seed-bootstrap intervals. This cohort has a different
training history and budget from \cref{tab:sequential-results}.}
\label{tab:sequential-baseline}
\begin{tabular}{@{}rll@{}}
\toprule
Initial $p_E$ & Rank 2 & Rank 64 \\
\midrule
0.1 & $0.606\ [0.600,0.612]$ & $1.673\ [1.666,1.679]$ \\
0.5 & $0.605\ [0.600,0.610]$ & $1.676\ [1.669,1.684]$ \\
0.9 & $0.606\ [0.600,0.612]$ & $1.675\ [1.666,1.683]$ \\
\bottomrule
\end{tabular}
\end{table}

The fork experiment uses separate parents trained for 4,194,304 phases
at \(p_0=0.9\). Half the simulator lanes train from Hub. The other half
start valid control episodes at stage 3, branch \(E\), with one phase
remaining and masks 12/13 equally likely. This supplies balanced
advance-branch exposure through the existing reward. Half the lanes means
half the transitions, approximately two-thirds of control samples, since
ordinary trajectories alternate routing and control. Both ranks receive
the same pretraining protocol and budget, although their on-policy data
and competence can differ.

Each narrow parent supplies four forks: uniform-stage versus Hub starts,
crossed with trainable versus frozen encoder. Each wide parent supplies
the two trainable-encoder forks. A uniform reset samples all four routing
stages equally with a fresh 32-phase budget; it never forces a route
action. Reset changes affect future episodes. The first 4,194,304 fork
phases use the assigned starts, and the next 4,194,304 use Hub in every
arm. Forks inherit model, optimizer, simulator and random-generator
states. Encoder freezing leaves the controller, route policy and critic
trainable. Checkpoint comparisons verify that all 40 frozen encoders
remain unchanged.

\subsection{Measurements and paired statistics}
\label{app:sequential-measurements}

Each policy is evaluated on 2,048 fresh first episodes from Hub, retaining
both route and Gaussian action sampling. Reported task return is
undiscounted. To estimate the root route-value gap, two further rollout
evaluations force only the first route to \(E\) or \(S\), then use the
current stochastic continuation policy. Their discounted-return difference
uses \(\gamma=0.99\); it is not a critic estimate.

Control probes use 2,048 samples per valid mask. The root context has 31
remaining phases. Let \(\mu_j(X)\) be the mean control action in coordinate
\(j\) produced from sensory input \(X\), and let \(G\) be equally weighted
over masks 12 and 13. Two different
functional statistics are retained:
\begin{align*}
H_E^{\rm partner}
&=\mathbb E_G\sum_{j\in G}\tfrac12\mathbb E\bigl[
 \bigl(\mu_j(X)-\mu_j(X_{-i},X_i')\bigr)^2\mid G\bigr],\\
H_E^{\rm joint}
&=\mathbb E_G\sum_{j\in G}\tfrac12\mathbb E\bigl[
 \bigl(\mu_j(X)-\mu_j(X_j,X_{-j}')\bigr)^2\mid G\bigr].
\end{align*}
Here \(\{i\}=G\setminus\{j\}\), and primes denote independent draws from
the conditional-mask sensory law.
The baseline presentation uses active-partner harm; the fork analysis
uses joint non-target harm, resampling all seven other inputs. These
statistics differ from the earlier two-step PPO Jacobian cross-talk.
Root $E$ MSE averages
\(\mathbb E[\sum_{j\in G}(\mu_j(X)-Z_j)^2\mid G]\) equally over the two
masks. It sums the active-target errors and excludes Gaussian action
variance. Success is evaluated separately as the sample mean of
\(\exp(-L/0.5)\).

\paragraph{Why the interference statistic measures control error.}
Fix the root context and define
$\bar\mu_{j,G}(X_j)=\mathbb E[\mu_j(X)\mid X_j,G]$.
In this simulator, $(Z_j,X_j)$ is independent of $X_{-j}$ conditional
on $G$. Averaging out the other inputs therefore removes a squared-error
component without discarding information about the target:
\begin{equation}
\mathbb E_G\sum_{j\in G}\mathbb E[(\mu_j(X)-Z_j)^2\mid G]
=
\mathbb E_G\sum_{j\in G}\mathbb E[(\bar\mu_{j,G}(X_j)-Z_j)^2\mid G]
+H_E^{\rm joint}.
\label{eq:sequential-interference-decomposition}
\end{equation}
To see this, the original and resampled outputs are independent and
identically distributed conditional on $(X_j,G)$, giving
$\tfrac12\mathbb E[(\mu_j(X)-\mu_j(X_j,X_{-j}'))^2\mid G]
=\mathbb E[(\mu_j(X)-\bar\mu_{j,G}(X_j))^2\mid G]$.
Expanding $\mu_j(X)-Z_j$ around $\bar\mu_{j,G}(X_j)$ leaves a cross term
whose conditional expectation is zero by the stated independence.
Summing over active coordinates and averaging over masks proves the
identity. Thus $H_E^{\rm joint}$ measures harmful dependence on inputs
that provide no additional target information under this diagnostic law.
The equality holds for population expectations, and the paired
resampling statistic is a Monte Carlo estimate of its final term.

The conditional averages above use the evaluator's mask and need not be
jointly realizable by the same constrained encoder and controller. They justify the
measurement rather than supply a deployable repair. The identity also
does not locate the interference in the encoder or establish how much
of an episode-return gain it explains. We measure task return through
separate rollouts because the learned policy also changes subsequent states
and decisions.

Formal seeds are 2--21; development seed 1 is kept separate. All 120
ordinary runs, 40 parents and 120 forks completed. Each paired analysis
uses 20 complete seed blocks: four narrow arms for the encoder comparison,
or four trainable rank/start arms for the capacity comparison. The analyses
share narrow runs. Pointwise percentile intervals use 10,000 bootstrap
resamples of seeds, preserving the relevant pairings; episodes and
checkpoints are not independent replicates. Numerical comparisons align at
total phase counts 4,194,304, 8,388,608 and 12,582,912, including Hub
controls whose curriculum duration is zero. The learning curves in
\cref{fig:sequential-intervention-time} include all nine saved evaluations,
spaced 1,048,576 phases apart from the parent onward. Each point reports
the mean and pointwise 95\% interval over the same 20 paired seeds,
evaluated from the original Hub. Unrecorded optional short-horizon and
deep-transfer probes are unavailable rather than zero. Configurations,
source hashes, checkpoints and every declared outcome are retained.

\subsection{What the interventions resolve, and what remains open}
\label{app:sequential-limits}

\begin{figure}[t]
\centering
\includegraphics[width=\textwidth]{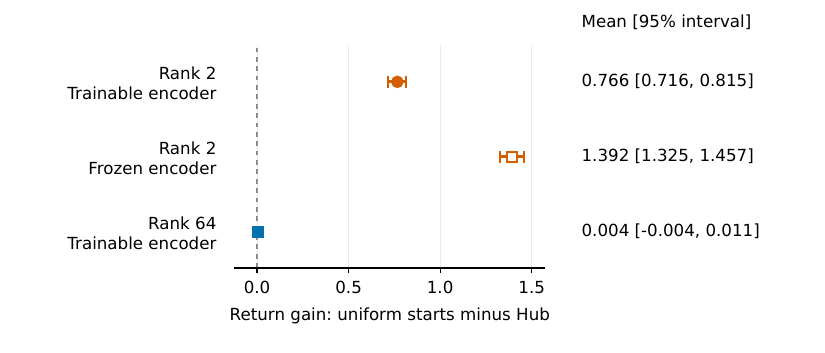}
\caption{Paired final return differences for \cref{tab:sequential-results}.
Points and bars show means and pointwise 95\% bootstrap intervals over
20 complete training-seed blocks.}
\label{fig:sequential-paired-effects}
\end{figure}

The paired uniform-minus-Hub return gain is positive in all 20 narrow
seeds, including with frozen encoders. Its trainable-minus-frozen
interaction is \(-0.626\,[-0.665,-0.585]\), while its rank-two-minus-rank-64
interaction is \(0.762\,[0.712,0.812]\). Thus a positive curriculum
effect does not require continued encoder learning. Initially useful
narrow control also deteriorates under Hub-only adaptation: root E MSE
rises from \(0.323\) to \(0.610\), against \(0.219\) after uniform
starts. Corresponding joint non-target harm is \(0.057\) at the parent,
\(0.259\) after Hub starts, and \(0.039\) after uniform starts. Its paired
uniform-minus-Hub difference is \(-0.220\,[-0.237,-0.202]\).
All 20 paired trainable-encoder seeds have lower final root $E$ MSE
and joint non-target harm after uniform starts, alongside higher Hub
return. This contrast chiefly captures prevented degradation. With the
encoder frozen, root $E$ MSE is $0.283$ after Hub starts and $0.246$ after
uniform starts, improving in all 20 pairs. Joint non-target harm changes
from $0.070$ to $0.062$ and decreases in 16 of 20 pairs, although return
improves in all 20. The return benefit therefore does not require every
seed to improve this particular root diagnostic.

Wide parents already return \(1.655\) before the fork, against
\(-0.333\) for narrow parents. Their small curriculum gain therefore
also reflects less remaining improvement opportunity; matched training
histories do not match competence. Successful narrow forks begin from
these balanced-pretrained parents, not the mature failed baseline.
\Cref{tab:sequential-development} preserves the unsuccessful late
interventions. These findings establish exposure-dependent neural control
in this constructed task, without an optimality certificate or empirical
verification of every condition in \cref{def:trap}.

\begin{table}[H]
\centering\small
\caption{Retained sequential development failures, all using seed 1.
Late interventions start from the 8,388,608-phase ordinary-training pilot
checkpoint. Fixed-mask forks receive 4,194,304 additional phases;
balanced-E forks receive 8,388,608. The fresh pilot receives 4,194,304
pretraining and 4,194,304 Hub-only phases. These results do not establish
impossibility.}
\label{tab:sequential-development}
\begin{tabularx}{\textwidth}{@{}l >{\raggedright\arraybackslash}X >{\raggedright\arraybackslash}X@{}}
\toprule
Study & Observed result & Interpretation \\
\midrule
Late fixed-mask exposure & Five forks/control runs return $0.615$--$0.628$; no observed high exits. & Mask12/13 exposure with trainable/frozen encoders did not recover this mature narrow checkpoint. \\
Late balanced-E exposure & Five forks/control runs return $0.612$--$0.636$; no observed high exits. & Balanced masks likewise did not produce late recovery in the tested budget. \\
Fresh balanced-E pilot & Hub follow-up returns $0.445$, with root E MSE $0.620$ and no observed high exits. & Useful auxiliary learning alone did not maintain downstream behavior. \\
\bottomrule
\end{tabularx}
\end{table}

In the replicated narrow Hub forks, each encoder
condition has one observed high exit among 40,960 evaluation episodes;
these rates are near zero, not exactly zero. Exit frequency alone also
does not rank policies: successful narrow forks exit high more often
than wide policies while earning less return.

\clearpage

\section{Frozen-Encoder Controller Probes}
\label{app:decoder-probe}

The tied construction makes the effect of feature interference explicit, but
its tied weights also limit how the controller can use the
code. We therefore hold the encoder fixed and ask whether an independently
fitted controller preserves the value reversal, and whether its fitting exposure
changes the answer. This is a supervised recovery diagnostic. It isolates
controller adaptation before asking what a learning policy would subsequently do.

\paragraph{Information and fitting distributions.}
The task has the original three independent signed standard-Gaussian
amplitudes. Branch $S$ samples $\{1,3\}$ or $\{2,3\}$ equally; branch $E$
samples $\{1,2\}$ or $\{1,3\}$ equally. A frozen unit-column encoder
$V\in\mathbb R^{d\times3}$ produces
\[
h=V(Z\odot m_G)+\sigma\xi,\qquad \xi\sim\mathcal N(0,I_d),
\]
where $m_G$ is the active-pair indicator. Each controller receives only $h$.
Neither the branch nor the active mask is provided as an input. The mask and
signed targets are available for the supervised diagnostic loss
$\sum_{i\in G}(f_i(h)-Z_i)^2$. Thus the noise is added after encoding,
as in the core task; the reward-only neural studies below instead add noise
to sensory inputs before encoding.

There are six encoder/original-policy cases. The exact $S$ code has
$v_1=v_2=e_1$, $v_3=e_2$, with original probability $p(E)=0.1$;
the exact $E$ code has $v_1=e_1$, $v_2=v_3=e_2$, with $p(E)=0.9$.
The fitted $S$ and $E$ groups import the first twenty retained E1
best-of-eight codes trained under these respective action distributions
(\cref{app:sampled-check}). These retain their dataset identities and are
not selected using the controller results. The two remaining cases use
$V=I_3$, paired with original probabilities $0.1$ and $0.9$.
For each case we use $\sigma\in\{0.1,0.3\}$ and fit at either the original
probability or $r=0.5$, holding $V$ fixed. Fitting exposure $r$ gives pair
probabilities $(r/2,1/2,(1-r)/2)$ on $(12,13,23)$. At $r=0.5$, balanced
branches therefore give neither uniform pairs nor equal singleton frequencies.

\paragraph{Controllers, references, and optimization.}
We compare the unchanged tied controller $V^\top h$, an independently
parameterized linear controller with no intercept, and a conventional MLP
with two width-64 ReLU hidden layers and an unrestricted linear output.
The MLP controller and signed Gaussian regression output follow the
standard construction of \citet[Sections~14.3--14.4]{goodfellow2016deep};
width 64 is our experimental choice. The linear controller uses ordinary
least squares separately for each output on samples where that output is
active. The tied reference has no fitted controller parameters; the linear
and MLP controllers have $6$ and $4{,}547$ parameters for $d=2$, and $9$ and
$4{,}611$ for $d=3$. Controller parameter counts are therefore not matched.

Each fitting distribution supplies $65{,}536$ observations to both OLS
and the MLP. The MLP uses Adam for $8{,}192$ minibatch updates of size
$2{,}048$, at learning rate $10^{-3}$, reduced to $10^{-4}$ for the final
quarter. A separate $16{,}384$-sample validation set is evaluated every
512 updates. We report the final controller, without selecting a checkpoint
by validation or test performance. Both fitting exposures share underlying
random draws and initialization within a seed block. The twenty new
probe/test seeds yield $6\times2\times20=240$ completed blocks and 480
MLP fits. The earlier seed-zero technical pilot is excluded. For analytic
codes the replication unit is a controller-fitting and test seed conditional on the
same encoder; in fitted-code groups it also includes the corresponding
retained encoder dataset.

Population linear regression and an $h$-only Bayes estimator locate the
remaining fitting error. For pair $G$, write
$\Sigma_G=V_GV_G^\top+\sigma^2I_d$ and let $\phi_{\Sigma_G}$ be its
Gaussian density. Under fitting prior $\pi_G(r)$, the Bayes estimator for
active-output loss is
\begin{equation}
 f_{i,r}^{\star}(h)=
 \frac{\sum_{G\ni i}\pi_G(r)\phi_{\Sigma_G}(h)
           v_i^\top\Sigma_G^{-1}h}
      {\sum_{G\ni i}\pi_G(r)\phi_{\Sigma_G}(h)}.
 \label{eq:decoder-probe-bayes}
\end{equation}
It minimizes risk under the fitting prior, not each test branch separately.
It uses no observed mask at inference. A further pair-aware reference is
retained as a lower bound with extra information, and is not a same-information
baseline or a training target.

\paragraph{Values and uncertainty.}
Fresh tests use $65{,}536$ observations per forced branch, with common
amplitudes and channel noise across branches, controllers, and fitting exposures
within a block. We measure the summed active-coordinate errors $D_S,D_E$,
the actual value gap $\Delta q=\delta-D_E+D_S$, and
\[
 J_{p}=p\delta-(1-p)D_S-pD_E,
\]
where $p$ remains the original policy probability even after balanced fitting.
These values use the controller output as the terminal control. Branch
contrasts use paired per-observation losses; exposure contrasts resample
whole matched seed blocks. Pointwise 95\% percentile intervals use 10,000
seed-bootstrap resamples. The reward margins $\delta=0.25,0.5,0.75$ are
postprocessed from the same errors, not separately trained cohorts.
\Cref{tab:decoder-probe-full} reports every code, noise level, and controller
at $\delta=0.5$; adding $\delta-0.5$ to either gap and
$p(\delta-0.5)$ to either return gives the other declared margins.

\paragraph{Exposure changes control errors and action order.}
For the exact $S$ code, an independent controller reduces the negative gap
substantially, but the MLP still gives negative gaps under the original
exposure in all twenty seeds at both noise levels. Balanced fitting moves
its mean gap from $-0.08118$ to $0.00978$ at $\sigma=0.1$, and from
$-0.09599$ to $0.05187$ at $\sigma=0.3$. The paired changes are
$0.09096\,[0.08949,0.09232]$ and
$0.14787\,[0.14545,0.15027]$, respectively. Balanced MLP gap estimates
are positive in $19/20$ and $20/20$ seeds. OLS retains a negative balanced
gap at $\sigma=0.1$ but makes it positive at $\sigma=0.3$.
The imported $S$ codes give similar MLP changes,
$-0.08206\to0.00958$ and $-0.09547\to0.05165$.
Those encoders already have nearly the same collinear geometry, so they
do not constitute an additional encoder architecture.

Reversing the gap does not immediately improve the unchanged policy.
For the exact $S$ code with $p=0.1$, balanced MLP fitting changes return
from $-0.03785$ to $-0.04838$ at $\sigma=0.1$, and from $-0.18341$
to $-0.20406$ at $\sigma=0.3$. It reallocates error toward the branch
that this policy still visits most. The $E$ codes retain positive gaps,
and both identity controls remain near the base gap $0.5$.
Across all 480 MLP fits, the largest validation excess over the
fitting-prior Bayes estimator is $0.001667$, and the largest absolute
validation change over the final 512 updates is $0.000541$.
These diagnostics support a well-fitted controller comparison without
certifying global MLP optimality. The study shows that fitting exposure
can change action order through an unchanged encoder; it does not test
actor escape, encoder adaptation, or closed-loop attraction. At the
smaller margin $\delta=0.25$, the low-code MLP mean gaps remain negative
even after balanced fitting, whereas at $\delta=0.75$ they are already
positive under original fitting. Recovery therefore depends on both the
controller and the task's value margin.

\begingroup
\small
\setlength{\tabcolsep}{3pt}
\renewcommand{\arraystretch}{1.06}
\begin{longtable}{@{}lllr r r r@{}}
\caption{Independent controllers and fitting exposure change action values
through frozen codes. Means use twenty seed blocks at $\delta=0.5$;
O/B denote original/balanced fitting, with returns at the original policy.}
\label{tab:decoder-probe-full}\\
\toprule
Code & $\sigma$ & Controller & $\Delta q_{\rm O}$ & $\Delta q_{\rm B}$ & $J_{\rm O}$ & $J_{\rm B}$ \\
\midrule
\endfirsthead
\multicolumn{7}{l}{\tablename\ \thetable\ (continued)}\\
\toprule
Code & $\sigma$ & Controller & $\Delta q_{\rm O}$ & $\Delta q_{\rm B}$ & $J_{\rm O}$ & $J_{\rm B}$ \\
\midrule
\endhead
\midrule
\multicolumn{7}{r}{Continued on next page}\\
\endfoot
\bottomrule
\endlastfoot
Exact $S$ & 0.1 & Tied & $-0.4990$ & $-0.4990$ & $-0.0699$ & $-0.0699$ \\*
 &  & OLS & $-0.3216$ & $-0.0261$ & $-0.0593$ & $-0.1070$ \\*
 &  & Pop. linear & $-0.3207$ & $-0.0259$ & $-0.0593$ & $-0.1071$ \\*
 &  & MLP & $-0.0812$ & $0.0098$ & $-0.0378$ & $-0.0484$ \\*
 &  & Bayes & $-0.0797$ & $0.0106$ & $-0.0374$ & $-0.0481$ \\
\addlinespace[2pt]
Exact $S$ & 0.3 & Tied & $-0.4992$ & $-0.4992$ & $-0.2300$ & $-0.2300$ \\*
 &  & OLS & $-0.2140$ & $0.0313$ & $-0.1925$ & $-0.2326$ \\*
 &  & Pop. linear & $-0.2134$ & $0.0312$ & $-0.1925$ & $-0.2325$ \\*
 &  & MLP & $-0.0960$ & $0.0519$ & $-0.1834$ & $-0.2041$ \\*
 &  & Bayes & $-0.0950$ & $0.0520$ & $-0.1828$ & $-0.2033$ \\
\addlinespace[2pt]
Exact $E$ & 0.1 & Tied & $1.4996$ & $1.4996$ & $0.3300$ & $0.3300$ \\*
 &  & OLS & $1.3203$ & $1.0259$ & $0.3407$ & $0.2926$ \\*
 &  & Pop. linear & $1.3214$ & $1.0267$ & $0.3407$ & $0.2929$ \\*
 &  & MLP & $1.0834$ & $0.9917$ & $0.3621$ & $0.3518$ \\*
 &  & Bayes & $1.0811$ & $0.9905$ & $0.3626$ & $0.3520$ \\
\addlinespace[2pt]
Exact $E$ & 0.3 & Tied & $1.4999$ & $1.4999$ & $0.1700$ & $0.1700$ \\*
 &  & OLS & $1.2133$ & $0.9690$ & $0.2075$ & $0.1673$ \\*
 &  & Pop. linear & $1.2143$ & $0.9697$ & $0.2075$ & $0.1676$ \\*
 &  & MLP & $1.0976$ & $0.9502$ & $0.2166$ & $0.1963$ \\*
 &  & Bayes & $1.0968$ & $0.9493$ & $0.2172$ & $0.1968$ \\
\addlinespace[2pt]
Fitted $S$ & 0.1 & Tied & $-0.4990$ & $-0.4990$ & $-0.0699$ & $-0.0699$ \\*
 &  & OLS & $-0.3218$ & $-0.0263$ & $-0.0593$ & $-0.1069$ \\*
 &  & Pop. linear & $-0.3207$ & $-0.0259$ & $-0.0593$ & $-0.1070$ \\*
 &  & MLP & $-0.0821$ & $0.0096$ & $-0.0378$ & $-0.0482$ \\*
 &  & Bayes & $-0.0799$ & $0.0105$ & $-0.0374$ & $-0.0480$ \\
\addlinespace[2pt]
Fitted $S$ & 0.3 & Tied & $-0.4991$ & $-0.4991$ & $-0.2300$ & $-0.2300$ \\*
 &  & OLS & $-0.2144$ & $0.0310$ & $-0.1924$ & $-0.2324$ \\*
 &  & Pop. linear & $-0.2135$ & $0.0311$ & $-0.1924$ & $-0.2324$ \\*
 &  & MLP & $-0.0955$ & $0.0517$ & $-0.1834$ & $-0.2041$ \\*
 &  & Bayes & $-0.0954$ & $0.0518$ & $-0.1827$ & $-0.2034$ \\
\addlinespace[2pt]
Fitted $E$ & 0.1 & Tied & $1.4996$ & $1.4996$ & $0.3300$ & $0.3300$ \\*
 &  & OLS & $1.3203$ & $1.0256$ & $0.3407$ & $0.2924$ \\*
 &  & Pop. linear & $1.3215$ & $1.0268$ & $0.3407$ & $0.2929$ \\*
 &  & MLP & $1.0842$ & $0.9919$ & $0.3621$ & $0.3518$ \\*
 &  & Bayes & $1.0813$ & $0.9907$ & $0.3625$ & $0.3520$ \\
\addlinespace[2pt]
Fitted $E$ & 0.3 & Tied & $1.5000$ & $1.5000$ & $0.1700$ & $0.1700$ \\*
 &  & OLS & $1.2131$ & $0.9689$ & $0.2075$ & $0.1672$ \\*
 &  & Pop. linear & $1.2144$ & $0.9698$ & $0.2075$ & $0.1676$ \\*
 &  & MLP & $1.0981$ & $0.9508$ & $0.2166$ & $0.1966$ \\*
 &  & Bayes & $1.0971$ & $0.9496$ & $0.2173$ & $0.1968$ \\
\addlinespace[2pt]
$I_3$, $p=.1$ & 0.1 & Tied & $0.5000$ & $0.5000$ & $0.0300$ & $0.0300$ \\*
 &  & OLS & $0.5000$ & $0.5000$ & $0.0302$ & $0.0302$ \\*
 &  & Pop. linear & $0.5000$ & $0.5000$ & $0.0302$ & $0.0302$ \\*
 &  & MLP & $0.4999$ & $0.5000$ & $0.0301$ & $0.0301$ \\*
 &  & Bayes & $0.5000$ & $0.5000$ & $0.0302$ & $0.0302$ \\
\addlinespace[2pt]
$I_3$, $p=.1$ & 0.3 & Tied & $0.5000$ & $0.5000$ & $-0.1300$ & $-0.1300$ \\*
 &  & OLS & $0.5001$ & $0.5001$ & $-0.1151$ & $-0.1151$ \\*
 &  & Pop. linear & $0.5001$ & $0.5001$ & $-0.1151$ & $-0.1151$ \\*
 &  & MLP & $0.4998$ & $0.5001$ & $-0.1157$ & $-0.1157$ \\*
 &  & Bayes & $0.5001$ & $0.5001$ & $-0.1151$ & $-0.1151$ \\
\addlinespace[2pt]
$I_3$, $p=.9$ & 0.1 & Tied & $0.5000$ & $0.5000$ & $0.4300$ & $0.4300$ \\*
 &  & OLS & $0.5000$ & $0.5000$ & $0.4302$ & $0.4302$ \\*
 &  & Pop. linear & $0.5000$ & $0.5000$ & $0.4302$ & $0.4302$ \\*
 &  & MLP & $0.5001$ & $0.5000$ & $0.4301$ & $0.4301$ \\*
 &  & Bayes & $0.5000$ & $0.5000$ & $0.4302$ & $0.4302$ \\
\addlinespace[2pt]
$I_3$, $p=.9$ & 0.3 & Tied & $0.5000$ & $0.5000$ & $0.2700$ & $0.2700$ \\*
 &  & OLS & $0.5001$ & $0.5001$ & $0.2850$ & $0.2850$ \\*
 &  & Pop. linear & $0.5001$ & $0.5001$ & $0.2850$ & $0.2850$ \\*
 &  & MLP & $0.5004$ & $0.5001$ & $0.2844$ & $0.2844$ \\*
 &  & Bayes & $0.5001$ & $0.5001$ & $0.2850$ & $0.2850$ \\
\end{longtable}
\endgroup

\section{Shared High-Dimensional Representations Learned by PPO}
\label{app:highdim-learning}

The controller probes keep the representation fixed and reveal what an
independent controller can recover. We next ask what happens when the encoder,
controller, and root policy must learn together from rewards as the number of
relevant features increases. This experiment preserves controlled
co-activation but uses a shared nonlinear sensory encoder, an independently
parameterized controller, and PPO \citep{schulman2017ppo}. Its comparisons
measure capacity-dependent learning and feedback in this specified neural
system; the solved linear phase diagram is not assumed to describe its
optima or dynamics.

\paragraph{Task and information available to the learner.}
For even $n$, let $M_S$ and $M_E$ be disjoint perfect matchings on
$\{1,\ldots,n\}$ whose union is one alternating cycle. Branch $B$
selects an unordered pair uniformly from the complement of $M_B$,
which contains $n(n-2)/2$ pairs. Thus every feature has activation
probability $2/n$ on either branch. Changing the policy changes pair
coverage while leaving singleton activation probabilities equal in
expectation; finite training counts are sampled, not forced equal.

Each episode has two decision steps: choose $B\in\{S,E\}$, then output
$a\in\mathbb R^n$. The terminal reward is
\[
 R=\delta_n\mathbf1\{B=E\}-\sum_{i\in G}(a_i-Z_i)^2,
 \qquad \delta_n=\frac{1}{n-2}.
\]
There is no earlier reward and $\gamma=1$. The sensory input consists
of $n$ masked standard-Gaussian amplitudes with independent Gaussian
noise of standard deviation $0.3$ on all feature coordinates, together
with $n/2$ independent standard-Gaussian nuisance coordinates.
The noise is added before the learned encoder. Only sensory input and
branch identity reach the terminal actor or critic; targets and active
masks determine the environmental reward and diagnostic counts.
No reconstruction loss, latent supervision, analytic value, or Bayes
reference enters the PPO update. The joint Gaussian likelihood includes
all $n$ action coordinates, although only the two active coordinates
are scored by the environment.

\paragraph{Capacity and controller comparison.}
A dense affine map followed by tanh maps all sensory inputs to 128 units.
A fixed orthogonal projector $P_k$ restricts this representation to rank $k$,
with gain $\sqrt{128/k}$. The projector bases are nested within a seed;
there is no sensory bypass to the actor. A branch one-hot vector is
appended after projection. The independent controller is either affine
or a two-hidden-layer width-64 ReLU MLP with a signed linear output.
The affine controller can use branch-specific biases, whereas the MLP can
also condition nonlinearly on branch context. An independent two-layer
width-64 tanh critic receives the original sensory input and branch
context, so its value loss does not use or update the actor's bottleneck.
The root policy has two learned logits and actual probability
$p(E)=0.05+0.9\operatorname{softmax}(u)_E$.

At fixed $n$ and controller, parameter counts and dense tensor dimensions
are the same across ranks. Total trainable counts for linear/MLP policies
are $11{,}268/22{,}756$ at $n=16$, $17{,}988/28{,}420$ at $n=32$,
and $31{,}428/39{,}748$ at $n=64$. The comparison changes rank without
changing these counts; it does not match parameter counts across controller
families. The rank-dependent gain also changes forward and gradient
scales, so the experiment does not isolate capacity from every
optimization consequence of the architecture.

The completed grid uses $(n,k)\in\{(16,8),(16,16),(32,8),(32,16),
(32,32),(64,32),(64,64)\}$, both controllers, and actual initial
$p(E)\in\{0.1,0.9\}$. All 28 conditions retain five training seeds
(1--5), for 140 completed runs. The seed-zero development cohort and
seed-900 technical check are kept separate. Initial encoder and critic
weights and random-stream identities are shared within each training
seed; the rank projectors are nested. Policy-dependent visitation is
allowed to change the realized training data.

\paragraph{Training and measurement.}
Every run uses 2,000 PPO updates with 4,096 two-step episodes per update:
8,192,000 episodes or 16,384,000 environment steps. PPO uses four epochs,
minibatches of 1,024, Adam at $3\times10^{-4}$ with a root-logit learning
rate multiplier of $0.3$, clipping $0.2$, GAE parameter $0.95$, value-loss
coefficient $0.5$, and gradient-norm cap $0.5$. Both entropy coefficients
are zero. Gaussian log standard deviations start at $-0.5$ and are
learned, with effective values clamped to $[-5,1]$.

Evaluations occur initially and every 100 updates, giving 21 observations
per run. Separate, common evaluation streams use 32,768 samples for each
forced branch and 32,768 deployment episodes. Let $D_B$ be summed
active-coordinate MSE with the terminal mean action. The mean-control
value gap is $\Delta q=\delta_n-D_E+D_S$, measured from paired forced
rollouts, not from the critic. Deployment reports both mean terminal
control and sampled Gaussian control, with stochastic routing in both;
\cref{tab:highdim-learning-full} uses the latter return,
$J_{\rm sampled}$. Evaluation Monte Carlo standard errors are conditional
on a trained policy. The table's standard deviations instead describe
variation across the five training seeds. Paired capacity, controller,
and initialization comparisons retain seed identity; their exploratory
percentile intervals enumerate the $5^5$ ordered seed-bootstrap resamples
and are neither simultaneous nor corrected for multiple comparisons.

The full-sensory Bayes active-error reference is
$2\sigma^2/(1+\sigma^2)=0.16514$, with mean prediction
$x_i/(1+\sigma^2)$ on a scored coordinate. It diagnoses competent
sensory control rather than asserting attainability at every rank.
Additional probes use 8,192 observations per branch to independently
resample an active sensory coordinate while retaining its partner and
branch. The partner response is
$\mathbb E[(\mu_j(x')-\mu_j(x))^2]/[2(1+\sigma^2)]$;
for a linear map it equals a squared cross coefficient. Own-coordinate
signed gains and MSE accompany this statistic, because a controller
that drops features can also have small cross-responses. These
conditional nonlinear sensitivities are not an additive decomposition
of control error. Branch, singleton, and pair exposure counts, all
scheduled evaluations, and final checkpoints are retained for every run.

\paragraph{Capacity improves return; reversal occurs in a subset of runs.}
Within every fixed task, each adjacent increase in rank improves sampled
return in all five paired seeds, for both controllers and both initial
probabilities. At $n=32$ and low initial visitation, the linear controller's
mean returns at ranks $8,16,32$ are $-0.6912,-0.3894,-0.1831$;
the corresponding MLP values are $-0.7099,-0.4212,-0.2590$.
The MLP has lower sampled return than the linear controller in all fourteen
group-mean comparisons and 69 of 70 paired-seed comparisons. This is
a finite-budget learning result for the stated PPO configurations,
not evidence that the MLP function class cannot express a good controller.

Only four final gap point estimates are negative, all at
$n=16,k=8,p_0=0.1$. The linear seed-3 estimate,
$-0.000728\pm0.002588$ (one evaluation MC standard error), leaves its
sign unresolved and its visitation is still increasing. Linear seed 4
and MLP seeds 4 and 5 have gaps
$-0.027368\pm0.003079$, $-0.023313\pm0.002999$, and
$-0.017819\pm0.002852$, respectively, with decreasing visitation over
the final 500 updates. These three trajectories are consistent with
adverse feedback over the observed interval. None is an asymptotic
stability or neural-global-optimality certificate. No final gap is
negative at $n=32$ or $64$.

The adequate-rank controls explain why low endpoint visitation alone
would be misleading. At $n=k=64$, low-start final $p(E)$ averages
$0.3630$ with the linear controller and $0.1771$ with the MLP, but both
mean gaps are positive ($0.01528$ and $0.01205$).
The MLP's low-start $D_S,D_E$ average $0.38315,0.38723$, still well
above the sensory Bayes reference. Its $D_E$ improves over the last
500 and 1,000 updates in all ten rank-64 MLP runs across both starts.
The measured learning process is therefore still changing; these
endpoints should not be classified as stable traps. Finally,
$\delta_n$ and individual-pair coverage change with $n$. The results
support within-task capacity comparisons, not a dimensional scaling
law, and the grid ends at $k=n$ without testing excess capacity $k>n$.

\begingroup
\small
\setlength{\tabcolsep}{3pt}
\renewcommand{\arraystretch}{1.10}
\begin{longtable}{@{}rrlrrrrrr@{}}
\caption{Capacity improves return, while negative-gap estimates occur only
at low-start $n=16,k=8$. All 28 conditions report five-seed means;
gap and sampled return include sample SD.}\label{tab:highdim-learning-full}\\
\toprule
$n$ & $k$ & Controller & $p_0$ & Final $p(E)$ & $D_S$ & $D_E$ & $\Delta q$ & $J_{\rm sampled}$ \\
\midrule
\endfirsthead
\multicolumn{9}{l}{\tablename\ \thetable\ (continued)}\\
\toprule
$n$ & $k$ & Controller & $p_0$ & Final $p(E)$ & $D_S$ & $D_E$ & $\Delta q$ & $J_{\rm sampled}$ \\
\midrule
\endhead
\midrule
\multicolumn{9}{r}{Continued on next page}\\
\endfoot
\bottomrule
\endlastfoot
16 & 8 & Linear & $0.1$ & $0.363$ & $0.345$ & $0.414$ & $0.0028\pm0.0204$ & $-0.3579\pm0.0119$ \\*
16 & 8 & Linear & $0.9$ & $0.949$ & $0.416$ & $0.321$ & $0.1666\pm0.0084$ & $-0.2723\pm0.0090$ \\*
16 & 8 & MLP & $0.1$ & $0.500$ & $0.363$ & $0.422$ & $0.0124\pm0.0330$ & $-0.3675\pm0.0120$ \\*
16 & 8 & MLP & $0.9$ & $0.949$ & $0.424$ & $0.347$ & $0.1484\pm0.0109$ & $-0.3008\pm0.0076$ \\
\addlinespace[3pt]
16 & 16 & Linear & $0.1$ & $0.883$ & $0.175$ & $0.174$ & $0.0726\pm0.0005$ & $-0.1206\pm0.0002$ \\*
16 & 16 & Linear & $0.9$ & $0.949$ & $0.176$ & $0.174$ & $0.0734\pm0.0014$ & $-0.1158\pm0.0003$ \\*
16 & 16 & MLP & $0.1$ & $0.872$ & $0.198$ & $0.196$ & $0.0728\pm0.0017$ & $-0.1477\pm0.0116$ \\*
16 & 16 & MLP & $0.9$ & $0.949$ & $0.194$ & $0.191$ & $0.0745\pm0.0003$ & $-0.1370\pm0.0019$ \\
\addlinespace[3pt]
32 & 8 & Linear & $0.1$ & $0.371$ & $0.656$ & $0.661$ & $0.0286\pm0.0039$ & $-0.6912\pm0.0041$ \\*
32 & 8 & Linear & $0.9$ & $0.942$ & $0.666$ & $0.657$ & $0.0429\pm0.0018$ & $-0.6731\pm0.0042$ \\*
32 & 8 & MLP & $0.1$ & $0.373$ & $0.674$ & $0.678$ & $0.0299\pm0.0025$ & $-0.7099\pm0.0029$ \\*
32 & 8 & MLP & $0.9$ & $0.941$ & $0.683$ & $0.677$ & $0.0396\pm0.0041$ & $-0.6922\pm0.0053$ \\
\addlinespace[3pt]
32 & 16 & Linear & $0.1$ & $0.637$ & $0.380$ & $0.383$ & $0.0303\pm0.0030$ & $-0.3894\pm0.0020$ \\*
32 & 16 & Linear & $0.9$ & $0.946$ & $0.388$ & $0.378$ & $0.0432\pm0.0028$ & $-0.3748\pm0.0027$ \\*
32 & 16 & MLP & $0.1$ & $0.600$ & $0.407$ & $0.409$ & $0.0314\pm0.0018$ & $-0.4212\pm0.0039$ \\*
32 & 16 & MLP & $0.9$ & $0.945$ & $0.410$ & $0.405$ & $0.0392\pm0.0026$ & $-0.4055\pm0.0036$ \\
\addlinespace[3pt]
32 & 32 & Linear & $0.1$ & $0.813$ & $0.196$ & $0.195$ & $0.0343\pm0.0006$ & $-0.1831\pm0.0018$ \\*
32 & 32 & Linear & $0.9$ & $0.948$ & $0.197$ & $0.195$ & $0.0361\pm0.0010$ & $-0.1781\pm0.0009$ \\*
32 & 32 & MLP & $0.1$ & $0.719$ & $0.255$ & $0.255$ & $0.0329\pm0.0007$ & $-0.2590\pm0.0039$ \\*
32 & 32 & MLP & $0.9$ & $0.946$ & $0.252$ & $0.249$ & $0.0362\pm0.0029$ & $-0.2446\pm0.0033$ \\
\addlinespace[3pt]
64 & 32 & Linear & $0.1$ & $0.227$ & $0.412$ & $0.414$ & $0.0141\pm0.0004$ & $-0.4687\pm0.0037$ \\*
64 & 32 & Linear & $0.9$ & $0.936$ & $0.418$ & $0.413$ & $0.0208\pm0.0006$ & $-0.4565\pm0.0031$ \\*
64 & 32 & MLP & $0.1$ & $0.167$ & $0.492$ & $0.495$ & $0.0137\pm0.0013$ & $-0.5758\pm0.0097$ \\*
64 & 32 & MLP & $0.9$ & $0.935$ & $0.501$ & $0.497$ & $0.0202\pm0.0012$ & $-0.5682\pm0.0075$ \\
\addlinespace[3pt]
64 & 64 & Linear & $0.1$ & $0.363$ & $0.254$ & $0.255$ & $0.0153\pm0.0007$ & $-0.2861\pm0.0038$ \\*
64 & 64 & Linear & $0.9$ & $0.941$ & $0.257$ & $0.254$ & $0.0193\pm0.0011$ & $-0.2770\pm0.0040$ \\*
64 & 64 & MLP & $0.1$ & $0.177$ & $0.383$ & $0.387$ & $0.0120\pm0.0012$ & $-0.4556\pm0.0049$ \\*
64 & 64 & MLP & $0.9$ & $0.937$ & $0.388$ & $0.383$ & $0.0216\pm0.0022$ & $-0.4418\pm0.0101$ \\
\end{longtable}
\endgroup

\section{Public DoorKey Study}
\label{app:public-doorkey}

\subsection{Capacity and Exploration}

We test whether increasing representation capacity consistently improves
learning on an existing sparse-reward task, MiniGrid DoorKey
\citep{chevalierboisvert2023minigrid}. A $5\times5$ calibration checks
learnability; the $8\times8$ study crosses four actor ranks with two
entropy coefficients. This comparison measures learning outcomes on a
public task without a certificate for optimal representations or a
superposition mechanism.

\paragraph{Task and comparisons.}
We use the official DoorKey environments with \texttt{FullyObsWrapper}
and \texttt{ImgObsWrapper}. The agent observes the complete symbolic grid;
the seven categorical actions, generated layouts, transitions and rewards
are unchanged. A successful episode returns $1-0.9t/T$, where $t$ is the
completion time and $T=250$ or $640$ for $5\times5$ or $8\times8$;
unsuccessful episodes return zero. Key pickup and door opening are
recorded events, with no auxiliary reward.

Both actor and critic use independent CNNs with three $2\times2$
convolutions of widths 16, 32 and 64, ReLU activations, and a
128-dimensional tanh feature layer. Only the actor applies a fixed orthogonal
rank-$k$ projector with gain $\sqrt{128/k}$, followed by two 64-unit tanh
layers and a categorical head. The projectors use a common nested basis;
$k=128$ gives the identity. We compare $k\in\{4,16,64,128\}$ on
$8\times8$, and $k\in\{16,128\}$ on $5\times5$. Trainable parameter
counts and initial weights are matched across ranks within a seed
(112,232 parameters on $5\times5$; 456,296 on $8\times8$). The critic
has no projection and shares no actor parameters. Symbolic images are
cast to floating point without division by 255.

On $8\times8$, PPO \citep{schulman2017ppo} uses entropy coefficient
$0$ or $0.01$ at each rank, giving eight conditions and 24 training runs.
The rank-4 and rank-64 runs extend the earlier rank-16/rank-128 comparison
with the same learning code, hyperparameters and budgets. All runs start
from fresh initialization. An auxiliary rank-16 comparison uses Proximal
Feature Optimization (PFO; \citealt{moalla2024representation}) with zero
entropy coefficient. PFO applies coefficient $1$ to the mean squared
displacement of the actor's final 64-dimensional MLP preactivation,
averaged over coordinates and rollout observations. A detached reference
actor is fixed before each rollout and retained throughout its
optimization epochs. This follows the official implementation's
coordinate mean; an unnormalized coordinate sum would multiply the
penalty by 64. The $5\times5$ calibration uses PPO with zero entropy
coefficient at both ranks.

\paragraph{Training and measurement.}
We use Stable-Baselines3 PPO \citep{raffin2021sb3}, version 2.9.0, with
16 environments, 128 transitions per rollout per environment, 10 epochs,
minibatches of 64, learning rate $3\times10^{-4}$, discount $0.99$, GAE
parameter $0.95$, clipping ratio $0.2$, value coefficient $0.5$, and
gradient-norm cap $0.5$. MiniGrid is version 3.1.0 and Gymnasium is 1.3.0.
The standard PPO branch is used directly when the PFO coefficient is
zero. Time-limit-only training transitions bootstrap from their terminal
observation; true terminal transitions do not. Evaluation accumulates
the original environment rewards.

All configurations use training seeds 0, 1 and 2. Each run receives
524,288 transitions on $5\times5$ or 2,097,152 on $8\times8$.
Final policies are evaluated stochastically on the same 200 declared
layout seeds and independently seeded action draws. Layout seeds begin
at 2,000,000 and action seeds at 7,000,003; different seeds need not
generate different layouts. Intermediate evaluations use 32 episodes
every 65,536 or 131,072 transitions, respectively. Tables use the final
checkpoint of every run, with no checkpoint selection. The unit of
replication is a training seed: we show all three success rates and
their arithmetic mean, without estimating confidence intervals from
three seeds. Mean return first averages episodes within a seed and then
averages seeds.

\begin{table}[t]
\centering\small
\caption{DoorKey $8\times8$: mean success with entropy regularization
is nonmonotone in actor rank. Final success (\%) and original-task return
for all three training seeds per condition; the last row is the auxiliary
PFO comparison.}
\label{tab:public-doorkey8}
\setlength{\tabcolsep}{5pt}
\begin{tabular}{@{}lr rrr rr@{}}
\toprule
 & & \multicolumn{3}{c}{Success by training seed (\%)} & Mean & Mean \\
Method & $k$ & 0 & 1 & 2 & success (\%) & return \\
\midrule
PPO & 4 & 0.0 & 0.0 & 0.0 & 0.0 & 0.000 \\
PPO + entropy & 4 & 0.0 & 3.0 & 0.5 & 1.2 & 0.004 \\
PPO & 16 & 0.5 & 0.0 & 0.5 & 0.3 & 0.002 \\
PPO + entropy & 16 & 39.0 & 26.0 & 12.0 & 25.7 & 0.161 \\
PPO & 64 & 0.0 & 0.0 & 0.0 & 0.0 & 0.000 \\
PPO + entropy & 64 & 16.5 & 0.0 & 27.5 & 14.7 & 0.091 \\
PPO & 128 & 0.0 & 1.0 & 0.5 & 0.5 & 0.001 \\
PPO + entropy & 128 & 23.5 & 6.0 & 3.5 & 11.0 & 0.064 \\
\midrule
PFO & 16 & 0.5 & 0.0 & 32.5 & 11.0 & 0.074 \\
\bottomrule
\end{tabular}
\end{table}

\begin{table}[t]
\centering\small
\caption{The smaller DoorKey task is learnable at rank 16.
Final success (\%) and return on $5\times5$, three training seeds per rank.}
\label{tab:public-doorkey5}
\begin{tabular}{@{}r rrr rr@{}}
\toprule
 & \multicolumn{3}{c}{Success by training seed (\%)} & Mean & Mean \\
$k$ & 0 & 1 & 2 & success (\%) & return \\
\midrule
16 & 100.0 & 100.0 & 100.0 & 100.0 & 0.937 \\
128 & 99.5 & 100.0 & 49.5 & 83.0 & 0.742 \\
\bottomrule
\end{tabular}
\end{table}

\paragraph{What these outcomes establish.}
The calibration demonstrates learnability at rank 16 on the smaller task
(\cref{tab:public-doorkey5}); it does not determine performance on
$8\times8$. Without entropy regularization, mean success on $8\times8$
is below $1\%$ at every rank. Entropy coefficient $0.01$ increases mean
success at all four capacities, but the capacity ordering is nonmonotone:
rank 4 has the lowest mean, while rank 16 exceeds ranks 64 and 128
(\cref{tab:public-doorkey8}). Variation across training seeds is large;
for example, rank-64 success ranges from $0\%$ to $27.5\%$, and the
best rank differs across seeds. PFO's mean improvement is driven by one
seed. These results show that the capacity ordering from the controlled
tasks does not transfer uniformly to this public task at the tested
training budget.

Recorded mean key-pickup/door-opening rates rise from
$32.3\%/2.0\%$ to $72.2\%/35.2\%$ for rank-16 PPO with entropy,
and from $32.7\%/3.8\%$ to $51.0\%/17.0\%$ for rank-128 PPO with
entropy. These events locate progress through the task but do not
identify representational interference. Three-seed mean differences,
with projection gain and optimization changing alongside rank, do not
establish an optimal capacity or a failure mechanism caused by excess
capacity. Nor does limited success within these budgets establish that a
policy class cannot solve DoorKey. The comparison therefore measures
capacity and exploration effects without certifying a superposition trap.

\subsection{Matched Exposure Interventions}
\label{app:public-interventions}

The continuation studies hold capacity fixed and change access to
training states. We use all six original $8\times8$ parents at ranks
16 and 128 with entropy coefficient $0.01$, three training seeds per
rank, saved at 2,097,152 transitions. Each parent first supplies four
arms: native versus stage-mixture exposure, crossed with trainable versus
frozen actor sensory encoder. A subsequent study adds factor-mixture
exposure with the same two encoder conditions, reusing the native arms
and parent evaluations. There are 24 stage-study continuations and 12
new factor continuations; shared controls are not new replications.
In \cref{fig:public-interventions}, points report changes in native-start
success relative to the shared native continuation. Lines connect the two
interventions for the same parent and encoder condition.

Every continuation inherits its original parent's weights and Adam
state and receives exactly 262,144 learner-controlled transitions, for
9,437,184 new transitions across the two studies. None starts from a
smoke test or a previous intervention endpoint. All use fresh paired
environment seeds and common initial action/minibatch RNG conventions,
with the PPO settings above. Encoder freezing covers the actor CNN and
its final 128-dimensional layer. The nonlinear actor MLP, action head
and independent critic remain trainable; the fixed projector does not
change. Thus freezing does not freeze every hidden representation or
hold the entire learner fixed. It also changes which gradients enter
global gradient clipping.

A stage-mixture reset is native with probability one half. Otherwise,
it uniformly selects the state immediately after pickup, unlocking or
door traversal on a shortest legal route from that freshly generated
native layout. The state retains its true prefix clock. A factor-mixture
reset is also native with probability one half; otherwise it uniformly
selects one of the four diagnostic configurations described below.
Legal pickup, unlock, deposit and possible re-closing actions reach
each configuration from the actual native initial state. No-op padding
sets all four clocks to 100. Prefixes are installed, not executed as
learner transitions. Both interventions preserve the original reward
and deadline, but supply privileged oracle access to later states.
Equal learner budgets therefore do not match information or planning
cost. Factor mixing was designed after observing the earlier study's
intervention--probe mismatch; it is an exploratory follow-up.

Final performance uses 200 native-start episodes per model, with layout
seeds starting at 4,000,000 and per-episode action streams at 11,000,003.
Actions use independent NumPy inverse-CDF categorical draws shared across
arms. The first 64 streams provide the midpoint evaluation at 131,072
additional transitions. The six original parents are reevaluated on
this same panel; these evaluations differ from the original capacity
sweep's panel and sampler. All reported continuation comparisons use
the final checkpoint at the fixed budget, with no intermediate selection.
The 36 endpoints and six parent evaluations contribute 10,704 native
evaluation episodes, including midpoints. The replication unit remains
the original training seed, not an episode or a reused endpoint.

\begin{table}[t]
\centering\small
\caption{DoorKey continuation outcomes at fixed capacity. Entries give
mean native-start success (\%) / original return over three parents.
Every arm has 262,144 additional learner transitions. Native and broad
stage controls are shared with the subsequent factor study.}
\label{tab:public-continuations}
\setlength{\tabcolsep}{5pt}
\begin{tabular}{@{}rlrrr@{}}
\toprule
Rank & Encoder & Native & Stage mixture & Factor mixture \\
\midrule
16 & Trainable & $22.5\,/\,0.1357$ & $30.5\,/\,0.2034$ & $23.2\,/\,0.1464$ \\
16 & Frozen & $35.3\,/\,0.2444$ & $40.7\,/\,0.2789$ & $28.0\,/\,0.1748$ \\
128 & Trainable & $12.5\,/\,0.0780$ & $20.8\,/\,0.1294$ & $20.5\,/\,0.1259$ \\
128 & Frozen & $20.2\,/\,0.1356$ & $25.5\,/\,0.1690$ & $19.5\,/\,0.1170$ \\
\bottomrule
\end{tabular}
\end{table}

\begin{table}[t]
\centering\small
\caption{Every parent-level native-start success rate (\%). T/F denotes
trainable/frozen sensory encoder; the actor head and critic learn in
both conditions. No parent is excluded.}
\label{tab:public-continuation-seeds}
\setlength{\tabcolsep}{5pt}
\begin{tabular}{@{}rr rrr rrr@{}}
\toprule
& & \multicolumn{3}{c}{Trainable} & \multicolumn{3}{c}{Frozen} \\
Rank & Seed & Native & Stage & Factor & Native & Stage & Factor \\
\midrule
16 & 0 & 43.0 & 51.5 & 18.0 & 54.5 & 52.5 & 34.0 \\
16 & 1 & 7.0 & 20.5 & 14.5 & 28.5 & 43.5 & 33.0 \\
16 & 2 & 17.5 & 19.5 & 37.0 & 23.0 & 26.0 & 17.0 \\
128 & 0 & 20.5 & 31.0 & 27.5 & 31.0 & 35.0 & 20.0 \\
128 & 1 & 9.5 & 12.5 & 10.5 & 18.0 & 26.5 & 17.5 \\
128 & 2 & 7.5 & 19.0 & 23.5 & 11.5 & 15.0 & 21.0 \\
\bottomrule
\end{tabular}
\end{table}

Broad stage mixing improves success and original return in all six
trainable-encoder comparisons and five of six frozen-encoder comparisons
(\cref{tab:public-continuations,tab:public-continuation-seeds}). Its one
exception is rank 16, seed 0, with a frozen encoder. Factor mixing
improves success and return in seven of twelve native comparisons: five
with a trainable encoder and two with a frozen encoder. Relative to broad
stage mixing, factor mixing lowers success and return in nine of twelve.
The gains therefore depend on which later states are supplied, not
merely on increasing their total availability. Neither public study
includes a subsequent withdrawal phase, so native-start transfer does
not establish persistence after further native-only training.

\subsection{Local Action and Representation Diagnostics}
\label{app:public-factor-diagnostic}

The fixed panel crosses an open versus closed-but-unlocked door with
two remote positions of a deposited, unused key. The agent stands
immediately left of the door facing it, carries nothing, and has clock
100. Sixteen declared layout seeds deduplicate into eight wall/door
geometries and 32 distinct states. Exact planning verifies that the
complete seven-action optimal-value vector is identical across key
positions at each door status. Forward is uniquely optimal when open,
and toggle when closed. This equality concerns optimal continuation;
later trajectories under a learned policy can still depend on key
position. The broad stage starts carry the key and do not explicitly
supply these deposited-key configurations, motivating factor mixing.

For each state we measure the optimal-action probability, greedy
correctness, and the one-action oracle regret
\[
 \ell_{\rm oracle}(s)
 =\max_a Q^*(s,a)-\sum_a\pi(a\mid s)Q^*(s,a).
\]
This score assumes that only the first action uses the learned policy
and all later actions are optimal. It is not $Q^\pi$ or the full-episode
regret. Here all first actions can finish within the horizon, so the
undiscounted score equals $0.9/640$ times the expected excess path length.
These are two units of the same measure, not independent evidence.
Absolute key sensitivity is the difference in optimal-action probability
between the two positions; it has no intrinsic direction of harm.

\begin{table}[t]
\centering\small
\caption{Factor-mixture versus native continuation on the fixed panel.
Each arrow gives the native and factor means over three parents.
Sensitivity is an absolute probability difference in percentage points;
excess steps assume optimal continuation after one sampled action.}
\label{tab:public-factor-actions}
\setlength{\tabcolsep}{5pt}
\begin{tabular}{@{}rlrrr@{}}
\toprule
Rank & Encoder & Optimal-action probability (\%) & Excess steps & Sensitivity (pp) \\
\midrule
16 & Trainable & $22.0\to40.3$ & $1.119\to0.839$ & $8.5\to19.9$ \\
16 & Frozen & $29.0\to39.5$ & $1.012\to0.782$ & $11.6\to9.6$ \\
128 & Trainable & $29.3\to39.2$ & $0.994\to0.819$ & $9.4\to10.2$ \\
128 & Frozen & $27.0\to35.4$ & $0.993\to0.859$ & $5.1\to11.9$ \\
\bottomrule
\end{tabular}
\end{table}

Factor mixing raises optimal-action probability and lowers oracle regret
in all twelve native comparisons, including all six frozen-encoder
conditions (\cref{tab:public-factor-actions}). Greedy correctness
improves in eleven and is unchanged in one. Nevertheless, key sensitivity
increases in eight and decreases in four. Broad stage mixing improves
native success while key sensitivity increases in nine of twelve
comparisons. Sensitivity therefore does not consistently track either
local decision quality or task performance.

We also compare finite responses to door status and key position at the
raw encoder output, projected representation and final actor-MLP hidden layer.
For door status $D$ and key position $k$, write the response as $h_{D,k}$,
the key displacement as $d_D=h_{D,1}-h_{D,0}$, and the mean door
displacement as
$u=\frac12\sum_k(h_{{\rm closed},k}-h_{{\rm open},k})$.
The Euclidean parallel-energy fraction is
$(u^\top d_D)^2/(\|u\|^2\|d_D\|^2)$; zero-norm directions are undefined.
No directional degeneracy occurs in this panel. Projected-code alignment
increases in five of six trainable factor arms relative to native, and
is exactly unchanged in all six frozen arms, despite local improvement
in all twelve. Across all 18 frozen endpoints, raw and projected-representation
responses equal their parents', while the actor-MLP responses change.

Holding each head fixed, we symmetrically remove the key displacement's
component parallel to $u$ from its two projected-representation inputs. After
factor training, mean changes in optimal-action probability are
$-0.196,-0.377,+0.168,+0.017$ percentage points for rank 16/trainable,
16/frozen, 128/trainable and 128/frozen, respectively. Complementary
orthogonal edits and full key-displacement removal are retained. The
small, mixed-sign effects do not identify uniformly harmful overlap.
The edits can leave the observation manifold; components have different
norms, and the Euclidean decomposition depends on representation
coordinates. These diagnostics measure selected feature responses and
head behavior, not an identified semantic feature count or a causal
learning mechanism.

\subsection{Delivered Exposure, Failure Stages, and Limits}
\label{app:public-exposure-limits}

The stage curriculum supplies 47.4--50.4\% of resets but 23.1--35.3\%
of learner transitions from its added starts. Factor starts supply
48.0--50.3\% of resets and 22.6--35.9\% of learner transitions. The
denominators differ because episode duration depends on the starting
state and learned policy. The two interventions respectively skip
33,734 and 400,700 oracle-prefix steps; none is executed as a hidden
training transition. The factor total includes no-op padding to clock
100 and does not measure oracle CPU cost.

Factor arms record 4,007 added starts and 4,006 first actions; the single
difference is an automatic reset at the final rollout boundary before
another action. They record 57,107 pre-action visits to the exact factor
configuration family at any clock, 1.82\% of their 3,145,728 learner
transitions. All four categories and all 20 generated wall/door
geometries are covered across the cohort. Ten arms reset into all 32
fixed diagnostic states and two into 31. Fresh layout seeds can therefore
reproduce the diagnostic configurations: this tests deliberate coverage,
not held-out factor generalization. Old native controls lack the new
exact-occupancy counters; their counts are unavailable, not zero.

Saved final native-evaluation logs further locate the remaining failures
(\cref{tab:public-failure-stages}). In the factor group, 1,741 of 1,853
failed episodes (94.0\%) never open the door, whereas every diagnostic
state has an already unlocked door. At least these episodes never
reach the probed stage. Factor training also allocates 30.3--36.6\%
of learner transitions after unlocking, compared with 16.7--20.4\%
under broad stage mixing and 3.8--8.0\% under native training.
This locates an exposure and evaluation mismatch; it does not identify
the cause of each failure or imply that later control is irrelevant.

\begin{table}[t]
\centering\small
\caption{Failure-stage decomposition of all 7,200 final native-start
episodes, with 2,400 per exposure across twelve endpoints. Counts are
evaluation outcomes, not independent training replications.}
\label{tab:public-failure-stages}
\setlength{\tabcolsep}{5pt}
\begin{tabular}{@{}lrrrr@{}}
\toprule
Exposure & No pickup & Pickup, no opening & Opened, no goal & Success \\
\midrule
Native & 907 & 818 & 132 & 543 \\
Stage mixture & 893 & 706 & 96 & 705 \\
Factor mixture & 1,019 & 722 & 112 & 547 \\
\bottomrule
\end{tabular}
\end{table}

All declared arms finish their budgets. Archive checks preserve the
original parents, fixed projectors and frozen encoders; trainable heads
and critics change. The new trainer differs only in its reset-wrapper
import and exposure labels, and two 4,096-step native smoke
continuations exactly reproduce the previous policy and Adam tensors,
supporting control reuse. Raw records retain every parent, midpoint,
endpoint and factor probe. The combined diagnostic panel contains 42
models, 1,344 state records, 1,008 layer contrasts and 5,376 head patches;
repeated checkpoints and states do not increase the number of independent
parents.

The public evidence is consequently strongest for exposure-dependent
learning: supplying reachable later states can improve native-start
performance at fixed capacity, and local decisions can improve through
an unchanged sensory representation. Capacity, entropy, head adaptation and
optimization remain competing explanations. Neither the local response
scores nor these finite-budget interventions establish globally optimal
conditional representations, adaptation-induced value reversal,
persistent self-confirming feedback, or removal of superposition.

\section{Native Continuous-Control Interventions}
\label{app:public-dmcontrol}

\paragraph{Learners and continuation budget.}
These experiments test whether replay, overlap, and entropy interventions
improve continued learning with the task's original state access.
We use Finger-turn-hard and Hopper-hop from DMControl
\citep{tassa2018control}, with their native state observations, resets,
and rewards. Each action is repeated twice, giving 500 learner decisions
per episode, and the two simulator rewards are summed. The official
5M TD-MPC2 model \citep{hansen2024tdmpc2} retains its learned state
encoder, latent dynamics, reward and value models, policy prior, and
MPC planner. SAC uses the Stable-Baselines3 2.7.0 implementation
\citep{raffin2021sb3}, with two 256-unit ReLU layers in its actor and
critics. We add neither a rank constraint nor an auxiliary reconstruction model.

Every parent is trained from random initialization for 100,000 decisions.
For each task and learner, training seeds 1, 2, and 3 produce separate
parents. Baseline and intervention continuations start from the same
parent state within each seed and receive 32,000 further decisions and
32,000 updates. The saved state includes replay, optimizer state, and
random-number generators. TD-MPC2 seed 1 was the initial exploratory
cohort, followed by seeds 2 and 3. Matching interaction and update counts
does not match computation, because replay balancing draws extra proposals
and overlap regularization requires additional feature evaluations.
All reported checkpoints are taken at the fixed endpoint of 132,000
decisions, without selecting an intermediate training checkpoint.

\paragraph{Phase replay.}
To give rarely sampled task situations more weight in fitting, we group
native replay proposals into four cells. Finger uses contact and whether
the spinner tip is in the target region. Hopper uses the native standing
condition and whether forward speed reaches the reward's saturation
threshold. These labels are computed from the observed physical state.
Four native proposal batches supply the candidate pool. We mix uniform
sampling over candidates with equal weight over its occupied cells,
using a mixing fraction at most $0.25$ and capping each candidate's
probability at four times its uniform-pool probability. The sampled
objects remain intact TD-MPC2 sequences or SAC transitions. Empty cells
receive no new data. No importance correction cancels the reweighting,
so the intervention changes the fitting distribution while retaining
native collection and rewards.

\paragraph{Local overlap regularization.}
The explicit feature directions in the theory are unavailable in these
learners. We instead penalize overlap between local responses to the
observed coordinates. Let $f$ denote the TD-MPC2 state encoder or the SAC
actor's hidden representation, and let $s_j$ be coordinate $j$'s standard
deviation in the parent replay, floored at $0.01$. Its response is
approximated by
\[
 u_j(o)=\frac{f(o+0.02s_j e_j)-f(o-0.02s_j e_j)}{0.04}.
\]
The auxiliary loss averages squared cosine similarity between different
responses and adds a penalty when a response norm falls below $0.8$
times its frozen-parent value, with penalty weight 10. The reference
mask and coordinate scales remain fixed. We use eight sampled states
per update and add the auxiliary with coefficient $0.003$ or $0.03$
to the existing model loss for TD-MPC2 or actor loss for SAC.
Observation coordinates need not be independent semantic features, and
these perturbations need not stay on the physical state manifold.
Consequently, this penalty is a local response proxy rather than a
direct measurement of semantic superposition.

\paragraph{Entropy adjustments.}
Both baseline learners already regularize entropy. For TD-MPC2,
we multiply the policy-prior coefficient $10^{-4}$ by two or five,
leaving MPC unchanged. This coefficient does not specify the entropy
of the actions selected by the planner. Standard SAC learns its
entropy coefficient with target $-d$, where $d$ is the action dimension
(two for Finger and four for Hopper). Its interventions change this
target to $-d/2$ or $0$ while retaining automatic temperature updates
and entropy terms in both actor learning and critic targets.
In particular, target $0$ does not turn entropy off.

\paragraph{Checkpoint selection and repeated evaluation.}
We first evaluate each final checkpoint on 50 common native-start
episodes. Separately for each method and task, its highest-mean training
seed is selected before any further evaluation. The selected weights
then remain fixed for five fresh groups of 50 episodes. All methods
share the new environment-seed groups, and TD-MPC2 also shares planner
seeds. These groups are disjoint from the earlier evaluation episodes.
TD-MPC2 uses its original MPC planner and SAC uses deterministic mean
actions. The reported mean averages the five group means, and its
sample standard deviation measures evaluation variation at a fixed
checkpoint. Bold entries in \cref{tab:public-continuous} mark the highest
mean within each learner and task. It is not variation over five independent training runs.
Finite evaluation noise can affect which training seed is selected,
and different methods can select different seeds. The resulting
comparison therefore describes selected-checkpoint performance rather
than a paired effect across training seeds.

For Finger, TD-MPC2 selects seed 1 except for entropy $\times2$, which
selects seed 3. For Hopper, it selects seed 2 except for overlap $0.03$,
which selects seed 3. Standard SAC selects seed 3 throughout Finger
and Hopper except for Hopper overlap $0.003$, which selects seed 2.
All five evaluation groups are retained for every selected checkpoint.

\paragraph{Results across training seeds.}
\Cref{tab:public-continuous-all-seeds} retains the original evaluations
of all three training seeds. TD-MPC2 phase replay and overlap $0.003$
each improve return in all six task--seed pairs. Stronger overlap is
less consistent, including a large Hopper regression in seed 2.
SAC does not show a consistent intervention benefit across the same
tasks. These full-seed comparisons support the learner dependence
seen in \cref{tab:public-continuous}, without treating its repeated
evaluations as additional training replications. The fixed training
budget supplies no certificate that these public-task representations
or policies are optimal. Their return changes therefore test the
practical interventions, not the optimality assumptions or the complete
superposition-trap mechanism of the theoretical constructions.

\begin{table}[t]
\centering
\caption{Original continuation results across all three training seeds.
Entries are mean $\pm$ sample SD of the three checkpoint means, each
estimated from 50 episodes. Unlike \cref{tab:public-continuous},
this table summarizes variation across training runs.}
\label{tab:public-continuous-all-seeds}
\small
\setlength{\tabcolsep}{4pt}
\begin{tabular*}{\textwidth}{@{\extracolsep{\fill}}llrr@{}}
\toprule
\textbf{Learner} & \textbf{Intervention} & \textbf{Finger-turn-hard} & \textbf{Hopper-hop} \\
\midrule
TD-MPC2 & Baseline & $694.9\,\pm\,15.2$ & $97.2\,\pm\,85.5$ \\
 & Phase replay & $780.6\,\pm\,41.6$ & $\mathbf{133.2}\,\pm\,74.2$ \\
 & Overlap $\lambda=0.003$ & $759.0\,\pm\,62.6$ & $108.5\,\pm\,78.0$ \\
 & Overlap $\lambda=0.03$ & $\mathbf{814.8}\,\pm\,43.9$ & $60.2\,\pm\,47.4$ \\
 & Prior entropy $\times2$ & $738.4\,\pm\,64.3$ & $114.7\,\pm\,57.6$ \\
 & Prior entropy $\times5$ & $799.9\,\pm\,75.5$ & $105.3\,\pm\,105.2$ \\
\midrule
SAC & Baseline & $211.3\,\pm\,47.9$ & $\mathbf{22.3}\,\pm\,23.9$ \\
 & Phase replay & $193.5\,\pm\,78.3$ & $13.2\,\pm\,16.4$ \\
 & Overlap $\lambda=0.003$ & $189.7\,\pm\,98.5$ & $18.1\,\pm\,16.0$ \\
 & Overlap $\lambda=0.03$ & $197.2\,\pm\,103.3$ & $21.3\,\pm\,17.5$ \\
 & Target entropy $-d/2$ & $216.0\,\pm\,107.1$ & $12.8\,\pm\,21.3$ \\
 & Target entropy $0$ & $\mathbf{216.5}\,\pm\,107.1$ & $7.2\,\pm\,11.4$ \\
\midrule
SAC (ablation) & Entropy off throughout & $\mathbf{211.9}\,\pm\,97.8$ & $\mathbf{24.7}\,\pm\,23.6$ \\
 & Entropy restored & $169.6\,\pm\,60.7$ & $10.6\,\pm\,16.4$ \\
\bottomrule
\end{tabular*}
\end{table}

\paragraph{Separate entropy-off ablation.}
To distinguish adjustment of existing entropy from its introduction,
we also train SAC parents with a fixed zero entropy coefficient from
initialization. Their actor remains stochastic during training.
One continuation keeps the coefficient zero, whereas the other enables
automatic entropy at 100,000 decisions with initial coefficient one,
a fresh temperature optimizer, and target $-d$. All other parent state
is inherited. These comparisons belong to a separate parent family.
Restoring entropy lowers mean return on both tasks in the full-seed
comparison and in the selected-checkpoint reevaluations
(\cref{tab:public-entropy-off}). The latter select seed 2 for both
Finger arms, seed 2 for Hopper's continued-zero arm, and seed 3 for
its restoration arm.

\begin{table}[t]
\centering
\caption{Entropy restoration from SAC parents trained with entropy disabled.
Each selected checkpoint uses the five-evaluation protocol of
\cref{tab:public-continuous}. These parents are separate from standard SAC.}
\label{tab:public-entropy-off}
\small
\begin{tabular*}{\textwidth}{@{\extracolsep{\fill}}lrr@{}}
\toprule
\textbf{Continuation} & \textbf{Finger-turn-hard} & \textbf{Hopper-hop} \\
\midrule
Entropy off throughout & $244.5\,\pm\,41.1$ & $49.8\,\pm\,1.3$ \\
Entropy restored & $171.0\,\pm\,51.3$ & $29.5\,\pm\,1.5$ \\
\bottomrule
\end{tabular*}
\end{table}

\section{Protected World-model Fitting on Crafter}
\label{app:public-dreamer}

\paragraph{Learner and intervention.}
We use DreamerV3 \citep{hafner2025dreamerv3} with its \texttt{size12m}
configuration: a recurrent state-space model with 2,048 deterministic units,
32 stochastic variables with 16 categories each, convolutional depth 16,
and 256-unit feedforward layers. Inputs are $64\times64$ RGB images.
Each run uses four environments, exactly one million total environment
steps, training ratio 128, batches of 16 sequences of length 64, learning
rate $4\times10^{-5}$, and bfloat16 computation. Both conditions retain the
same architecture, native loss scales, replay sampler including its online
component, and actor/critic objectives.

Each replay position receives a tier from the highest currently held
pickaxe or sword: none, wood, stone, or iron. Inventory supplies fitting
metadata only. The tier is excluded from model and policy inputs. Let
$\mu_g$ be tier $g$'s probability mass over loss-bearing positions under
the base batch sampler, excluding the replay-context prefix, and let
$K=|\{g:\mu_g>0\}|$. The protected condition multiplies each world-model
loss at tier $g$ by
\[
w_g=(1-\rho)+\frac{\rho}{K\mu_g},\qquad \rho=0.25,
\]
so its expected fitting mass is $(1-\rho)\mu_g+\rho/K$ for represented
tiers. The uniform condition uses $\rho=0$. These weights are neither
clipped nor normalized within the minibatch. They apply to reconstruction,
reward, continuation, dynamics, and representation losses. Policy and
value losses receive no protection multiplier. Unrepresented tiers receive
no fitting mass.

\paragraph{Measurement and complete cohort.}
The experimental unit is a paired training seed, with seeds 0, 1, and 2.
For $N$ completed training episodes, let $c_j$ count episodes attaining
achievement $j$ at least once. We report the cumulative training score
\[
C=\exp\!\left[\frac1{22}\sum_{j=1}^{22}
\log\!\left(1+100\frac{c_j}{N}\right)\right]-1.
\]
All six runs completed their budgets. The protected seed-2 score uses
the authors' updated run. These are cumulative training scores, not
independent evaluations of frozen final policies. The comparison in
\cref{tab:dreamer-paired} improves all three paired scores and increases
the mean by $7.68\%$.

\begin{table}[t]
\centering
\caption{Cumulative training Crafter scores at one million steps.
Differences are protected minus uniform. SD is the sample standard
deviation across three paired seeds.}
\label{tab:dreamer-paired}
\begin{tabular}{lrrr}
\toprule
Seed/statistic & Uniform & Protected & Difference \\
\midrule
0 & 7.968 & 8.229 & 0.261 \\
1 & 7.578 & 8.862 & 1.284 \\
2 & 7.943 & 8.201 & 0.258 \\
\midrule
Mean & 7.830 & 8.431 & 0.601 \\
SD & 0.218 & 0.374 & 0.592 \\
\bottomrule
\end{tabular}
\end{table}

\paragraph{Displayed case and predictions.}
\Cref{fig:dreamer-crafter-fitting} displays seed 1, selected for illustration after inspecting all three paired outcomes.
Its learning curve aggregates completed episodes on a 25,000-step grid.
For each final checkpoint, we evaluate the same two recorded clips, one
from each condition's replay. Each model starts from its own zero recurrent
state, observes the identical 32-frame history, and predicts 32 future
frames using identical recorded actions, without future images. Evaluation
uses one latent random seed, 0. Clips are the earliest valid reset-start
windows in the newest eligible replay chunks, selected before inference.
The stored 32-frame mean pixel errors are $0.00535078\to0.00482467$ and
$0.00394098\to0.00335702$. These two cases are not a held-out prediction
benchmark. Displayed frame labels report pixel MSE $\times1{,}000$ at
horizons $+8$ and $+32$, whereas these means average all 32 future frames.\footnote{The protected-model labels $0.75$ and $18.12$ on the first clip
use author-provided updates. The 32-frame means above are computed from
the archived forecasts.}

\paragraph{Scope.}
The study was narrowed from four conditions after viewing the first seed.
The earlier Curious Replay comparison, $11.2092\to8.1910$, is retained
separately. The updated cohort shows higher mean cumulative training scores under
protected fitting. This exploratory comparison does not establish
superposition as the cause of the gains, and the illustrative forecasts
do not replace evaluation of prediction quality across trajectories.

\section{Related Work}
\label{sec:related}

Superposition analyses relate sparsity, co-activation, and capacity to
feature geometry
\citep{elhage2022toy,scherlis2022polysemanticity,prieto2026statistics,sharma2026temporal}.
Our constructions make co-activation depend on the policy and solve the
squared overlaps shared by every optimal code. Self-confirming equilibrium
studies choices sustained by beliefs correct along the equilibrium path
\citep{fudenberg1993selfconfirming}; performative prediction and RL study
feedback from deployed decisions to data or environments
\citep{perdomo2020performative,mandal2023performative,basu2025performative}.
Here the environment is fixed, and the deployed values are the actual
returns of the current encoder and controller. Envelope and replicator analysis
\citep{kakade2001natural,hennes2020neurd,mertikopoulos2018riemannian}
connect those values to the actor.

Partial-observation equilibria, policy confounding, representation-rank
deterioration, and primacy bias also expose difficulties caused by
policy-dependent data
\citep{langosco2022training,suau2024habits,moalla2024representation,nikishin2022primacy}.
Our specific result links the geometry of all global representation
optima to interference, a reversed comparison, attraction, and a strict
adapted-return deficit. The sequential experiments examine related
capacity and exposure effects without assigning every neural failure
to this mechanism.

\section{Discussion, Limitations, and Future Directions}
\label{app:discussion-limitations}

The central distinction is between the return supported by the current
representation and the return available after adapting that representation
to a different policy's data. Our constructions make these two comparisons
disagree even when fitting is globally optimal and the actor receives exact
returns. This isolates a source of poor behavior that cannot be resolved
simply by improving value estimation or fitting the same data more accurately.
The relevant cost is where shared capacity places interference, rather than
the presence of sharing alone.

The complete trap guarantees apply to the stated feature models and actor
dynamics. Equal-frequency constructions separate pair co-activation from
individual feature rarity; the untied-controller construction removes weight
sharing between encoder and controller under its own noise assumptions.
Neither result establishes prevalence in arbitrary neural architectures.
Similarly, the finite-action envelope analysis uses losses affine in
the action distribution and the specified replicator field. Its conclusions do not
transfer automatically to a general MDP parameterized by neural policy
weights. Finite adaptation, imperfect optimization, and changing critics
can introduce additional dynamics.

The empirical studies test these dependencies at different levels of
control. Matched feature data identify an allocation change without changing
singleton frequencies. PPO experiments use environmental rewards and
independent learned controllers, but retain constructed tasks and capacity
constraints. Sequential curricula improve later behavior, including through
fixed sensory encoders. This shows that downstream learning and access to
training states can be sufficient for improvement; it does not identify
encoder adaptation as the unique source of that improvement. Unsuccessful
late interventions also delimit the conditions under which the tested
curriculum helps.

DoorKey supplies an existing task and its original reward, with a learned
CNN and nonlinear action head. Its public-task status does not eliminate
experimental restrictions: the comparisons use fully observed symbolic
grids, chosen capacity constraints, three parent seeds per rank, and a
finite training budget. The state-start curricula provide privileged access
to reachable intermediate states. Matched learner-transition counts
therefore do not imply matched information access. Positive success-rate
changes establish performance benefits in these comparisons, while the
mixed transfer of local improvements shows why performance alone cannot
identify superposition or a self-confirming attractor.

A useful next step is to diagnose feature interference at decisions the
policy actually encounters, then test whether those measurements predict
subsequent learning or guide effective interventions. Such diagnostics
should distinguish current control error from persistent feedback, account
for recovery through controller adaptation, and improve on simple predictors such as current
return and visitation. Their measurement cost also matters. Establishing
that connection would extend the present existence and mechanism results
toward practical risk assessment without treating every representation
overlap as harmful.

\end{document}

%% file: figures/untied_phase_diagram.tex
\begin{tikzpicture}
\begin{axis}[
    width=0.94\linewidth,
    height=4.0cm,
    xmin=0,xmax=1,
    ymin=-1.08,ymax=1.08,
    xlabel={policy occupancy $p=\pi(E)$},
    ylabel={normalized untied slope},
    xtick={0,0.25,0.5,0.75,1},
    ytick={-1,-0.5,0,0.5,1},
    legend style={
        font=\footnotesize,
        draw=none,
        fill=none,
        at={(0.5,1.02)},
        anchor=south,
        legend columns=4
    },
    tick label style={font=\footnotesize},
    label style={font=\small},
    grid=major,
    grid style={gray!18},
]
\addplot[deepred,very thick]
table[x=p,y=eta_0p2,col sep=comma]{figures/untied_phase_data.csv};
\addlegendentry{$\eta=0.2$}
\addplot[deepblue,thick,dashed]
table[x=p,y=eta_1p0,col sep=comma]{figures/untied_phase_data.csv};
\addlegendentry{$\eta=1$}
\addplot[deepgreen,thick,dashdotted]
table[x=p,y=eta_5p0,col sep=comma]{figures/untied_phase_data.csv};
\addlegendentry{$\eta=5$}
\addplot[black,thick,dotted]
table[x=p,y=eta_20p0,col sep=comma]{figures/untied_phase_data.csv};
\addlegendentry{$\eta=20$}
\draw[gray,dotted] (axis cs:0.5,-1.08) -- (axis cs:0.5,1.08);
\end{axis}
\end{tikzpicture}

%% file: figures/phase_diagram.tex
% Appendix-only audit figure for the completed sampled-phase experiment.
% Native width is below the ICLR 5.5 in text block; do not resize this input.
\usepgfplotslibrary{groupplots}
\definecolor{oiBlue}{RGB}{213,94,0}
\definecolor{oiOrange}{RGB}{0,114,178}
\definecolor{oiVermillion}{RGB}{213,94,0}
\begin{tikzpicture}
\begin{groupplot}[
    group style={group size=3 by 1,horizontal sep=11mm},
    width=0.36\linewidth,
    height=4.45cm,
    tick label style={font=\footnotesize},
    label style={font=\footnotesize},
    title style={font=\footnotesize,yshift=-1pt},
    axis line style={black!65,line width=0.45pt},
    tick style={black!65,line width=0.45pt},
    ymajorgrids,
    grid style={black!10,line width=0.35pt},
    axis on top,
    clip=false,
]

% Sampling -> optimized geometry.
\nextgroupplot[
    title={\textbf{(a)} Geometry},
    xmin=0,xmax=1,
    ymin=-1.08,ymax=1.08,
    xlabel={occupancy $p$},
    ylabel={$\cD$},
    xtick={0,0.5,1},
    xticklabels={$0$,$1/2$,$1$},
    ytick={-1,0,1},
]
\path[fill=oiBlue,fill opacity=0.035]
    (axis cs:0,-1.08) rectangle (axis cs:0.381966,1.08);
\path[fill=oiOrange,fill opacity=0.04]
    (axis cs:0.618034,-1.08) rectangle (axis cs:1,1.08);
\draw[black!35,densely dotted,line width=0.5pt]
    (axis cs:0.381966,-1.08)--(axis cs:0.381966,1.08);
\draw[black!35,densely dotted,line width=0.5pt]
    (axis cs:0.618034,-1.08)--(axis cs:0.618034,1.08);
\addplot[oiBlue,line width=1.05pt]
    table[x=p,y=D,col sep=comma]{figures/phase_data.csv};
\addplot[
    only marks,oiBlue,mark=o,mark size=2.0pt,line width=0.6pt,
    mark options={fill=white,solid},
    error bars/.cd,y dir=both,y explicit,
    error bar style={line width=0.45pt},
    error mark options={rotate=90,mark size=1.2pt}
] table[x=p,y=D_median,y error plus=D_errplus,y error minus=D_errminus,col sep=comma]
    {experiments/sampled_phase/outputs/figure_summary.csv};

% Optimized geometry -> deployed distortion, with the capacity control.
\nextgroupplot[
    title={\textbf{(b)} Control},
    xmin=0.12,xmax=0.91,
    ymin=-0.66,ymax=1.66,
    xlabel={occupancy $p$},
    ylabel={$\Delta_{\rm dep}$},
    xtick={0.2,0.5,0.8},
    ytick={-0.5,0,0.5,1,1.5},
    legend to name=phasecapacitylegend,
    legend columns=2,
    legend style={font=\footnotesize,draw=none,
                  /tikz/every even column/.append style={column sep=4mm}},
]
\path[fill=oiVermillion,fill opacity=0.055]
    (axis cs:0.12,-0.66) rectangle (axis cs:0.91,0);
\draw[black!50,line width=0.55pt] (axis cs:0.12,0)--(axis cs:0.91,0);
\addplot[
    oiBlue,dashed,line width=0.9pt,mark=o,mark size=2.0pt,
    mark options={fill=white,solid},
    error bars/.cd,y dir=both,y explicit,
    error bar style={line width=0.5pt},
    error mark options={rotate=90,mark size=1.3pt}
] table[x=p,y=d2_median,y error plus=d2_errplus,y error minus=d2_errminus,col sep=comma]
    {figures/sampled_gap_summary.csv};
\addlegendentry{$d=2$}
\addplot[
    oiOrange,line width=1.0pt,mark=square*,mark size=2.0pt,
    error bars/.cd,y dir=both,y explicit,
    error bar style={line width=0.5pt},
    error mark options={rotate=90,mark size=1.3pt}
] table[x=p,y=d3_median,y error plus=d3_errplus,y error minus=d3_errminus,col sep=comma]
    {figures/sampled_gap_summary.csv};
\addlegendentry{$d=3$}

% Independent rollout check of the analytic bridge.
\nextgroupplot[
    title={\textbf{(c)} Bridge},
    xmin=-1.08,xmax=1.08,
    ymin=-1.08,ymax=1.08,
    xlabel={Gram prediction},
    ylabel={rollout distortion},
    xtick={-1,0,1},
    ytick={-1,0,1},
]
\addplot[black!40,line width=0.7pt]
    coordinates {(-1.05,-1.05) (1.05,1.05)};
\addplot[
    only marks,oiOrange,mark=square*,mark size=2.0pt,
    opacity=0.45,
    restrict expr to domain={\thisrow{dimension}}{3:3}
] table[x=D_gram,y=D_mc,col sep=comma]
    {experiments/sampled_phase/outputs/heldout_branch_distortion.csv};
\addplot[
    only marks,oiBlue,mark=o,mark size=1.8pt,line width=0.45pt,
    mark options={fill=white,solid},opacity=0.38,
    restrict expr to domain={\thisrow{dimension}}{2:2}
] table[x=D_gram,y=D_mc,col sep=comma]
    {experiments/sampled_phase/outputs/heldout_branch_distortion.csv};
\end{groupplot}
% Keep every explanatory label outside the plotting regions.  The caption
% explains the shaded phases and the reversal; this shared legend is placed
% above all three axes so it cannot cover curves, points, or error bars.
\node[anchor=south] at ([yshift=3mm]current bounding box.north)
    {\ref{phasecapacitylegend}};
\end{tikzpicture}

%% file: experiments/sampled_phase/outputs/sample_size_table_rows.tex
% Generated by experiments/sampled_phase/run_experiment.py
\def\SampleSizeTableRows{%
256 & 0.000 / 0.485 & 0.177 / 0.507 & 0.000 / 0.356 \\%
1024 & 0.002 / 0.254 & 0.102 / 0.236 & 0.000 / 0.218 \\%
4096 & 0.000 / 0.106 & 0.039 / 0.074 & 0.011 / 0.143 \\%
16384 & 0.000 / 0.065 & 0.020 / 0.043 & 0.011 / 0.077 \\%
}